\documentclass{article}

\usepackage[final]{neurips_2021}

\usepackage[utf8]{inputenc} 
\usepackage[T1]{fontenc}    
\usepackage[colorlinks=true, linkcolor=blue, citecolor=blue,urlcolor=black]{hyperref}      
\usepackage{url}            
\usepackage{booktabs}       
\usepackage{amsfonts}       
\usepackage{nicefrac}       
\usepackage{microtype}      
\usepackage{xcolor}         
\usepackage{multirow}
\usepackage{enumitem}
\usepackage{makecell}
\usepackage{tablefootnote}

\definecolor{Green}{rgb}{0.13, 0.65, 0.3}
\definecolor{Amber}{rgb}{0.3, 0.5, 1.0}
\usepackage[]{color-edits}
\addauthor{HL}{Green}

\usepackage[utf8]{inputenc} 
\usepackage[T1]{fontenc}    
\usepackage{hyperref}       
\usepackage{url}            
\usepackage{booktabs}       
\usepackage{amsfonts}       
\usepackage{nicefrac}       
\usepackage{microtype}      
\usepackage{mathtools}
\usepackage{natbib}
\usepackage{xcolor}
\usepackage{amsmath}
\usepackage{mathtools}
\usepackage{amsthm}
\usepackage{amssymb}

\usepackage[algo2e, ruled, vlined]{algorithm2e}
\usepackage{algorithm}
\SetKwBlock{MyRepeat}{repeat}{}

\newcommand{\Qref}{\mathring{Q}}
\newcommand{\Vref}{\mathring{V}}
\newcommand{\Bref}{\mathring{B}}
\newcommand{\sinit}{s_{\text{init}}}

\newcommand{\calA}{{\mathcal{A}}}

\newcommand{\calV}{{\mathcal{V}}}
\newcommand{\calH}{\mathcal{H}}
\newcommand{\calX}{{\mathcal{X}}}
\newcommand{\calS}{{\mathcal{S}}}
\newcommand{\calF}{{\mathcal{F}}}

\newcommand{\calL}{{\mathcal{L}}}

\newcommand{\calR}{{\mathcal{R}}}

\newcommand{\calP}{{\mathcal{P}}}

\newcommand{\frakn}{\mathfrak{n}}

\DeclareMathOperator*{\argmin}{argmin}

\newcommand{\eat}[1]{}

\newcommand{\rbr}[1]{\left(#1\right)}
\newcommand{\sbr}[1]{\left[#1\right]}
\newcommand{\cbr}[1]{\left\{#1\right\}}

\newcommand{\abr}[1]{\left|#1\right|}
\newcommand{\bigO}[1]{\order\left( #1 \right)}
\newcommand{\tilO}[1]{\otil\left( #1 \right)}
\newcommand{\lowO}[1]{\lorder\left( #1 \right)}
\newcommand{\bigo}[1]{\order( #1 )}
\newcommand{\tilo}[1]{\otil( #1 )}
\newcommand{\lowo}[1]{\lorder( #1 )}
\DeclarePairedDelimiter\ceil{\lceil}{\rceil}
\DeclarePairedDelimiter\floor{\lfloor}{\rfloor}

\newcommand{\T}{\ensuremath{T_\star}}
\newcommand{\B}{B_\star}
\newcommand{\cmin}{\ensuremath{c_{\min}}}

\newcommand{\var}{\textsc{Var}}

\newcommand{\SA}{\calS\times\calA}

\newcommand{\mb}{\textsc{SVI-SSP}\xspace}
\newcommand{\mf}{\textsc{LCB-Advantage-SSP}\xspace}
\newcommand{\eps}{\varepsilon}

\renewcommand{\P}{\bar{P}}
\newcommand{\n}{n^+}
\newcommand{\istar}{i^{\star}}

\newcommand{\optV}{V^{\star}}
\newcommand{\optQ}{Q^{\star}}
\newcommand{\sumt}{\sum_{t=1}^T}
\newcommand{\sumi}{\sum_{i=1}^{I_k}}

\newcommand{\refV}{V^{\text{\rm ref}}}
\newcommand{\refB}{B^{\text{\rm ref}}}

\newcommand{\RefC}{C_{\text{\rm REF}}}
\newcommand{\RefCs}{C_{\text{\rm REF, 2}}}
\newcommand{\RefV}{V^{\text{\rm REF}}}
\newcommand{\tile}{\widetilde{e}}

\newcommand{\opttilQ}{\widetilde{Q}^{\star}}
\newcommand{\opttilV}{\widetilde{V}^{\star}}

\newcommand{\tilB}{\widetilde{B}}

\newcommand{\tilsinit}{\widetilde{s}_{\text{init}}}
\newcommand{\optq}{q^{\star}}

\newcommand{\optp}{p^{\star}}
\newcommand{\optr}{r^{\star}}
\newcommand{\Lstar}{\calL^{\star}}
\newcommand{\thetastar}{\theta^{\star}}
\newcommand{\phistar}{\phi^{\star}}
\newcommand{\hatC}{\widehat{C}}

\newcommand{\sumsa}[1][s, a]{\sum_{(#1)}}

\newcommand{\sumk}{\sum_{k=1}^K}

\newcommand{\hatc}{\widehat{c}}

\newcommand{\optpi}{\pi^\star}

\newcommand{\tilN}{\widetilde{N}}

\newcommand{\tilc}{\widetilde{c}}

\newcommand{\tilM}{\widetilde{M}}
\newcommand{\tilP}{\widetilde{P}}

\newcommand{\tilS}{\widetilde{\calS}}

\newcommand{\tiloptpi}{{\widetilde{\pi}^\star}}

\newcommand{\refmu}{\mu^{\text{ref}}}
\newcommand{\refsigma}{\sigma^{\text{ref}}}
\newcommand{\refnu}{\nu^{\text{ref}}}
\newcommand{\cl}{\check{l}}
\newcommand{\lti}[1][t,i]{l_{#1}}
\newcommand{\clti}[1][t,i]{\cl_{#1}}
\newcommand{\lt}[1][t]{l_{#1}}

\newcommand{\field}[1]{\mathbb{#1}}

\newcommand{\fR}{\field{R}}

\newcommand{\fN}{\field{N}}
\newcommand{\E}{\field{E}}
\newcommand{\fV}{\field{V}}

\newcommand{\Ind}{\field{I}}

\newcommand{\norm}[1]{\left\|{#1}\right\|}

\newtheorem{lemma}{Lemma}
\newtheorem{theorem}{Theorem}
\newtheorem{cor}[theorem]{Corollary}
\newtheorem{remark}{Remark}
\newtheorem{prop}{Property}

\newcommand{\order}{\ensuremath{\mathcal{O}}}
\newcommand{\lorder}{\ensuremath{\Omega}}
\newcommand{\otil}{\ensuremath{\tilde{\mathcal{O}}}}

\usepackage{prettyref}
\newcommand{\pref}[1]{\prettyref{#1}}
\newcommand{\pfref}[1]{Proof of \prettyref{#1}}
\newcommand{\savehyperref}[2]{\texorpdfstring{\hyperref[#1]{#2}}{#2}}
\newrefformat{eq}{\savehyperref{#1}{Eq.~\textup{(\ref*{#1})}}}
\newrefformat{eqn}{\savehyperref{#1}{Equation~\ref*{#1}}}
\newrefformat{lem}{\savehyperref{#1}{Lemma~\ref*{#1}}}
\newrefformat{def}{\savehyperref{#1}{Definition~\ref*{#1}}}
\newrefformat{line}{\savehyperref{#1}{Line~\ref*{#1}}}
\newrefformat{thm}{\savehyperref{#1}{Theorem~\ref*{#1}}}
\newrefformat{corr}{\savehyperref{#1}{Corollary~\ref*{#1}}}
\newrefformat{cor}{\savehyperref{#1}{Corollary~\ref*{#1}}}
\newrefformat{sec}{\savehyperref{#1}{Section~\ref*{#1}}}
\newrefformat{subsec}{\savehyperref{#1}{Section~\ref*{#1}}}
\newrefformat{app}{\savehyperref{#1}{Appendix~\ref*{#1}}}
\newrefformat{assum}{\savehyperref{#1}{Assumption~\ref*{#1}}}
\newrefformat{ex}{\savehyperref{#1}{Example~\ref*{#1}}}
\newrefformat{fig}{\savehyperref{#1}{Figure~\ref*{#1}}}
\newrefformat{alg}{\savehyperref{#1}{Algorithm~\ref*{#1}}}
\newrefformat{rem}{\savehyperref{#1}{Remark~\ref*{#1}}}
\newrefformat{conj}{\savehyperref{#1}{Conjecture~\ref*{#1}}}
\newrefformat{prop}{\savehyperref{#1}{Proposition~\ref*{#1}}}
\newrefformat{proto}{\savehyperref{#1}{Protocol~\ref*{#1}}}
\newrefformat{prob}{\savehyperref{#1}{Problem~\ref*{#1}}}
\newrefformat{claim}{\savehyperref{#1}{Claim~\ref*{#1}}}
\newrefformat{que}{\savehyperref{#1}{Question~\ref*{#1}}}
\newrefformat{op}{\savehyperref{#1}{Open Problem~\ref*{#1}}}
\newrefformat{fn}{\savehyperref{#1}{Footnote~\ref*{#1}}}
\newrefformat{tab}{\savehyperref{#1}{Table~\ref*{#1}}}
\newrefformat{fig}{\savehyperref{#1}{Figure~\ref*{#1}}}
\newrefformat{prop}{\savehyperref{#1}{Property~\ref*{#1}}}

\title{Implicit Finite-Horizon Approximation and Efficient Optimal Algorithms for Stochastic Shortest Path}

\author{
  Liyu Chen\\ 
  University of Southern California\\
  \texttt{liyuc@usc.edu} \\
  \And
  Mehdi Jafarnia-Jahromi \\
  University of Southern California\\
  \texttt{mjafarni@usc.edu} \\
  \And
  Rahul Jain \\
  University of Southern California\\
  \texttt{rahul.jain@usc.edu} \\
  \And
  Haipeng Luo \\
  University of Southern California\\
  \texttt{haipengl@usc.edu} \\
}

\begin{document}

\maketitle

\begin{abstract}
	We introduce a generic template for developing regret minimization algorithms in the Stochastic Shortest Path (SSP) model, which achieves minimax optimal regret as long as certain properties are ensured.
	The key of our analysis is a new technique called implicit finite-horizon approximation, which approximates the SSP model by a finite-horizon counterpart \textit{only in the analysis} without explicit implementation.
	Using this template, we develop two new algorithms: the first one is model-free (the first in the literature to our knowledge) and minimax optimal under strictly positive costs; the second one is model-based and minimax optimal even with zero-cost state-action pairs, matching the best existing result from~\citep{tarbouriech2021stochastic}.
	Importantly, both algorithms admit highly sparse updates, making them  computationally more efficient than all existing algorithms.
	Moreover, both can be made completely parameter-free.
\end{abstract}

\section{Introduction}

We study the Stochastic Shortest Path (SSP) model, where an agent aims to reach a goal state with minimum cost in a stochastic environment.
SSP is well-suited for modeling many real-world applications, such as robotic manipulation, car navigation, and others.
Although it is widely studied empirically (e.g.,~\citep{andrychowicz2017hindsight,nasiriany2019planning}) and in optimal control theory (e.g.,~\citep{bertsekas1991analysis,bertsekas2013stochastic}), it has received less attention under the regret minimization setting where a learner needs to learn the environment and improve her policy on-the-fly through repeated interaction.
Specifically, the problem proceeds in $K$ episodes.
In each episode, the learner starts at a fixed initial state, sequentially takes action, suffers some cost, and transits to the next state, until reaching a predefined goal state.
The performance of the learner is measured by her regret, which is the difference between her total costs and that of the best policy.

\cite{tarbouriech2020no} develop the first regret minimization algorithm for SSP with a regret bound of $\tilo{D^{3/2}S\sqrt{AK/\cmin}}$, where $D$ is the diameter, $S$ is the number of states, $A$ is the number of actions, and $\cmin$ is the minimum cost among all state-action pairs.
\cite{cohen2020near} improve over their results and give a near optimal regret bound of $\tilo{\B S\sqrt{AK}}$, where $\B \leq D$ is the largest expected cost of the optimal policy starting from any state.
Even more recently, \cite{cohen2021minimax} achieve minimax regret of $\tilo{\B\sqrt{SAK}}$ through a finite-horizon reduction technique, and concurrently \cite{tarbouriech2021stochastic} also propose minimax optimal and parameter-free algorithms. 
Notably, all existing algorithms are model-based with space complexity $\lowo{S^2A}$.
Moreover, they all update the learner's policy through full-planning (a term taken from~\citep{efroni2019tight}), incurring a relatively high time complexity.

In this work, we further advance the state-of-the-art by proposing a generic template for regret minimization algorithms in SSP (\pref{alg:template}), which achieves minimax optimal regret as long as some properties are ensured.
By instantiating our template differently, we make the following two key algorithmic contributions:
\begin{itemize}[leftmargin=1em]
  \setlength\itemsep{1em}
  
	\item In \pref{sec:mf}, we develop the \textit{first model-free} SSP algorithm called \mf (\pref{alg:Q}).
	Similar to most model-free reinforcement learning algorithms, 
	\mf does not estimate the transition directly, enjoys a space complexity of $\tilo{SA}$, and also takes only $\bigO{1}$ time to update certain statistics in each step, making it a highly efficient algorithm.
	It achieves a regret bound of $\tilo{\B\sqrt{SAK} + \B^5S^2A/\cmin^4 }$, which is minimax optimal when $\cmin>0$.
	Moreover, it can be made parameter-free without worsening the regret bound.
	
	\item In \pref{sec:mb}, we develop another simple model-based algorithm called \mb (\pref{alg:SVI}), which achieves minimax regret $\tilo{ \B\sqrt{SAK} + \B S^2A}$ even when $\cmin=0$, matching the best existing result by~\citet{tarbouriech2021stochastic}.\footnote{Depending on the available prior knowledge, the final bounds achieved by \mb are slightly different, but they all match that of EB-SSP. See \citep[Table 1]{tarbouriech2021stochastic} for more details.}
	Notably, compared to their algorithm (as well as other model-based algorithms), \mb is computationally much more efficient since it updates each state-action pair only logarithmically many times, and each update only performs \textit{one-step planning} (again, a term taken from~\citep{efroni2019tight}) as opposed to full-planning (such as value iteration or extended value iteration); see more concrete time complexity comparisons in \pref{sec:mb}.
	\mb can also be made parameter-free following the idea of \citep{tarbouriech2021stochastic}.
\end{itemize}

We include a summary of regret bounds of all existing SSP algorithms as well as more complexity comparisons in~\pref{app:table}.



\paragraph{Techniques} Our main technical contribution is a new analysis framework called \textit{implicit finite-horizon approximation} (\pref{sec:method}), which is the key to analyze algorithms developed from our template.
The high level idea is to approximate an SSP instance by a finite-horizon counterpart.
However, the approximation \textit{only happens in the analysis}, a key difference compared to~\citep{chen2020minimax, chen2021finding, cohen2021minimax} that explicitly implement such an approximation in their algorithms.
As a result, our method not only avoids blowing up the space complexity by a factor of the horizon, but also allows one to derive a horizon-free regret bound (more explanation to follow).

In order to achieve the minimax optimal regret, our model-free algorithm \mf uses a key  variance reduction idea via a reference-advantage decomposition by~\citep{zhang2020almost}.
However, crucial distinctions exist.
For example, we update the reference value function more frequently instead of only one time, which helps reduce the sample complexity and improve the lower-order term in the regret bound.
We also maintain an empirical upper bound on the value function in a doubling manner, which is the key to eventually make the algorithm parameter-free.
On the other hand, for our model-based algorithm \mb, we adopt a special Bernstein-style bonus term and bound the learner's total variance via recursion, taking inspiration from~\citep{tarbouriech2021stochastic,zhang2020reinforcement}.



\paragraph{Empirical Evaluation} We support our theoretical findings with experiments in \pref{app:exp}.
Our model-free algorithm demonstrates a better convergence rate compared to vanilla Q learning with naive $\epsilon$-greedy exploration.
Our model-based algorithm has competitive performance compared to other model-based algorithms, while spending the least amount of time in updates.

\paragraph{Related Work} 

For a detailed comparison of existing results for the same problem, we refer the readers to~\citep[Table~1]{tarbouriech2021stochastic} as well as our \pref{tab:summary}.
There are also several works~\citep{rosenberg2020adversarial, chen2020minimax, chen2021finding} that consider the even more challenging SSP setting where the cost function is decided by an adversary and can change over time.
Apart from regret minimization, \cite{tarbouriech2021sample} study the sample complexity of SSP with a generative model;
\cite{lim2012autonomous} and \cite{tarbouriech2020improved} investigate exploration problems involving multiple goal states (multi-goal SSP).


The special case of SSP with a fixed horizon has been studied extensively, for both stochastic costs (e.g., \citep{azar2017minimax,jin2018q,efroni2019tight,zanette2019tighter,zhang2020reinforcement}) and adversarial costs (e.g.,  \citep{neu2012adversarial,zimin2013online,rosenberg2019online,jin2019learning}).
Importantly, recent works~\citep{wang2020long,zhang2020reinforcement} find that when the cost for each episode is at most a constant, it is in fact possible to obtain a regret bound with only logarithmic dependency on the horizon.
\cite{tarbouriech2021stochastic} generalize this concept to SSP and define horizon-free regret as a bound with only logarithmic dependence on the expected hitting time of the optimal policy starting from any state (which is bounded by $\B/\cmin$).
They also propose the first algorithm with horizon-free regret for SSP, which is important for arguing minimax optimality even when $\cmin=0$.
Notably, our model-based algorithm \mb also achieves horizon-free regret (but the model-free one does not).


\section{Preliminaries}
\label{sec:prelim}

An SSP instance is defined by a Markov Decision Process (MDP) $M=(\calS, \calA, \sinit, g, c, P)$, where $\calS$ is the state space, $\calA$ is the action space, $\sinit\in\calS$ is the initial state, and $g\notin\calS$ is the goal state.
When taking action $a$ in state $s$, the learner suffers a cost drawn in an i.i.d manner from an unknown distribution with mean $c(s, a)\in[0, 1]$ and support $[\cmin, 1]$ ($\cmin\geq 0$), and then transits to the next state $s'\in\calS^+=\calS\cup\{g\}$ with probability $P_{s, a}(s')$.
We assume that the transition $P$ and the cost mean $c$ are unknown to the learner, while all other parameters are known.

The learning process goes as follows: the learner interacts with the environment for $K$ episodes.
In the $k$-th episode, the learner starts in initial state $\sinit$, sequentially takes an action, suffers a cost, and transits to the next state until reaching the goal state $g$.
More formally, at the $i$-th step of the $k$-th episode, the learner observes the current state $s^k_i$ (with $s^k_1 = \sinit$), takes action $a^k_i$, suffers a cost $c^k_i$, and transits to the next state $s^k_{i+1}\sim P_{s^k_i, a^k_i}$.
An episode ends when the current state is $g$, and we define the length of episode $k$ as $I_k$, such that $s^k_{I_k+1}=g$.

\paragraph{Learning Objective} 
At a high level, the learner's goal is to reach the goal with a small total cost.
To this end, we focus on \textit{proper policies} --- a (stationary and deterministic) policy $\pi:\calS\rightarrow \calA$ is a mapping that assigns an action $\pi(s)$ to each state $s\in\calS$, and it is proper if the goal is reached with probability $1$ when following $\pi$ (that is, taking action $\pi(s)$ whenever in state $s$).
Given a proper policy $\pi$, one can define the cost-to-go function $V^{\pi}:\calS\rightarrow [0,\infty)$ as 
$
	V^{\pi}(s) = \E\sbr{\left.\sum_{i=1}^Ic_i\right|P, \pi, s_1=s},
$
where the expectation is with respect to the randomness of the cost $c_i$ incurred at state-action pair $(s_i, \pi(s_i))$, next state $s_{i+1}\sim P_{s_i, \pi(s_i)}$, and the number of steps $I$ before reaching $g$.
The optimal proper policy $\optpi$ is then defined as a policy such that $V^{\optpi}(s) = \min_{\pi\in\Pi}V^{\pi}(s)$ for all $s \in \calS$, where $\Pi$ is the set of all proper policies assumed to be nonempty.
The formal objective of the learner is then to minimize her regret against $\optpi$, the difference between her total cost and that of the optimal proper policy, defined as
\begin{align*}
	R_K = \sumk\sum_{i=1}^{I_k}c^k_i - K\cdot \optV(\sinit),
\end{align*}
where we use $\optV$ as a shorthand for $V^{\optpi}$.
The minimax optimal regret is known to be $\tilo{\B\sqrt{SAK}}$,
where $\B=\max_{s\in\calS}\optV(s)$, and $S = |\calS^+|$ and $A = |\calA|$ are the numbers of states (including the goal state) and actions respectively~\citep{cohen2020near}.


\paragraph{Bellman Optimality Equation}
For a proper policy $\pi$, the corresponding action-value function $Q^{\pi}:\SA\rightarrow [0, \infty)$ is defined as $Q^{\pi}(s,a) = c(s,a) + \E_{s'\sim P_{s,a}}[V^{\pi}(s')]$.
Similarly, we use $\optQ$ as a shorthand for $Q^{\optpi}$.
it is known that $\optpi$ satisfies the Bellman optimality equation:
$\optV(s) = \min_{a \in \calA} \optQ(s,a)$ for all $s \in \calS$~\citep{bertsekas1991analysis}.

\paragraph{Assumption on $\cmin$}
Similar to many previous works, our analysis requires $\cmin$ being known and strictly positive.
When $\cmin$ is unknown or known to be $0$,  a simple workaround is to solve a modified SSP instance with all observed costs clipped to $\epsilon$ if they are below some $\epsilon>0$, so that $\cmin = \epsilon > 0$.
Then the regret in this modified SSP is similar to that in the original SSP up to an additive term of order $\bigO{\epsilon K}$~\citep{tarbouriech2020no}.
Therefore, throughout the paper we assume that $\cmin$ is known and strictly positive unless explicitly stated otherwise.

\paragraph{Other Notations} 

For simplicity, we use $C_K=\sumk\sumi c^k_i$ in the analysis to denote the total costs suffered by the learner over $K$ episodes.
For a function $X: \calS^+ \rightarrow \fR$ and a distribution $P$ over $\calS^+$, denote by $PX = \E_{S\sim P}[X(S)]$, $PX^2 = \E_{S\sim P}[X(S)^2]$, and $\fV(P, X)=\var_{S\sim P}[X(S)]$ the expectation, second moment, and variance of $X(S)$ respectively where $S$ is drawn from $P$.
For a scalar $x$, define $(x)_+=\max\{x, 0\}$, 
and denote by $\ceil{x}_2=2^{\ceil{\log_2x}}$ and $\floor{x}_2 = 2^{\floor{\log_2x}}$ the closest power of two upper and lower bounding $x$ respectively.
For an integer $m$, $[m]$ denotes the set $\{1, \ldots, m\}$.
In pseudocode, $x \overset{+}{\leftarrow} y$ is a shorthand for the increment operation $x \leftarrow x + y$.


\section{Implicit Finite-Horizon Approximation}
\label{sec:method}

In this section, we introduce our main analytical technique, that is, implicitly approximating the SSP problem with a finite-horizon counterpart.
We start with a general template of our algorithms shown in \pref{alg:template}.
For notational convenience, we concatenate state-action-cost trajectories of all episodes as one single sequence $(s_t, a_t, c_t)$ for $t=1, 2, \ldots, T$, where $s_t \in \calS$ is one of the non-goal state, $a_t \in\calA$ is the action taken at $s_t$, and $c_t$ is the resulting cost incurred by the learner.
Note that the goal state $g$ is never included in this sequence (since no action is taken there), and we also use the notation $s_t' \in \calS^+$ to denote the next-state following $(s_t, a_t)$, so that $s_{t+1}$ is simply $s_t'$ unless $s_t' = g$ (in which case $s_{t+1}$ is reset to the initial state $\sinit$); see \pref{line:t_notation}. 

\DontPrintSemicolon 
\setcounter{AlgoLine}{0}
\begin{algorithm}[t]
	\caption{A General Algorithmic Template for SSP}
	\label{alg:template}
	
	\textbf{Initialize:} $t \leftarrow 0$, $s_1\leftarrow \sinit$, $Q(s, a)\leftarrow 0$ for all $(s, a)\in\SA$. 
	
	
	\For{$k=1,\ldots,K$}{
		
		\MyRepeat{
			\nl Increment time step $t\overset{+}{\leftarrow}1$.
			
			\nl Take action $a_t= \argmin_aQ(s_t, a)$, suffer cost $c_t$, transit to and observe $s'_t$. \label{line:greedy}
			
			\nl Update $Q$ (so that it satisfies \pref{prop:optimism} and \pref{prop:recursion}).
			
			\nl \lIf{$s_t' \neq g$}{$s_{t+1}\leftarrow s'_t$; \textbf{else} $s_{t+1}\leftarrow \sinit$, \textbf{break}.} \label{line:t_notation}
		}	
	}
	Record $T \leftarrow t$ (that is, the total number of steps).
\end{algorithm}

The template follows a rather standard idea for many reinforcement learning algorithms:
maintain an (optimistic) estimate $Q$ of the optimal action-value function $\optQ$,
and act greedily by taking the action with the smallest estimate: $a_t= \argmin_aQ(s_t, a)$; see \pref{line:greedy}.
The key of the analysis is often to bound the estimation error $\optQ(s_t, a_t)-Q(s_t, a_t)$, which is relatively straightforward in a discounted setting (where the discount factor controls the growth of the error) or a finite-horizon setting (where the error vanishes after a fixed number of steps), but becomes highly non-trivial for SSP due to the lack of similar structures.

A natural idea is to explicitly solve a discounted problem or a finite-horizon problem that approximates the original SSP well enough.
Unfortunately, both approaches are problematic: approximating an undiscounted MDP by a discounted one often leads to suboptimal regret~\citep{wei2020model};
on the other hand, while explicitly approximating SSP with a finite-horizon problem can lead to optimal regret~\citep{chen2020minimax, cohen2021minimax}, it greatly increases the space complexity of the algorithm,
and also produces non-stationary policies, which is unnatural and introduces unnecessary complexity since the optimal policy in SSP is stationary.

Therefore, we propose to approximate the original SSP instance $M$ with a finite-horizon counterpart $\tilM$ implicitly (that is, only in the analysis).
We defer the formal definition of $\tilM$ to \pref{app:method}, which is similar to those in~\citep{chen2020minimax, cohen2021minimax} and corresponds to interacting with the original SSP for $H$ steps (for some integer $H$) and then teleporting to the goal.
All we need in the analysis are the optimal value function $\optV_{h}$ and optimal action-value function $\optQ_h$ of $\tilM$ for each step $h \in [H]$, which can be defined recursively without resorting to the definition of $\tilM$:
\begin{equation}\label{eq:Q_H_def}
	\optQ_h(s, a) = c(s, a) + P_{s, a}\optV_{h-1},\qquad \optV_h(s) = \min_a\optQ_h(s, a),
\end{equation}
with $\optQ_0(s,a) = 0$ for all $(s,a)$.\footnote{%
Note that our notation is perhaps unconventional compared to most works on finite-horizon MDPs, where $\optQ_h$ usually refers to our $\optQ_{H-h}$.
We make this switch since we want to highlight the dependence on $H$ for $\optQ_H$.
}
Intuitively, $\optQ_H$ approximates $\optQ$ well when $H$ is large enough.
This is formally summarized in the lemma below, whose proof is similar to prior works (see \pref{app:method}).

\begin{lemma}
	\label{lem:loop free approx}
	For any value of $H$,  $\optQ_H(s, a) \leq \optQ(s, a)$ holds for all $(s,a)$.
	For any $\beta \in (0,1)$, if $H\geq\frac{4\B}{\cmin}\ln(2/\beta)+1$, then $\optQ(s, a) \leq \optQ_H(s, a) + \B\beta$ holds for all $(s,a)$.
\end{lemma}

In the remaining discussion, we fix a particular value of $H$.
To carry out the regret analysis, we now specify two general requirements of the estimate $Q$.
Let $Q_t$ be the value of $Q$ at the beginning of time step $t$ (that is, the value used in finding $a_t$).
Then $Q_t$ needs to satisfy:
\begin{prop}[Optimism]\label{prop:optimism}
With high probability, $Q_t(s, a)\leq\optQ(s, a)$ holds for all $(s, a)$ and $t\geq 1$.
\end{prop}
\begin{prop}[Recursion]\label{prop:recursion}
There exists a ``bonus overhead'' $\xi_H >0$ and an absolute constant $d>0$ such that
the following holds with high probability: 
\[
\sumt(\Qref(s_t, a_t) - Q_t(s_t ,a_t))_+ \leq \xi_H + \rbr{1+\frac{d}{H}}\sumt(\Vref(s_t)-Q_t(s_t, a_t))_+,
\]
for $\Qref = \optQ_h$ and $\Vref = \optV_{h-1}$  ($h = 1, \ldots, H$) as well as $\Qref = \optQ$ and $\Vref = \optV$.\footnote{Note that $\xi_H$ might be a random variable. In fact, it often depends on $C_K$.} 
\end{prop}

\pref{prop:optimism} is standard and can usually be ensured by using a certain ``bonus'' term derived from concentration equalities in the update.
These bonus terms on $(s_t,a_t)$ accumulate into some bonus overhead in the final regret bound, which is exactly the role of $\xi_H$ in \pref{prop:recursion}.
In both of our algorithms, $\xi_H$ has a leading-order term $\tilo{\sqrt{\B SAC_K}}$
and a lower-order term that increases in $H$.
 
\pref{prop:recursion} is a key property that provides a recursive form of the estimation error and allows us to connect it to the finite-horizon approximation.
This is illustrated through the following two lemmas.

\begin{lemma}\label{lem:Q_H_recursion}
\pref{prop:recursion} implies $\sumt(\optQ_H(s_t, a_t) - Q_t(s_t ,a_t))_+ \leq \bigO{H\xi_H}$.
\end{lemma}
\begin{proof}
With $\Qref = \optQ_H$ and $\Vref = \optV_{H-1}$, 
\pref{prop:recursion} implies
\begin{align*}
\sumt(\optQ_H(s_t, a_t) - Q_t(s_t ,a_t))_+ 
&\leq \xi_H + \rbr{1+\frac{d}{H}}\sumt(\optV_{H-1}(s_t)-Q_t(s_t, a_t))_+  \\
&\leq \xi_H + \rbr{1+\frac{d}{H}}\sumt(\optQ_{H-1}(s_t, a_t)-Q_t(s_t, a_t))_+, 
\end{align*}
where in the last step we use the optimality of $\optV_{H-1}$ from \pref{eq:Q_H_def}.
Repeatedly applying this argument, we eventually arrive at
$\sumt(\optQ_H(s_t, a_t) - Q_t(s_t ,a_t))_+ \leq H\rbr{1+\frac{d}{H}}^H\xi_H +\rbr{1+\frac{d}{H}}^H\sumt(\optQ_{0}(s_t, a_t)-Q_t(s_t, a_t))_+ = \bigO{H\xi_H}$,
where the last step uses the facts $\optQ_{0}(s_t, a_t) = 0$ and $\rbr{1+\frac{d}{H}}^H \leq e^d$ (an absolute constant).
\end{proof}

\begin{lemma}\label{lem:Q_star_recursion}
For any $\beta \in (0,1)$, if $H\geq\frac{4\B}{\cmin}\ln(2/\beta)+1$, then \pref{prop:optimism} and \pref{prop:recursion} together imply $\sumt \optQ(s_t, a_t) - \optV(s_t) = \bigO{\beta C_K + \xi_H}$.
\end{lemma}

\begin{proof}
Applying \pref{prop:recursion} with $\Qref = \optQ$ and $\Vref = \optV$, we have
$\sumt (\optQ(s_t, a_t) - Q_t(s_t, a_t))_+ \leq \xi_H + \rbr{1 + \frac{d}{H}}\sumt (\optV(s_t) - Q_t(s_t, a_t))_+$.
Now note that by \pref{prop:optimism}, the Bellman optimality equation $\optV(s_t) = \min_a \optQ(s_t,a)$, and the fact $Q_t(s_t, a_t) = \min_a Q_t(s_t,a)$ (by the definition of $a_t$), the arguments within the clipping operation $(\cdot)_+$ are all non-negative and thus the clipping can be removed.
Rearranging terms then gives
\begin{align*}
\sumt \optQ(s_t, a_t) - \optV(s_t) &\leq \xi_H + \frac{d}{H}\sumt (\optV(s_t) - Q_t(s_t, a_t)) \\
&\leq \xi_H + \frac{d}{H}\sumt (\optQ(s_t, a_t) - Q_t(s_t, a_t)). \tag{optimality of $\optV$}
\end{align*}	
It remains to bound the last term using the finite-horizon approximation $\optQ_H$ as a proxy:
\begin{align*}
\sumt (\optQ(s_t, a_t) - Q_t(s_t, a_t)) 
&= \sumt (\optQ(s_t, a_t) - \optQ_H(s_t, a_t) + \optQ_H(s_t, a_t)  - Q_t(s_t, a_t)) \\
& = \bigO{T\B\beta + H\xi_H},
\end{align*}
where the last step uses \pref{lem:loop free approx} and \pref{lem:Q_H_recursion}.
Importantly, this term is finally scaled by $d/H$,
which, together with the fact $\frac{T\B}{H} \leq \cmin T \leq C_K$, proves the claimed bound.
\end{proof}

Readers familiar with the literature might already recognize the term $\sumt \optQ(s_t, a_t) - \optV(s_t)$ considered in \pref{lem:Q_star_recursion}, which is closely related to the regret.
Indeed, with this lemma, we can conclude a regret bound for our generic algorithm.

\begin{theorem}\label{thm:main_regret} 
For any $\beta \in (0,1)$, if $H\geq\frac{4\B}{\cmin}\ln(2/\beta)+1$, then \pref{alg:template} ensures (with high probability) $R_K = \tilO{\sqrt{\B C_K} + \B + \beta C_K + \xi_H}$.
\end{theorem}
\begin{proof}
We first decompose the regret as follows, which holds generally for any algorithm:
\begin{align}
	R_K &= \sumk\rbr{\sum_{i=1}^{I_k} c^k_i - \optV(s^k_1)} \notag\\ 
	&\leq \sumk\sum_{i=1}^{I_k} \rbr{c^k_i - \optV(s^k_i) + \optV(s^k_{i+1})}= \sumt (c_t - \optV(s_t) + \optV(s'_t))\notag \\
	&= \sumt (c_t - c(s_t, a_t)) + \sumt (\optV(s'_t) - P_{s_t, a_t}\optV) + \sumt (\optQ(s_t, a_t) - \optV(s_t)). \label{eq:reg}
\end{align}
The first and the second term are the sum of a martingale difference sequence (since $s_t'$ is drawn from $P_{s_t, a_t}$) and can be bounded by $\tilO{\sqrt{C_K}}$ and $\tilO{\sqrt{\B C_K} + \B}$ respectively using concentration inequalities; see \pref{lem:cost diff}, \pref{lem:any interval freedman}, and \pref{lem:var optV}.
The third term can be bounded using \pref{lem:Q_star_recursion} directly, which finishes the proof.
\end{proof}

To get a sense of the regret bound in \pref{thm:main_regret},
first note that since $1/\beta$ only appears in a logarithmic term of the required lower bound of $H$, one can pick $\beta$ to be small enough so that the term $\beta C_K$ is dominated by others.
Moreover, if $\xi_H$ is $\tilo{\sqrt{\B SAC_K}}$ plus some lower-order term $\rho_H$ (which as mentioned is the case for our algorithms),
then by solving a quadratic of $\sqrt{C_K}$, the regret bound of \pref{thm:main_regret} implies $R_K = \tilo{\B\sqrt{SAK} + \rho_H}$,
which is minimax optimal (ignoring $\rho_H$)!

Based on this analytical technique, it remains to design algorithms satisfying the two required properties. 
In the following sections, we provide two such examples, leading to the first model-free SSP algorithm and an improved model-based SSP algorithm.

\section{The First Model-free Algorithm: \mf}
\label{sec:mf}

\setcounter{AlgoLine}{0}
\begin{algorithm}[t]
	\caption{\mf}
	\label{alg:Q}
	
	
	\textbf{Parameters:} horizon $H$, threshold $\thetastar$, and failure probability $\delta \in (0,1)$.
	
	\textbf{Define:} $\Lstar=\{E_j\}_{j\in\fN^+}$ where $E_j=\sum_{i=1}^je_i$, 
	$e_1=H$ and $e_{j+1}=\floor{(1+1/H)e_j}$.
	
	\textbf{Initialize:} $t \leftarrow 0$, $s_1\leftarrow \sinit$, $B\leftarrow1$, for all $(s, a), N(s, a)\leftarrow 0, M(s, a)\leftarrow 0$.
	
	\textbf{Initialize:} for all $(s, a), Q(s, a)\leftarrow 0, V(s)\leftarrow 0, \refV(s)\leftarrow V(s), \hatC(s, a)\leftarrow 0$.
	
	\textbf{Initialize:} for all $(s, a), \refmu(s, a)\leftarrow 0, \refsigma(s, a)\leftarrow 0, \mu(s, a)\leftarrow 0, \sigma(s, a)\leftarrow 0$, $v(s, a)\leftarrow 0$.
	
	\For{$k=1,\ldots,K$}{
	
	     \MyRepeat{	
              Increment time step $t\overset{+}{\leftarrow}1$.
              
			Take action $a_t= \argmin_aQ(s_t, a)$, suffer cost $c_t$, transit to and observe $s'_t$. 
			
			\nl  Increment visitation counters: $n=N(s_t, a_t)\overset{+}{\leftarrow} 1, m=M(s_t, a_t) \overset{+}{\leftarrow} 1$. 
			
			\nl  Update global accumulators: $\refmu(s_t, a_t) \overset{+}{\leftarrow} \refV(s'_t), \; \refsigma(s_t, a_t)\overset{+}{\leftarrow} \refV(s'_t)^2$, \; $\hatC(s_t, a_t)\overset{+}{\leftarrow} c_t$. \label{line:global_accum}
			
			\nl  Update local accumulators: $v(s_t, a_t)\overset{+}{\leftarrow} V(s'_t), \; \mu(s_t, a_t) \overset{+}{\leftarrow} V(s'_t) - \refV(s'_t), \; \sigma(s_t, a_t) \overset{+}{\leftarrow} (V(s'_t) - \refV(s'_t))^2$. \label{line:local_accum}

			
			\nl \If{$n \in \Lstar$}{ \label{line:update call}
			     \nl  Compute $\iota\leftarrow 256\ln^6(4SA\B^8n^5/\delta)$, cost estimator $\hatc=\frac{\hatC(s_t, a_t)}{n}$, bonuses $b' \leftarrow 2\sqrt{\frac{B^2\iota}{m}} + \sqrt{\frac{\hatc\iota}{n}} + \frac{\iota}{n}$ and $b \leftarrow$
			     \[ \label{line:compute}
 					\sqrt{\frac{\nicefrac{\refsigma(s_t,a_t)}{n}-(\nicefrac{\refmu(s_t,a_t)}{n})^2}{n}\iota} + \sqrt{\frac{\nicefrac{\sigma(s_t,a_t)}{m}-(\nicefrac{\mu(s_t,a_t)}{m})^2}{m}\iota} + \rbr{\frac{4B}{n} + \frac{3B}{m}}\iota + \sqrt{\frac{\hatc\iota}{n}}.
\] \label{line:bonus}
			
				\nl $Q(s_t, a_t)\leftarrow \max\left\{\hatc + \frac{v(s_t, a_t)}{m} -b', Q(s_t, a_t)\right\}$.\label{line:update mf1}
				
				\nl $Q(s_t, a_t)\leftarrow \max\left\{\hatc + \frac{\refmu(s_t, a_t)}{n} + \frac{\mu(s_t, a_t)}{m} - b, Q(s_t, a_t)\right\}$.\label{line:update mf2}
			
				\nl  $V(s_t)\leftarrow \min_a Q(s_t, a)$. \label{line:update V}
				
				\nl \lIf{$V(s_t)>B$}{\label{line:B}
				$B \leftarrow 2V(s_t)$.
			}
				
				\nl  Reset local accumulators: $v(s_t, a_t)\leftarrow 0, \mu(s_t, a_t)\leftarrow 0, \sigma(s_t, a_t)\leftarrow 0, M(s_t, a_t)\leftarrow 0$. \label{line:reset_local}
			}
						
			\nl \lIf{$\sum_aN(s_t, a)$ is a power of two not larger than $\thetastar$}{\label{line:update refV}
				$\refV(s_t)\leftarrow V(s_t)$.
			}

			 \lIf{$s_t' \neq g$}{$s_{t+1}\leftarrow s'_t$; \textbf{else} $s_{t+1}\leftarrow \sinit$, \textbf{break}.} 
		}		
	}
\end{algorithm}

In this section, we present a model-free algorithm (the first in the literature) called \mf that falls into our generic template and satisfies the required properties.
It is largely inspired by the state-of-the-art model-free algorithm \textsc{UCB-Advantage}~\citep{zhang2020almost} for the finite-horizon problem.
The pseudocode is shown in \pref{alg:Q}, with only the lines instantiating the update rule of the $Q$ estimates numbered.
Importantly, the space complexity of this algorithm is only $\bigO{SA}$ since we do not estimate the transition directly or conduct explicit finite-horizon reduction,
and the time complexity is only $\bigO{1}$ in each step.

Specifically, for each state-action pair $(s, a)$, we divide the samples received when visiting $(s,a)$ into consecutive stages of exponentially increasing length, and only update $Q(s, a)$ at the end of a stage.
The number of samples $e_j$ in stage $j$ is defined through $e_1=H$ and $e_{j+1}=\floor{(1+1/H)e_j}$ for some parameter $H$.
Further define $\Lstar=\{E_j\}_{j\in\fN^+}$ with $E_j=\sum_{i=1}^je_i$, which contains all the indices indicating the end of some stage.
As mentioned, the algorithm only updates $Q(s,a)$ when the total number of visits to $(s,a)$ falls into the set $\Lstar$ (\pref{line:update call}).
The algorithm also maintains an estimate $V$ for $\optV$, which always satisfies $V(s) = \min_a Q(s,a)$ (\pref{line:update V}),
and importantly another reference value function $\refV$ whose role and update rule are to be discussed later.

In addition, some local and global accumulators are maintained in the algorithm.
Local accumulators only store information related to the current stage.
These include: $M(s,a)$, the number of visits to $(s,a)$ within the current stage;
$v(s,a)$, the cumulative value of $V(s')$ within the current stage, where $s'$ represents the next state after each visit to $(s,a)$; and finally
$\mu(s,a)$ and $\sigma(s,a)$, the cumulative values of $V(s')-\refV(s')$ and its square respectively within the current stage (\pref{line:local_accum}).
These local accumulators are reset to zero at the end of each stage (\pref{line:reset_local}).

On the other hand, global accumulators store information related to all stages and are never reset.
These include: $N(s,a)$, the number of visits to $(s,a)$ from the beginning;
$\hatC(s, a)$, total cost incurs at $(s, a)$ from the beginning;
and $\refmu(s,a)$ and $\refsigma(s,a)$, the cumulative value of $\refV(s')$ and its square respectively from the beginning, where again $s'$ represents the next state after each visit to $(s,a)$ (\pref{line:global_accum}).

We are now ready to describe the update rule of $Q$.
The first update, \pref{line:update mf1}, is intuitively based on the equality $\optQ(s,a) = c(s,a) + P_{s,a}\optV$ and uses $v(s,a)/M(s,a)$ as an estimate for $P_{s,a}\optV$ together with a (negative) bonus $b'$ derived from Azuma's inequality (\pref{line:bonus}).
As mentioned, the bonus is necessary to ensure \pref{prop:optimism} (optimism) so that $Q$ is always a lower confidence bound of $\optQ$ (hence the name ``LCB'').
Note that this update only uses data from the current stage (roughly $1/H$ fraction of the entire data collected so far), which leads to an extra $\sqrt{H}$ factor in the regret.

To address this issue, \citet{zhang2020almost} introduce a variance reduction technique via a reference-advantage decomposition, which we borrow here leading to the second update rule in \pref{line:update mf2}.
This is intuitively based on the decomposition $P_{s,a}\optV=P_{s,a}\refV + P_{s,a}(\optV - \refV)$,
where $P_{s,a}\refV$ is approximated by $\refmu(s,a)/N(s,a)$ and
$P_{s,a}(\optV - \refV)$ is approximated by $\mu(s,a)/M(s,a)$.
In addition, a ``variance-aware'' bonus term $b$ is applied, which is derived from a tighter Freedman's inequality (\pref{line:bonus}).
The reference function $\refV$ is some snapshot of the past value of $V$,
and is guaranteed to be $\bigo{\cmin}$ close to $\optV$ on a particular state as long as the number of visits to this state exceeds some threshold $\thetastar = \tilO{\B^2H^3SA/\cmin^2}$ (\pref{line:update refV}).
Overall, this second update rule not only removes the extra $\sqrt{H}$ factor as in~\citep{zhang2020almost},
but also turns some terms of order $\tilo{\sqrt{T}}$ into $\tilo{\sqrt{C_K}}$ in our context, which is important for obtaining the optimal regret.

Despite the similarity, we emphasize several key differences between our algorithm and that of \citep{zhang2020almost}.
First, \citep{zhang2020almost} maintains a different $Q$ estimate for each step of an episode (which is natural for a finite-horizon problem), while we only maintain one $Q$ estimate (which is natural for SSP).
Second, we update the reference function $\refV(s)$ whenever the number of visits to $s$ doubles (while still below the threshold $\thetastar$; see \pref{line:update refV}),
instead of only updating it once as in \citep{zhang2020almost}. 
We show in \pref{lem:refV} that this helps reduce the sample complexity and leads to a smaller lower-order term in the regret.
Third, since there is no apriori known upper bound on $V$ (unlike the finite-horizon setting), we maintain an empirical upper bound $B$ (in a doubling manner) such that $V(s)\leq B\leq 2\B$ (\pref{line:B}), which is further used in computing the bonus terms $b$ and $b'$.
This is important for eventually developing a parameter-free algorithm.

In \pref{app:mf}, we show that \pref{alg:Q} indeed satisfies the two required properties.

\begin{theorem}\label{thm:mf_properties}
Let $H=\ceil{\frac{4\B}{\cmin}\ln(\frac{2}{\beta})+1}_2$ for $\beta=\frac{\cmin}{2\B^2 SAK}$ and $\thetastar = \tilO{\frac{\B^2H^3SA}{\cmin^2}}$ be defined in \pref{lem:refV},
then \pref{alg:Q} satisfies \pref{prop:optimism} and \pref{prop:recursion} with $d = 3$ and $\xi_H = \tilO{\sqrt{\B SAC_K} + \frac{\B^2H^3S^2A}{\cmin}}$.
\end{theorem}
\begin{proof}[Proof Sketch]
	The proof of \pref{prop:optimism} largely follows the analysis of \citep[Proposition 4]{zhang2020almost} for the designed bonuses.
	To prove \pref{prop:recursion}, similarly to \citep{zhang2020almost} we can show:
	\begin{align*}
		\sumt (\Qref(s_t, a_t) - Q_t(s_t, a_t))_+ \lesssim \xi_H + \sumt \frac{1}{m_t}\sum_{i=1}^{m_t}P_{s_{\clti}, a_{\clti}}(\Vref - V_{\clti})_+,
	\end{align*}
	where $m_t$ is the value of $m$ used in computing $Q_t(s_t, a_t)$, and $\clti$ is the $i$-th time step the agent visits $(s_t, a_t)$ among those $m_t$ steps.
	Now it suffices to show that $\sumt \frac{1}{m_t}\sum_{i=1}^{m_t}P_{s_{\clti}, a_{\clti}}(\Vref - V_{\clti})_+\lesssim (1+\frac{3}{H})\sumt (\Vref(s_t) - V_t(s_t))_+$, which is proven in \pref{lem:p2a}.
\end{proof}

As a direct corollary of \pref{thm:main_regret}, we arrive at the following regret guarantee.

\begin{theorem}
	\label{thm:Q}
With the same parameters as in \pref{thm:mf_properties}, 
with probability at least $1-60\delta$, \pref{alg:Q} ensures $R_K = \tilO{\B\sqrt{SAK} + \frac{\B^5S^2A}{\cmin^4} }$.
\end{theorem}

We make several remarks on our results.
First, while \pref{alg:Q} requires setting the two parameters $H$ and $\thetastar$ in terms of $\B$ to obtain the claimed regret bound,
one can in fact achieve the exact same bound without knowing $\B$ by slightly changing the algorithm.
The high level idea is to first apply the doubling trick from \cite{tarbouriech2021stochastic} to determine an upper bound on $\B$, then try logarithmically many different values of $H$ and $\thetastar$ simultaneously, each leading to a different update rule for $Q$ and $\refV$.
This only increases the time and space complexity by a logarithmic factor, without hurting the regret (up to log factors). Details are deferred to \pref{sec:parameter free}.

Second, as mentioned in \pref{sec:prelim}, when $\cmin$ is unknown or $\cmin=0$, one can clip all observed costs to $\epsilon$ if they are below $\epsilon>0$, which introduces an additive regret term of order $\bigO{\epsilon K}$. 
By picking $\epsilon$ to be of order $K^{-1/5}$, our bound becomes $\tilO{K^{4/5}}$ ignoring other parameters.
Although most existing works suffer the same issue, this is certainly undesirable, and our second algorithm to be introduced in the next section completely avoids this issue by having only logarithmic dependence on $1/\cmin$. 

Finally, we point out that, just as in the finite-horizon case, the variance reduction technique is crucial for obtaining the minimax optimal regret.
For example, if one instead uses an update rule similar to the (suboptimal) Q-learning  algorithm of~\citep{jin2018q}, then this is essentially equivalent to removing the second update (\pref{line:update mf2}) of our algorithm.
While this still satisfies \pref{prop:recursion}, the bonus overhead $\xi_H$ would be $\sqrt{H}$ times larger, resulting in a suboptimal leading term in the regret.

\section{An Optimal and Efficient Model-based Algorithm: \mb}
\label{sec:mb}

In this section, we propose a simple model-based algorithm called \mb (Sparse Value Iteration for SSP) following our template, which not only achieves the minimax optimal regret even when $\cmin=0$, matching the state-of-the-art by a recent work~\citep{tarbouriech2021stochastic}, but also admits highly sparse updates, making it more efficient than all existing model-based algorithms.
The pseudocode is in \pref{alg:SVI}, again with only the lines instantiating the update rule for $Q$ numbered.

Similar to \pref{alg:Q}, \mb divides samples of each $(s, a)$ into consecutive stages of (roughly) exponentially increasing length, and only update $Q(s, a)$ at the end of a stage (\pref{line:update scheme mb}).
However, the number of samples $e_j$ in stage $j$ is defined slightly differently through 
$e_j=\floor{\tile_j}, \tile_1=1$, and $\tile_{j+1}=\tile_j + \frac{1}{H}e_j$ for some parameter $H$.
In the long run, this is almost the same as the scheme used in \pref{alg:Q}, but importantly, it forces more frequent updates at the beginning --- for example, one can verify that $e_1 = \cdots = e_H = 1$, meaning that $Q(s,a)$ is updated every time $(s,a)$ is visited for the first $H$ visits.
This slight difference turns out to be important to ensure that the lower-order term in the regret has no $\text{poly}(H)$ dependence, as shown in \pref{lem:update scheme} and further discussed in \pref{rem:update constant}.
More intuition on the design of this update scheme is provided in \pref{sec:properties update mb}.

The update rule for $Q$ is very simple (\pref{line:update mb}).
It is again based on the equality $\optQ(s,a) = c(s,a) + P_{s,a}\optV$,
but this time uses $\bar{P}_{s,a}V - b$ as an approximation for $P_{s,a}\optV$,
where $\bar{P}_{s,a}$ is the empirical transition directly calculated from two counters $n(s,a)$ and $n(s,a,s')$ (number of visits to $(s,a)$ and $(s,a,s')$ respectively),
$V$ is such that $V(s) = \min_a Q(s,a)$,
and $b$ is a special bonus term (\pref{line:bonus mb}) adopted from \citep{tarbouriech2021stochastic,zhang2020reinforcement} which ensures that $Q$ is an optimistic estimate of $\optQ$ and also helps remove $\text{poly}(H)$ dependence in the regret.

\setcounter{AlgoLine}{0}
\begin{algorithm}[t]
	\caption{\mb}
	\label{alg:SVI}
	
	\textbf{Parameters:} horizon $H$, value function upper bound $B$, and failure probability $\delta \in (0,1)$.
	
	\textbf{Define:} $\calL=\{E_j\}_{j\in\fN^+}$, where $E_j=\sum_{i=1}^je_i, e_j=\floor{\tile_j}$, and $\tile_1=1, \tile_{j+1}=\tile_j + \frac{1}{H}e_j$.
	
	\textbf{Initialize:} $t\leftarrow 0, s_1\leftarrow \sinit$. 
	
	\textbf{Initialize:} for all $(s, a, s'), n(s, a, s')\leftarrow 0, n(s, a)\leftarrow 0$, $Q(s, a)\leftarrow 0$, $V(s)\leftarrow 0$, $\hatC(s, a)\leftarrow 0$. 
	
	\For{$k=1,\ldots,K$}{
		
		\MyRepeat{
			Increment time step $t\overset{+}{\leftarrow}1$.
			
			Take action $a_t= \argmin_aQ(s_t, a)$, suffer cost $c_t$, transit to and observe $s'_t$.
			
			\nl Update accumulators: $n=n(s_t, a_t)\overset{+}{\leftarrow} 1, n(s_t, a_t, s'_t)\overset{+}{\leftarrow} 1$, $\hatC(s_t, a_t)\overset{+}{\leftarrow} c_t$.
						
			\nl \If{$n \in\calL$}{ \label{line:update scheme mb}
			\nl Update empirical transition: $\P_{s_t, a_t}(s')\leftarrow\frac{n(s_t, a_t, s')}{n}$ for all $s'$. 
			
				\nl Compute $\iota\leftarrow 20\ln\frac{2SAn}{\delta}$, cost estimator $\hatc\leftarrow\frac{\hatC(s, a)}{n}$, and bonus $b\leftarrow \max\Big\{7\sqrt{\frac{\fV(\bar{P}_{s_t, a_t}, V)\iota}{n}}, \frac{49B\iota}{n}\Big\}+\sqrt{\frac{\hatc\iota}{n}}$. \label{line:bonus mb}
			
				\nl $Q(s_t, a_t) \leftarrow \max\{\hatc + \P_{s_t, a_t}V - b, Q(s_t, a_t)\}$. \label{line:update mb}
				
				\nl $V(s_t)\leftarrow \argmin_a Q(s_t, a)$.
			}
			
			\lIf{$s_t' \neq g$}{$s_{t+1}\leftarrow s'_t$; \textbf{else} $s_{t+1}\leftarrow \sinit$, \textbf{break}.} 
		}	
	}
\end{algorithm}

\mb exhibits a unique structure compared to existing algorithms.
In each update, it modifies only one entry of $Q$ (similarly to model-free algorithms), while other model-based algorithms such as~\citep{tarbouriech2021stochastic} perform value iteration for every entry of $Q$ repeatedly until convergence (concrete time complexity comparisons to follow).
We emphasize that our implicit finite-horizon analysis is indeed the key to enable us to derive a regret guarantee for such a sparse value iteration algorithm.
Specifically, in \pref{app:mb}, we show that \mb satisfies the two required properties.


\begin{theorem}
	\label{thm:mb_properties}
	If $B\geq\B$ and $H=\ceil{\frac{4B}{\cmin}\ln(\frac{2}{\beta})+1}_2$ for $\beta=\frac{\cmin}{2B^2 SAK}$,
	then \pref{alg:SVI} satisfies \pref{prop:optimism} and \pref{prop:recursion} with $d = 1$ and $\xi_H = \tilo{\sqrt{\B SAC_K} + BS^2A + \beta C_K}$, where the dependence on $H$ in $\xi_H$ is hidden in logarithmic terms.
\end{theorem}
\begin{proof}[Proof Sketch]
	The proof of \pref{prop:optimism} largely follows the analysis of \citep[Lemma 15]{tarbouriech2021stochastic}.
	To prove \pref{prop:recursion}, we first show
$\sumt (\Qref(s_t, a_t) - Q_t(s_t, a_t))_+ \lesssim \xi_H + \sumt P_t(\Vref - V_{\lt})_+$, 
where $\lt$ is the last time step $Q(s_t, a_t)$ is updated.
Then, the remaining main steps are shown below with all details deferred to the corresponding key lemmas:
	\begin{align*}
		\sumt P_t(\Vref - V_{\lt})_+ &\lesssim \rbr{1+\frac{1}{H}}\sumt P_t(\Vref - V_t)_+ \tag{\pref{lem:update scheme}}\\ 
		&\lesssim \rbr{1 + \frac{1}{H}}\sumt (\Vref(s_t) - V_t(s_t))_+ + \rbr{1 + \frac{1}{H}}\sumt (P_t - \Ind_{s'_t})(\Vref - V_t)_+\\
		&\lesssim \rbr{1 + \frac{1}{H}}\sumt (\Vref(s_t) - V_t(s_t))_+ + \xi_H, \tag{\pref{lem:var recursion} and \pref{lem:bound terms}}
	\end{align*}
which completes the proof.
\end{proof}

Again, as a direct corollary of \pref{thm:main_regret}, we arrive at the following regret guarantee.
\begin{theorem}
	\label{thm:mb}
	With the same parameters as in \pref{thm:mb_properties}, with probability at least $1-12\delta$, \pref{alg:SVI} ensures $R_K = \tilo{ \B\sqrt{SAK} + BS^2A}$.
\end{theorem}

Setting $B = \B$, our bound becomes $\tilo{ \B\sqrt{SAK} + \B S^2A}$, which is minimax optimal even when $\cmin$ is unknown or $\cmin = 0$ (this is because the dependence on $1/\cmin$ is only logarithmic, and one can clip all observed costs to $\epsilon$ if they are below $\epsilon=1/K$ in this case without introducing $\text{poly}(K)$ overhead to the regret).
When $\B$ is unknown, we can use the same doubling trick from~\cite{tarbouriech2021stochastic} to obtain almost the same bound (with only the lower-order term increased to $\tilO{\B^3S^3A}$); see \pref{sec:pf mb} for details.\footnote{We note that this doubling trick is in fact also applicable to \pref{alg:Q}. However, the specific approach we propose for this algorithm in \pref{sec:parameter free} is better in the sense that it does not worsen the regret at all.}

\paragraph{Comparison with EB-SSP~\citep{tarbouriech2021stochastic}}
Our regret bounds match exactly the state-of-the-art by \citet{tarbouriech2021stochastic}.
Thanks to the sparse update, however, \mb has a much better time complexity.
Specifically, for \mb, each $(s,a)$ is updated at most $\tilo{H} = \tilo{\nicefrac{\B}{\cmin}}$ times (\pref{lem:update scheme}), and each update takes $\bigo{S}$ time, leading to total complexity $\tilo{\nicefrac{\B S^2A}{\cmin}}$.
On the other hand, for EB-SSP, although each $(s,a)$ only causes $\tilo{1}$ updates,
each update runs value iteration on all entries of $Q$ until convergence, which takes $\tilo{\nicefrac{\B^2S^2}{\cmin^2}}$ iterations (see their Appendix~C) and leads to total complexity $\tilo{\nicefrac{\B^2S^5A}{\cmin^2}}$, much larger than ours.

\paragraph{Comparison with ULCVI~\citep{cohen2021minimax}}
Another recent work by \citet{cohen2021minimax} using explicit finite-horizon approximation also achieves minimax regret but requires the knowledge of some hitting time of the optimal policy.
Without this knowledge, their bound has a large $1/\cmin^4$ dependence in the lower-order term just as our model-free algorithm.
Our results in this section show that implicit finite-horizon approximation has advantage over explicit approximation apart from reducing space complexity:
the former does not necessarily introduce $\text{poly}{(H)}$ dependence even for the lower-order term, while the latter does under the current analysis.

\begin{ack}
	LC thanks Chen-Yu Wei for many helpful discussions.
	HL is supported by NSF Award IIS-1943607 and a Google Faculty Research Award.
	MJ and RJ's research is supported by NSF CCF-1817212, NSF ECCS-1810447 and ONR N00014-20-1-2258 awards.
\end{ack}

\bibliographystyle{plainnat}
\bibliography{ref}

\newpage
\appendix

\section{A Summary of Existing Bounds}\label{app:table}

\renewcommand{\arraystretch}{1.3}
\begin{table*}[t]
	\centering
	\caption{Summary of existing regret minimization algorithms for SSP with their best achievable bounds (assuming necessary prior knowledge). Here, $D, S, A$ are the diameter, number of states, and number of actions of the MDP, $\T$ is the maximum expected hitting time of the optimal policy over all states, $\B$ is the maximum expected costs of the optimal policy over all states, and $K$ is the number of episodes.
	}
	\label{tab:summary}
	\vspace{5pt}
	\begin{tabular}{| c | c | }
		\hline
		 Algorithm & Regret Bound \\
		 \hline
		 UC-SSP~\citep{tarbouriech2020no} & $\tilO{ DS\sqrt{DAK/\cmin} + S^2AD ^2 }$ \\
		 \hline
		 Bernstein-SSP~\citep{cohen2020near} & $\tilO{ \B S\sqrt{AK} + \sqrt{ \B^3S^2A^2/\cmin } }$ \\
		 \hline
		 ULCVI~\citep{cohen2021minimax} & $\tilO{ \B\sqrt{SAK} + \T^4S^2A }$ \\
		 \hline
		 EB-SSP~\citep{tarbouriech2021stochastic} & $\tilO{ \B\sqrt{SAK} + \B S^2A }$ \\
		 \hline
		 \mf \textbf{(Ours)} & $\tilO{ \B\sqrt{SAK}+\B^5S^2A/\cmin^4 }$\\
		 \hline
		 \mb \textbf{(Ours)} & $\tilO{ \B\sqrt{SAK} + \B S^2A }$\\
		 \hline
	\end{tabular}
\end{table*}

A summary of existing regret minimization algorithms for SSP and their regret bounds is shown in \pref{tab:summary}.
Note that although \mf has a larger lower order term depending on $\tilo{1/\cmin^4}$ among the minimax optimal algorithms, it actually nearly matches that of ULCVI when $\T$ is unknown, in which case their algorithm is run with $\T$ replaced by its upper bound $\B/\cmin$.

\paragraph{Time Complexity} 
When $\cmin=0$, the cost perturbation trick is applied (see paragraph ``Assumption on $\cmin$'' in \pref{sec:prelim} for more details) and $1/\cmin$ becomes a $K$-dependent quantity.
This leads to a worse $K$-dependent time complexity for all algorithms in \pref{tab:summary} except ULCVI.
In fact, this seems to be a shared limitation of all algorithms that learns a stationary policy.
On the other hand, when $\T$ is known, ULCVI (which learns a non-stationary policy) gives a better time complexity with no polynomial dependency on $K$.
How to learn a stationary policy while avoiding $K$-dependent time complexity when $\cmin=0$ is an interesting future direction.

\section{Preliminaries for the Appendix}\label{app:prelim}
\paragraph{Extra Notations in Appendix}
Denote by $\Delta_{\calX}$ the simplex over set $\calX$.
For conciseness, throughout the appendix, we use the following notational shorthands:
\begin{itemize}[leftmargin=2em]
  \setlength\itemsep{0em}
\item $\Ind_{s}(s')=\Ind\{s=s'\}$;
\item $P_t=P_{s_t, a_t}$;
\item for a function $f_t: \SA\rightarrow\fR$, we often abuse the notation and use $f_t$ to denote $f_t(s_t, a_t)$ when there is no confusion from the context;
in fact, in \pref{lem:mf optimistic} and \pref{lem:optimistic mb}, we also use $f_t$ to denote $f_t(s,a)$ for a particular $(s,a)$ pair;
\item $\calV_H=\{(\optQ, \optV)\}\cup\{(\optQ_h, \optV_{h-1})\}_{h=1}^H$.
\end{itemize}

Note that for any $(\Qref, \Vref) \in \calV_H$, we have $\Qref(s, a) = c(s, a) + P_{s, a}\Vref$, $\Vref(s)\in[0, \B]$, $\Vref(g)=0$ and $\Vref(s)\leq\min_a\Qref(s, a)$.
Throughout the paper, $\tilO{\cdot}$ also hides dependence on $\ln(1/\delta)$ and $\ln T$ where $\delta\in(0, e^{-1}]$ is some failure probability, and $T$ is a random variable but can be bounded by $\frac{C_K}{\cmin}$ under strictly positive costs. 
We include a summary of most notations in \pref{tab:notations}.



\renewcommand{\arraystretch}{1.5}
\begin{table}[t]
	\caption{Explanation of the notations}
	\centering
	\label{tab:notations}
	\begin{tabular}{| l | l |}
		\hline
		$\beta$ & precision of the implicit finite horizon approximation;\\
		\hline
		$Q_t, V_t$ & accumulators $Q, V$ at the beginning of time step $t$;\\
		\hline
		$\optQ, \optV$ & optimal value functions of the SSP instance;\\
		\hline
		$\optQ_h, \optV_h$ & \makecell[l]{optimal value functions of taking $h$ steps in the SSP instance and then \\ teleporting to the goal state; see \pref{eq:Q_H_def}}\\
		\hline
		$C_K$ & total costs the agent suffers in $K$ episodes;\\
		\hline
		$\refV_t$ & reference value function at the beginning of time step $t$;\\
		\hline
		$\RefV$ & reference value function at the end of learning; see \pref{lem:refV}\\
		\hline
		$\RefC$ & costs in regret of using reference value function; see \pref{lem:refV}\\
		\hline
		$\RefCs$ & another costs in the regret of using reference value function; see \pref{lem:refV}\\
		\hline
		$B_t$ & an upper bound of estimated value function $V_t$;\\
		\hline
		$\hatc_t(s, a)$ & cost estimator used in the last update of $Q_t(s, a)$;\\
		\hline
		$n_t(s, a)$ & the number of visits to $(s, a)$ before the current stage;\tablefootnote{In \pref{tab:notations}, ``the current stage'' means the current stage of $(s, a)$ at time step $t$, and ``the last stage'' means the last stage of $(s, a)$ before time step $t$.}\\
		\hline
		$m_t(s, a)$ & the number of visits to $(s, a)$ in the last stage;\\
		\hline
		$b_t(s, a), b'_t(s, a)$ & bonus terms used in the last update of $Q_t(s, a)$;\\
		\hline
		$\lti(s, a)$ & \makecell[l]{the $i$-th time step the agent visits $(s, a)$ among those $n_t(s, a)$ steps before \\ the current stage;}\\
		\hline
		$\clti(s, a)$ & \makecell[l]{the $i$-th time step the agent visits $(s, a)$ among those $m_t(s, a)$ steps within \\ the last stage;}\\
		\hline
		$\lt(s, a)$ & \makecell[l]{the last time step the agent visits $(s, a)$ before the current stage;}\\
		\hline
		$\iota_t(s, a)$ & logarithmic terms used in the last update of $Q_t(s, a)$;\\
		\hline
		$\eps_t$ & indicator of whether time step $t$ is in the first stage of $(s_t, a_t)$;\\
		\hline
		$\nu_t$ & \makecell[l]{empirical variance of the advantage (i.e., the difference between the estimate \\ value function and the reference value function) at time step $t$; see \pref{eq:nu_t}}\\
		\hline
		$\refnu_t$ & empirical variance of the reference value function at time step $t$; see \pref{eq:refnu_t}\\
		\hline
		$\bar{P}_{t, s, a}$ & empirical transition at $(s, a)$ at the beginning of the current stage of $(s, a)$;\\
		\hline
		$e_j, E_j$ & the length of the $j$-th stage and the total length of the first $j$ stages;\\
		\hline
	\end{tabular}
\end{table}

\paragraph{Truncating the Interaction} An important question in SSP is whether the algorithm halts in a finite number of steps.
To implicitly show this, we do the following trick throughout the analysis.
Fix any positive integer $T'$ and explicitly stop the algorithm after $T'$ steps.
Our analysis will show that in this case the regret $R_K$ is bounded by something independent of $T'$, which then allows us to take $T'$ to infinity and recover the original setting while maintaining the same bound.
This also implicitly shows that the algorithm must halt in a finite number of steps.


\section{Omitted Details for \pref{sec:method}}
\label{app:method}

In this section, we provide omitted details and proofs for \pref{sec:method}.
We first introduce the class of finite horizon MDPs used in the approximation:
given an SSP model $M=(\calS, \calA, \sinit, g, c, P)$, we consider the costs of interacting with $M$ for at most $H$ steps and then directly teleporting to the goal state.
Specifically, we define a finite-horizon SSP $\tilM=(\tilS, \calA, \tilsinit, g, \tilc, \tilP)$ as follows:
\begin{itemize}
	\item $\tilS = \calS \times [H], \tilsinit = (\sinit, 1)$ and the goal state $g$ remains the same;
	\item transition from $(s, h)$ to $(s', h')$ is only possible when $h'=h+1$, and the transition follows the original MDP: $\tilP((s', h+1)|(s, h), a) = P(s'|s, a)$ for $h\in [H-1]$ and $\tilP(g|(s, H), a)=1$;
	\item mean cost function also follows the original MDP: $\tilc_k((s, h), a)=c_k(s, a)$.
\end{itemize}
We also define $\optQ_0(s, a)=\optV_0(s)=0, \optQ_h(s, a)=\opttilQ((s, H-h+1), a), \optV_h(s)=\opttilV(s, H-h+1)$ for $h\in[H]$, where $\opttilQ$ and $\opttilV$ are optimal state-action and state value functions in $\tilM$.
Then, it is straightforward to verify that $\optQ_h$ and $\optV_h$ satisfy \pref{eq:Q_H_def}.
Since $M$ is equivalent to $\tilM$ with $H=\infty$, intuitively
 we should have $\optQ(s, a)\approx\optQ_H(s, a)$ for a sufficiently large $H$.
The formal statement, shown in \pref{lem:loop free approx}, is proven below:
\begin{proof}[\pfref{lem:loop free approx}]
	By definition $\optQ_h(s, a) \leq \optQ(s, a)$ holds for all $(s, a)\in\SA$ and $h\in [H]$, since $\tilM$ is a truncated version of $M$.
	Therefore, $\optV_h(s) \leq \B$ holds, and the expected hitting time (the number of steps needed to reach the goal) of the optimal policy in $\tilM$ starting from any $(s, h)$ is upper bounded by $\frac{\B}{\cmin}$.
	By \citep[Lemma 6]{rosenberg2020adversarial}, when $h\geq \frac{4\B}{\cmin}\ln\frac{2}{\beta}$, the probability of not reaching $g$ in $h$ steps is at most $\beta$.
	Denote by $\tiloptpi_L$ the optimal policy of $\tilM$, and $\optpi_L$ a non-stationary policy in $M$ which follows $\tiloptpi_L$ for the first $H$ steps, and then follows $\optpi$ afterwards.
	We have for any $s\in\calS, \optV(s) - \optV_{H-1}(s) \leq V^{\optpi_L}(s) - V^{\tiloptpi_L}_{H-1}(s) \leq \B\beta$,
	where we apply $H\geq \frac{4\B}{\cmin}\ln\frac{2}{\beta}+1, \optV(s) \leq V^{\tiloptpi_L}(s)$ and $\optV_{H-1}(s) = V^{\tiloptpi_L}_{H-1}(s)$.
	Finally, $\optQ(s, a) - \optQ_H(s, a) = P_{s, a}(\optV - \optV_{H-1})\leq \B\beta$.
\end{proof}

\begin{lemma}
	\label{lem:cost diff}
	With probability at least $1-2\delta$, $\sumt c_t - c(s_t, a_t) = \tilO{\sqrt{C_K}}$.
\end{lemma}
\begin{proof}
	By \pref{eq:anytime strong freedman} of \pref{lem:any interval freedman}, $\norm{c}_{\infty}\in[0, 1]$, and \pref{lem:e2r} with $\alpha=1$, with probability at least $1-2\delta$:
	\begin{align*}
		\sumt c_t - c(s_t, a_t) = \tilO{\sqrt{\sumt\E[c_t^2]} } = \tilO{ \sqrt{\sumt c(s_t, a_t)} } = \tilO{\sqrt{C_K}}.
	\end{align*}
\end{proof}

The next lemma is used in the proof of \pref{thm:main_regret}, which shows that the sum of the variances of the optimal value function is of order $\tilo{\B C_K}$.
It is also useful in bounding the overhead of Bernstein-style confidence interval (see \pref{lem:sum bt mf} and \cite[Lemma 4.7]{cohen2020near} for example).

\begin{lemma}
	\label{lem:var optV}
	With probability at least $1-2\delta$, $\sumt \fV(P_{s_t, a_t}, \optV)=\tilO{\B^2+\B C_K}$.
\end{lemma}
\begin{proof}
	Note that:
	\begin{align*}
		&\sumt\fV(P_{s_t, a_t}, \optV) = \sumt P_{s_t, a_t}(\optV)^2 - (P_{s_t, a_t}\optV)^2\\
		&= \sumk\sum_{i=1}^{I_k}P_{s^k_i, a^k_i}(\optV)^2 - \optV(s^k_i)^2 + \sumk\sum_{i=1}^{I_k}\optV(s^k_i)^2 - (P_{s^k_i, a^k_i}\optV)^2\\
		&\leq \sumk\sum_{i=1}^{I_k}P_{s^k_i, a^k_i}(\optV)^2 - \optV(s^k_{i+1})^2 + \sumk\sum_{i=1}^{I_k}\optQ(s^k_i, a^k_i)^2 - (P_{s^k_i, a^k_i}\optV)^2. 
		\tag{$\optV(s^k_{I_k+1})=0$ 
	and	$\optV(s^k_i) \leq \optQ(s^k_i, a^k_i)$}\\
	\end{align*}
	For the first term, by \pref{eq:anytime strong freedman} of \pref{lem:any interval freedman} with $\optV(s)\leq \B$ and \pref{lem:var xy} with $X=\optV(S'), S'\sim P_{s_t, a_t}$, we have with probability at least $1-\delta$,
	\begin{align*}
		\sumk\sumi P_{s^k_i, a^k_i}(\optV)^2 - \optV(s^k_{i+1})^2 &= \tilO{ \sqrt{\sumt \fV(P_{s_t, a_t}, (\optV)^2) } + \B^2 }\\
	&= \tilO{ \B\sqrt{ \sumt \fV(P_{s_t, a_t}, \optV) } + \B^2 }.
	\end{align*}
	For the second term, note that:
	\begin{align*}
		&\sumk\sumi \optQ(s^k_i, a^k_i)^2 - (P_{s^k_i, a^k_i}\optV)^2 = \sumk\sumi \rbr{\optQ(s^k_i, a^k_i) - P_{s^k_i, a^k_i}\optV}\rbr{ \optQ(s^k_i, a^k_i) + P_{s^k_i, a^k_i}\optV }\\
		&\leq \sumk\sumi 3\B c(s^k_i, a^k_i). \tag{$\optQ(s, a)\leq 2\B$ and $\optV(s)\leq \B$ for any $(s, a)\in\SA$}
	\end{align*}
	Therefore, $\sumt\fV(P_{s_t, a_t}, \optV) = \tilO{ \B\sqrt{\sumt\fV(P_{s_t, a_t}, \optV)} + \B^2 + \B\sumk\sumi c(s^k_i, a^k_i) }$.
	By \pref{lem:quad} with $x=\sumt\fV(P_{s_t, a_t}, \optV)$ and \pref{lem:e2r}, we have with probability at least $1-\delta$,
	$$\sumt\fV(P_{s_t, a_t}, \optV) = \tilO{ \B^2 + \B \sumk\sumi c(s^k_i, a^k_i) } = \tilO{\B^2 + \B C_K}.$$
\end{proof}

\section{Omitted Details for \pref{sec:mf}}
\label{app:mf}


Before we present the proof of \pref{thm:Q} (\pref{sec:pf thm Q}), we first quantify the sample complexity of the reference value function (\pref{sec:sample complexity mf}) and prove the two required properties (\pref{sec:properties}).

\paragraph{Extra Notations}
Denote by $Q_t(s, a)$, $V_t(s)$, $\refV_t(s)$, $B_t$, $N_t(s, a)$ the value of $Q(s, a)$, $V(s)$, $\refV(s)$, $B$, $N(s, a)$ at the beginning of time step $t$.
Define $N_t(s)=\sum_aN_t(s, a)$.
Denote by $n_t(s, a), m_t(s, a), b_t(s, a), b'_t(s, a), \iota_t(s, a), \hatc_t(s, a)$ the value of $n, m, b, b', \iota, \hatc$ used in computing $Q_t(s, a)$.
Note that, these are \emph{not} necessarily their values at time step $t$.
For example, $n_t(s, a)$ is the number of visits to $(s, a)$ before the current stage (not before time $t$); $m_t(s, a)$ the number of visits to $(s, a)$ in the last stage; $b_t(s, a)$ and $b'_t(s, a)$ are the bonuses used in the last update of $Q_t(s, a)$; and $\hatc_t(s, a)$ is the cost estimator used in the last update of $Q_t(s, a)$ ($b_t(s, a)$, $b'_t(s, a)$ and $\hatc_t(s, a)$ are $0$ when $n_t(s, a)=0$).
Denote by $\lti(s, a)$ the $i$-th time step the agent visits $(s, a)$ among those $n_t(s, a)$ steps before the current stage, and by $\clti(s, a)$ the $i$-th time step the agent visits $(s, a)$ among those $m_t(s, a)$ steps within the last stage.
With these notations, we have by the update rule of the algorithm:
\begin{equation}\label{eq:mf_update_rule_alt}
\begin{split}
Q_t(s, a) &= \max\Bigg\{Q_{t-1}(s, a), \;\; \hatc_t(s, a) + \frac{1}{m_t}\sum_{i=1}^{m_t}V_{\clti}(s'_{\clti}) -b'_t, \\
& \hatc_t(s, a) + \frac{1}{n_t}\sum_{i=1}^{n_t}\refV_{\lti}(s'_{\lti}) + \frac{1}{m_t}\sum_{i=1}^{m_t}(V_{\clti}(s'_{\clti}) - \refV_{\clti}(s'_{\clti})) - b_t \Bigg\},
\end{split}
\end{equation}
where $m_t$ represents $m_t(s, a)$, $\clti$ represents $\clti(s, a)$, and similarly for $n_t$, $\lti$, $b_t$ and $b_t'$.

We also define two empirical variances at time step $t$ as: 
\begin{equation}
\label{eq:nu_t}
\nu_t = \frac{1}{m_t}\sum_{i=1}^{m_t}(V_{\clti}(s'_{\clti}) - \refV_{\clti}(s'_{\clti}))^2 - \rbr{ \frac{1}{m_t}\sum_{i=1}^{m_t}V_{\clti}(s'_{\clti}) - \refV_{\clti}(s'_{\clti}) }^2
\end{equation}
and
\begin{equation}
\label{eq:refnu_t}
\refnu_t = \frac{1}{n_t}\sum_{i=1}^{n_t}\refV_{\lti}(s'_{\lti})^2 - \rbr{ \frac{1}{n_t}\sum_{i=1}^{n_t}\refV_{\lti}(s'_{\lti}) }^2.
\end{equation}
Here, $\nu_t$ and $\refnu_t$ should be treated as a function of state-action pair $(s,a)$, so that $m_t$, $n_t$, $\clti$, and $\lti$ in the formulas all represent $m_t(s,a)$, $n_t(s,a)$, $\clti(s,a)$, and $\lti(s,a)$.
Except for \pref{lem:mf optimistic}, this input $(s,a)$ is simply $(s_t, a_t)$.

Further define $\eps_t=\Ind\{n_t>0\}=\Ind\{m_t>0\}$,
and $0/0$ to be $0$ so that formula in the form $\frac{1}{n_t}\sum_{i=1}^{n_t}X_{\lti}$ is treated as $0$ if $n_t=0$ (similarly for $m_t$).

\subsection{Sample Complexity for Reference Value Function}
\label{sec:sample complexity mf}

In this section, we assume $H=\ceil{\frac{4\B}{\cmin}\ln(\frac{2}{\beta})+1}_2$ for some $\beta>0$ (the form used in \pref{thm:mf_properties}).
We show that to obtain a reference value with precision $\rho\geq 2\B\beta$ at state $s$ (that is, $|\refV(s) - \optV(s)|\leq\rho$), $\tilO{\frac{\B^2H^3SA}{\rho^2}}$ number of visits to state $s$ is sufficient (\pref{cor:estimate refV}).
Moreover, the total costs appeared in regret for a reference value function with maximum precision $\rho$ is $\tilO{\frac{\B^2H^3S^2A}{\rho}}$ (\pref{lem:refV}).
Note that if we only update the reference value function once as in \cite{zhang2020almost}, instead of applying our ``smoother'' update, the total costs become $\tilO{\frac{\B^2H^3S^2A}{\rho^2}}$.

\begin{lemma}
	\label{lem: loop-free approx}
	With probability at least $1-8\delta$, \pref{alg:Q} ensures for any non-negative weights $\{w_t\}_{t=1}^T$, 
	\begin{align*}
		\sumt w_t(\optQ(s_t, a_t) - Q_t(s_t, a_t)) \leq \B\norm{w}_1\beta + \tilO{ H^2SA\B\norm{w}_{\infty} + \B\sqrt{H^3SA\norm{w}_{\infty}\norm{w}_1} }.
	\end{align*}
\end{lemma}
\begin{proof}
	Define $w^{(0)}_t = w_t$ and $w^{(h+1)}_{t+1} = \sum_{t'=1}^T\sum_{i=1}^{m_{t'}}\frac{w^{(h)}_{t'}}{m_{t'}}\Ind\{t=\check{l}_{t',i}\}$.
	We first argue the following properties related to $w^{(h)}_{t}$ and vector $w^{(h)} = (w^{(h)}_1, \ldots, w^{(h)}_{T})$.
	Denote by $j_t$ the stage to which time step $t$ belongs.
	When $t=\clti[t',i]$, we have $m_{t'}=e_{j_t}$.
	Therefore, 
	$$\sum_{t'=1}^T\sum_{i=1}^{m_{t'}}\frac{1}{m_{t'}}\Ind\{t=\check{l}_{t',i}\} \leq \frac{e_{j_t+1}}{e_{j_t}} \leq 1 + \frac{1}{H},$$
	and thus, $\norm{w^{(h)}}_{\infty} \leq (1+\frac{1}{H})\norm{w^{(h-1)}}_{\infty} \leq \cdots \leq (1+\frac{1}{H})^h\norm{w}_{\infty}$.
	Moreover,
	\begin{align*}
		&\norm{w^{(h+1)}}_1 = \sumt\sum_{t'=1}^T\sum_{i=1}^{m_{t'}}\frac{w^{(h)}_{t'}}{m_{t'}}\Ind\{t=\clti[t',i]\} = \sum_{t'=1}^Tw^{(h)}_{t'}\sum_{i=1}^{m_{t'}}\sumt\frac{\Ind\{t=\clti[t',i]\}}{m_{t'}} \leq \norm{w^{(h)}}_1,
	\end{align*}
	and thus $\norm{w^{(h)}}_1\leq\norm{w}_1$ for any $h$.
	Also note that for any $\{X_t\}_t$ such that $X_t\geq 0$:
	\begin{align}
		\label{eq:w}
		\sumt\frac{w^{(h)}_t}{m_t}\sum_{i=1}^{m_t}X_{\clti}=\sum_{t'=1}^T\sumt\frac{w^{(h)}_t}{m_t}\sum_{i=1}^{m_t}X_{t'}\Ind\{t'=\clti\}=\sum_{t'=1}^Tw^{(h+1)}_{t'+1}X_{t'}.
	\end{align}
	
	Next, for a fixed $(s, a)$, by \pref{lem:anytime bernstein}, with probability at least $1-\frac{\delta}{SA}$, when $n_t(s, a)>0$:
	\begin{equation}
		\label{eq:c-hatc}
		\abr{c(s, a) - \hatc_t(s, a)} \leq 2\sqrt{\frac{2\hatc_t(s, a)}{n_t(s, a)}\ln\frac{2SAn_t(s, a)}{\delta}} + \frac{19\ln\frac{2SAn_t(s, a)}{\delta}}{n_t(s, a)} \leq \sqrt{\frac{\hatc_t(s, a)\iota_t}{n_t(s, a)}} + \frac{\iota_t}{n_t(s, a)}.
	\end{equation}
	Taking a union bound, we have \pref{eq:c-hatc} holds for all $(s, a)$ when $n_t(s, a)>0$ with probability at least $1-\delta$.
	Then by definition of $b'_t$, we have
	\begin{equation}
		\label{eq:c-hatc b}
		c(s_t, a_t)-\hatc_t(s_t, a_t)\leq\Ind\{m_t=0\}+b'_t.
	\end{equation}
	Now we are ready to prove the lemma.
	First, we condition on \pref{lem:mf optimistic}, which happens with probability at least $1-7\delta$.
	Then for any $h\in\{0,\ldots, H-1\}, \Qref=Q_{H-h}, \Vref=Q_{H-h-1}$ we have:
	\begin{align*}
		&\sumt w^{(h)}_t(\Qref(s_t, a_t) - Q_t(s_t, a_t))_+\\
		&\leq \sumt w^{(h)}_t(c(s_t, a_t)-\hatc_t(s_t, a_t))_+ + w^{(h)}_t\rbr{ P_t\Vref - \frac{1}{m_t}\sum_{i=1}^{m_t}V_{\clti}(s'_{\clti}) }_+ + w^{(h)}_tb'_t 
		\tag{by \pref{eq:mf_update_rule_alt} and $\Qref(s,a)=c(s,a)+P_{s,a}\Vref$}\\
		&\leq \sumt 2\B w^{(h)}_t\Ind\{m_t=0\} + \sumt w^{(h)}_t\rbr{ \frac{1}{m_t}\sum_{i=1}^{m_t}P_{\clti}\Vref - \frac{1}{m_t}\sum_{i=1}^{m_t}V_{\clti}(s'_{\clti}) }_+ + 2w^{(h)}_tb'_t \tag{\pref{eq:c-hatc b}, $P_t=P_{\clti}$ and $P_t\Vref\leq \B\Ind\{m_t=0\} + \frac{1}{m_t}\sum_{i=1}^{m_t}P_{\clti}\Vref$}.
	\end{align*}
	Since $e_1=H$, we have $\sumt w^{(h)}_t\Ind\{m_t=0\}\leq SAH\norm{w^{(h)}}_{\infty}$.
	Moreover, by \pref{eq:anytime strong freedman} of \pref{lem:any interval freedman} with $X_t=\Vref(s'_t)$, we have with probability at least $1-\frac{\delta}{H}$: $\frac{1}{m_t}\sum_{i=1}^{m_t}P_{\clti}\Vref \leq \frac{1}{m_t}\sum_{i=1}^{m_t}\Vref(s'_{\clti}) + \tilO{\frac{\B\eps_t}{\sqrt{m_t}}}$.
	Plugging these back to the previous inequality and using the definition of $b'_t$ gives:
	\begin{align*}
		&\sumt w^{(h)}_t(\Qref(s_t, a_t) - Q_t(s_t, a_t))_+\\
		&\leq 2HSA\B\norm{w^{(h)}}_{\infty} + \sumt \frac{w^{(h)}_t}{m_t}\sum_{i=1}^{m_t}\rbr{ \Vref(s'_{\clti}) - V_{\clti}(s'_{\clti}) }_+ + \tilO{\frac{\B w^{(h)}_t\eps_t}{\sqrt{m_t}} + \frac{w^{(h)}_t\eps_t}{n_t} } \\
		&\leq 3HSA\B\norm{w^{(h)}}_{\infty} + \tilO{\B\sqrt{HSA\norm{w^{(h)}}_{\infty}\norm{w}_1}} + \sumt w^{(h+1)}_{t+1}\rbr{ \Vref(s'_t) - V_t(s'_t) }_+ \tag{\pref{eq:w} and \pref{lem:sum of count}}\\
		&\leq \tilO{HSA\B\norm{w^{(h)}}_{\infty} + \B\sqrt{HSA\norm{w^{(h)}}_{\infty}\norm{w}_1}} + \sumt w^{(h+1)}_t(\Qref(s_t, a_t) - Q_t(s_t, a_t))_+,
	\end{align*}
	where in the last inequality we apply:
	\begin{align*}
		&\sumt w^{(h+1)}_{t+1}\rbr{ \Vref(s'_t) - V_t(s'_t) }_+ \leq \sumt w^{(h+1)}_{t+1}(\Vref(s'_t) - V_{t+1}(s'_t))_+ + \tilO{\norm{w^{(h)}}_{\infty}S\B} \tag{apply \pref{lem:t diff} on $\sumt V_{t+1}(s'_t) - V_t(s'_t)$}\\
		&\leq \sumt w^{(h+1)}_t(\Vref(s_t) - V_t(s_t))_+ + \tilO{\norm{w^{(h)}}_{\infty}S\B} \tag{$(\Vref(s'_t)-V_{t+1}(s'_t))_+ \leq (\Vref(s_{t+1})-V_{t+1}(s_{t+1}))_+$ and $w_{T+1}^{(h+1)}=0$}\\
		&\leq \sumt w_t^{(h+1)}(\Qref(s_t, a_t) - Q_t(s_t, a_t))_+ + \tilO{\norm{w^{(h)}}_{\infty}S\B}. \tag{$\Vref(s_t)\leq\Qref(s_t, a_t)$ and $V_t(s_t)=Q_t(s_t, a_t)$}
	\end{align*}
	By a union bound, the inequality above holds for $\Qref=Q_{H-h}, \Vref=Q_{H-h-1}$ for all $h\in\{0,\ldots, H-1\}$ with probability at least $1-\delta$.
	Applying the inequality above recursively starting from $h=0$, and by $\optQ_0(s, a)-Q_t(s, a)\leq 0$, $(1+\frac{1}{H})^H\leq 3$:
	$$\sumt w_t(\optQ_H(s_t, a_t) - Q_t(s_t, a_t))_+ = \tilO{ H^2SA\B\norm{w}_{\infty} + \B\sqrt{H^3SA\norm{w}_{\infty}\norm{w}_1} }.$$
	Therefore, by \pref{lem:loop free approx},
	\begin{align*}
		\sumt w_t(\optQ(s_t, a_t) - Q_t(s_t, a_t)) &= \sumt w_t( \optQ(s_t, a_t) - \optQ_H(s_t, a_t) + \optQ_H(s_t, a_t) - Q_t(s_t, a_t) )\\
		&\leq \B\norm{w}_1\beta + \tilO{ H^2SA\B\norm{w}_{\infty} + \B\sqrt{H^3SA\norm{w}_{\infty}\norm{w}_1} }.
	\end{align*}
\end{proof}

Now by \pref{lem: loop-free approx} with $w_t=\Ind\{\optV(s_t)-V_t(s_t)\geq\rho\}$ for some threshold $\rho$, we can bound the sample complexity of obtaining a value function with precision $\rho$ (\pref{cor:estimate refV}), which is used to determine the value of $\thetastar$ (\pref{lem:refV}).
However, one caveat here is that the bound in \pref{lem: loop-free approx} has logarithmic dependency on $T$ from $\iota_t$, which should not appear in the definition of $\thetastar$ since $T$ is a random variable.
To deal with this, we obtain a loose bound on $T$ in the following lemma. 
\begin{lemma}
	\label{lem:bound T}
	With probability at least $1-13\delta$, $T=\tilo{\B K/\cmin + \B^2H^3SA/\cmin^2}$.
\end{lemma}
\begin{proof}
	By \pref{lem: loop-free approx} with $w_t=1$, we have with probability at least $1-8\delta$:
	\begin{align*}
		\sumt \optQ(s_t, a_t) - Q_t(s_t, a_t) = \B T\beta + \tilO{ H^2SA\B + \B\sqrt{H^3SAT} }.
	\end{align*}
	Now by \pref{eq:reg}, \pref{lem:cost diff}, \pref{lem:any interval freedman}, and \pref{lem:var optV}, with probability at least $1-5\delta$,
	\begin{align*}
		R_K &\leq \sumt (c_t - c(s_t, a_t)) + \sumt (\optV(s'_t) - P_{s_t, a_t}\optV) + \sumt (\optQ(s_t, a_t) - \optV(s_t))\\
		&\leq \tilo{\sqrt{\B C_K} + \B} + \sumt (\optQ(s_t, a_t) - Q_t(s_t, a_t)) \tag{$V_t=Q_t(s_t, a_t)$ and \pref{lem:mf optimistic}}\\
		&= \B T\beta + \tilO{ H^2SA\B + \B\sqrt{H^3SAT} }. \tag{$C_K\leq T$}
	\end{align*}
	Further using $\cmin T - K\B\leq R_K$, $\B\beta\leq \frac{\cmin}{2}$, and \pref{lem:quad} proves the statement.
\end{proof}

\begin{cor}
	\label{cor:estimate refV}
	With probability at least $1-13\delta$, \pref{alg:Q} ensures for any $\rho\geq 2\B\beta$:
	\begin{align*}
		\sum_{t=1}^T\Ind\cbr{\optV(s_t) - V_t(s_t) \geq \rho } = \tilO{ \frac{\B^2H^3SA}{\rho^2} } \triangleq U_{\rho} - 1,
	\end{align*}
	and for any $s\in\calS$, $N_t(s)\geq U_{\rho}$ implies $0 \leq \optV(s) - V_t(s) \leq \rho$.
\end{cor}
\begin{proof}
	We can assume $\rho\leq\B$ since $\sumt\Ind\{\optV(s_t)-V_t(s_t)\geq\rho\}=0$ when $\rho>\B$.
	By \pref{lem: loop-free approx} with $w_t=\Ind\{ \optV(s_t) - V_t(s_t) \geq \rho \}$, $\rho w_t \leq w_t(\optV(s_t)-V_t(s_t))$, $\rho\geq 2\B\beta$, and $\optV(s_t)-V_t(s_t)\leq\optQ(s_t, a_t)-Q_t(s_t, a_t)$, we have with probability at least $1-8\delta$:
	\begin{align*}
		\rho\norm{w}_1 &\leq \sumt w_t(\optV(s_t) - V_t(s_t)) \leq \frac{\rho}{2}\norm{w}_1 +  \tilO{ H^2SA\B + \B\sqrt{H^3SA\norm{w}_1} }.
	\end{align*}
	Therefore, by \pref{lem:quad} and \pref{lem:bound T},
	$\norm{w}_1 = \tilO{ \frac{H^2SA\B}{\rho} + \frac{\B^2H^3SA}{\rho^2} }$, which has no logarithmic dependency on $T$.
	We prove the second statement by contradiction: suppose $N_t(s)\geq U_{\rho}$ and $\optV(s) - V_t(s) > \rho$.
	Then since $V_t$ is non-decreasing in $t$, $N_t(s)\leq \norm{w}_1$.
	Thus, $U_{\rho} \leq N_t(s) \leq \norm{w}_1 < U_{\rho}$, a contradiction.
\end{proof}

\begin{lemma}
	\label{lem:refV}
	Define $\beta_i=\frac{\B}{2^i}, \tilN_0=0, \tilN_i=U_{\beta_i}$ (defined in \pref{cor:estimate refV}) for $i\geq 1$ and $\optq=\inf\{i: \beta_i\leq\cmin\}$.
	Define $\RefV=\refV_{T+1}, \thetastar=\ceil{\tilN_{\optq}}_2$, and $\refB_t$ such that:
	\begin{align*}
		\refB_t(s) &= \sum_{i=1}^{\optq}\beta_{i-1}\Ind\{\ceil{\tilN_{i-1}}_2 \leq N_t(s) < \ceil{\tilN_i}_2\}.
	\end{align*}
	Then with probability at least $1-13\delta$,  $\RefV(s) - \refV_t(s)\leq \refB_t(s)$, and
	\begin{align*}
		\sumt \RefV(s_t) - \refV_t(s_t) &\leq \sumt\refB_t(s_t) = \tilO{ \frac{\B^2H^3S^2A}{\cmin} } \triangleq \RefC,\\
		\sumt \rbr{\RefV(s_t) - \refV_t(s_t)}^2 &\leq \sumt\refB_t(s_t)^2 = \tilO{ \B^2H^3 S^2A } \triangleq \RefCs.
	\end{align*}
\end{lemma}
\begin{proof}
	We condition on \pref{cor:estimate refV}, which happens with probability at least $1-13\delta$.
	By \pref{cor:estimate refV} with $\rho=\beta_i$ for each $i\in[\optq]$, we have $\RefV(s) - \refV_t(s)\leq \refB_t(s)$.
	Moreover, $\refB_t(s)^2=\sum_{i=1}^{\optq}\beta_{i-1}^2\Ind\{\ceil{\tilN_{i-1}}_2 \leq N_t(s) < \ceil{\tilN_i}_2\}$.
	Thus,
	\begin{align*}
		&\sumt \refB_t(s_t) \leq \sum_s\sum_{i=1}^{\optq}\beta_{i-1}\ceil{\tilN_i}_2 = \tilO{ \sum_s\sum_{i=1}^{\optq} \frac{\B^2H^3SA}{\beta_i}  } = \tilO{ \frac{\B^2H^3S^2A}{\beta_{\optq}} }.\\
		&\sumt \refB_t(s_t)^2 \leq \sum_s\sum_{i=1}^{\optq}\beta_{i-1}^2\ceil{\tilN_i}_2 = \tilO{ \sum_s\sum_{i=1}^{\optq}\B^2H^3SA } = \tilO{ \B^2H^3 S^2A }.
	\end{align*}
\end{proof}

\subsection{Proofs of Required Properties}
\label{sec:properties}

In this section, we prove \pref{prop:optimism} and \pref{prop:recursion} of \pref{alg:Q}.

\begin{lemma}
	\label{lem:mf optimistic}
	With probability at least $1-7\delta$, \pref{alg:Q} ensures $Q_t(s, a) \leq Q_{t+1}(s, a)\leq \optQ(s, a)$ for any $(s, a)\in\SA, t\geq 1$.
\end{lemma}
\begin{proof} 
	We fix a pair $(s, a)$, and denote $n_t, m_t, \lti, \clti, b_t, b'_t, \iota_t$ as shorthands of the corresponding functions evaluated at $(s, a)$.
	The first inequality is by the update rule of $Q_t$.
	Next, we prove $Q_t(s, a)\leq \optQ(s, a)$ by induction on $t$.
	It is clearly true when $t=1$.
	For the induction step, the statement is clearly true when $n_t=m_t=0$. 
	When $n_t>0$, it suffices to consider two update rules, that is, the last two terms in the max operator of \pref{eq:mf_update_rule_alt}.
	For the second update rule, note that,
	\begin{align}
		&\hatc_t(s, a) + \frac{1}{n_t}\sum_{i=1}^{n_t}\refV_{\lti}(s'_{\lti}) + \frac{1}{m_t}\sum_{i=1}^{m_t}\rbr{V_{\clti}(s'_{\clti}) - \refV_{\clti}(s'_{\clti})} - b_t \notag\\
		&= \hatc_t(s, a) + \frac{1}{n_t}\sum_{i=1}^{n_t}P_{s, a}\refV_{\lti} + \frac{1}{m_t}\sum_{i=1}^{m_t}P_{s, a}\rbr{V_{\clti} - \refV_{\clti}} \notag\\
		&\qquad\qquad + \underbrace{\frac{1}{n_t}\sum_{i=1}^{n_t}\rbr{\Ind_{s'_{\lti}} - P_{s, a} }\refV_{\lti}}_{\chi_1} + \underbrace{\frac{1}{m_t}\sum_{i=1}^{m_t}\rbr{ \Ind_{s'_{\clti}} - P_{s, a} }\rbr{V_{\clti} - \refV_{\clti}}}_{\chi_2} - b_t. \label{eq:type 2}
	\end{align}
	Define $C'_t=\ceil{\ln (\B^4n_t)}^2 \leq \min\{4\ln^2(\B^4n_t), \B^8n_t^2\}$ (in general, we can set $C'_t=\ceil{\ln (\tilB^4n_t)}^2$ for some $\tilB\geq\B$).
	For $\chi_1$, by \pref{eq:anytime strong freedman} of \pref{lem:any interval freedman} with $b=\B^2$ and $C \leq C'_t$, we have with probability at least $1-\frac{\delta}{SA}$:
	\begin{align*}
		|\chi_1|=\abr{\frac{1}{n_t}\sum_{i=1}^{n_t}\rbr{\Ind_{s'_{\lti}} - P_{s, a} }\refV_{\lti}} \leq 4\ln^3\rbr{\frac{4SA\B^8n_t^5}{\delta}}\rbr{\sqrt{\frac{8\sum_{i=1}^{n_t}\fV(P_{s, a}, \refV_{\lti})}{n^2_t}} + \frac{5B_t}{n_t}},
	\end{align*}
	Note that (recall that $\refnu_t$ represents $\refnu_t(s, a)$)
	\begin{equation}
		\label{eq:ref - hatref}
		\frac{1}{n_t}\sum_{i=1}^{n_t}\fV(P_{s, a}, \refV_{\lti}) - \refnu_t = \chi_3 + \chi_4 + \chi_5,
	\end{equation}
	where
	\begin{align*}
		\chi_3 &= \frac{1}{n_t}\sum_{i=1}^{n_t}\rbr{P_{s, a}(\refV_{\lti})^2 - \refV_{\lti}(s'_{\lti})^2},\quad \chi_4 = \rbr{\frac{1}{n_t}\sum_{i=1}^{n_t}\refV_{\lti}(s'_{\lti})}^2 - \rbr{\frac{1}{n_t}\sum_{i=1}^{n_t}P_{s, a}\refV_{\lti}}^2,\\
		\chi_5 &= \rbr{\frac{1}{n_t}\sum_{i=1}^{n_t}P_{s, a}\refV_{\lti}}^2 - \frac{1}{n_t}\sum_{i=1}^{n_t}(P_{s, a}\refV_{\lti})^2.
	\end{align*}
	By \pref{eq:anytime strong freedman} of \pref{lem:any interval freedman} with $b=\B^2$ and $C \leq C'_t$, and \pref{lem:var xy} with $\norm{\refV_{\lti}}_{\infty}\leq B_t$, with probability at least $1-\frac{2\delta}{SA}$,
	\begin{align}
		|\chi_3| &\leq \frac{4\ln^3(4SA\B^8n_t^5/\delta)}{n_t}\rbr{\sqrt{8\sum_{i=1}^{n_t}\fV(P_{s, a}, (\refV_{\lti})^2)} + 5B_t^2 } \notag\\ 
		&\leq \frac{4\ln^3(4SA\B^8n_t^5/\delta)}{n_t}\rbr{2B_t\sqrt{8\sum_{i=1}^{n_t}\fV(P_{s, a}, \refV_{\lti})} + 5B_t^2 }. \label{eq:avg refV s}\\
		|\chi_4| &\leq \abr{ \frac{1}{n_t}\sum_{i=1}^{n_t}\refV_{\lti}(s'_{\lti}) + \frac{1}{n_t}\sum_{i=1}^{n_t}P_{s, a}\refV_{\lti} }\abr{ \frac{1}{n_t}\sum_{i=1}^{n_t}\refV_{\lti}(s'_{\lti}) - \frac{1}{n_t}\sum_{i=1}^{n_t}P_{s, a}\refV_{\lti}} \notag\\
		&\leq 2B_t\cdot\frac{4\ln^3(4SA\B^8n_t^5/\delta)}{n_t}\rbr{\sqrt{8\sum_{i=1}^{n_t}\fV(P_{s, a}, \refV_{\lti})} + 5B_t }. \label{eq:s avg refV}
	\end{align}
	Moreover, $\chi_5\leq 0$ by Cauchy-Schwarz inequality.
	Therefore, 
	$$\frac{1}{n_t}\sum_{i=1}^{n_t}\fV(P_{s, a}, \refV_{\lti}) - \refnu_t \leq \frac{4B_t\ln^3(4SA\B^8n_t^5/\delta)}{n_t}\rbr{4\sqrt{8 \sum_{i=1}^{n_t}\fV(P_{s, a}, \refV_{\lti}) } + 15B_t }.$$
	 Applying \pref{lem:quad} with $x=\sum_{i=1}^{n_t}\fV(P_{s, a}, \refV_{\lti})$, we obtain:
	 \begin{align*}
	 	\frac{1}{n_t}\sum_{i=1}^{n_t}\fV(P_{s, a}, \refV_{\lti}) \leq 2\refnu_t + \frac{4216B_t^2\ln^6\frac{4SA\B^8n_t^5}{\delta}}{n_t}.
	 \end{align*}
	Thus, $\abr{\frac{1}{n_t}\sum_{i=1}^{n_t}\rbr{ \Ind_{s'_{\lti}} - P_{s, a} }\refV_{\lti}} \leq \sqrt{\frac{ \refnu_t }{n_t}\iota_t} + \frac{3B_t\iota_t}{n_t}$. 
	By similar arguments, $|\chi_2| \leq \sqrt{\frac{ \nu_t }{m_t}\iota_t} + \frac{3B_t\iota_t}{m_t}$ with probability at least $1-\frac{3\delta}{SA}$.
	Finally, by \pref{eq:c-hatc} and $B_t\geq 1$, we have $\hatc_t(s, a)-c(s, a)\leq \sqrt{\frac{\hatc_t(s, a)\iota}{n_t}} + \frac{B_t\iota}{n_t}$.
	Therefore,
	\begin{equation}
		\label{eq:chi}
		|\hatc_t(s, a) - c(s, a) | + |\chi_1| + |\chi_2|\leq b_t.
	\end{equation}
	Plugging \pref{eq:chi} back to \pref{eq:type 2}, and by the non-decreasing property of $\refV_t$ and $V_{\clti}(s) \leq \optV(s)$ for any $s\in\calS^+$:
	\begin{align*}
		&\hatc_t(s, a) + \frac{1}{n_t}\sum_{i=1}^{n_t}\refV_{\lti}(s'_{\lti}) + \frac{1}{m_t}\sum_{i=1}^{m_t}\rbr{V_{\clti}(s'_{\clti}) - \refV_{\clti}(s'_{\clti})} - b_t\\
		&\leq c(s, a) + \frac{1}{n_t}\sum_{i=1}^{n_t}P_{s, a}\refV_{\lti} + \frac{1}{m_t}\sum_{i=1}^{m_t}P_{s, a}\rbr{V_{\clti} - \refV_{\clti}} \leq c(s, a) + P_{s, a}\optV = \optQ(s, a).
	\end{align*}
	For the first update rule, by \pref{eq:anytime strong freedman} of \pref{lem:any interval freedman} with $b=K$ and $C\leq C'_t$, with probability at least $1-\frac{\delta}{SA}$, $\frac{1}{m_t}\sum_{i=1}^{m_t}V_{\clti}(s'_{\clti}) - P_{\clti}V_{\clti} \leq 2\sqrt{\frac{B_t^2\iota_t}{m_t}}$.
	Therefore, by \pref{eq:c-hatc}:
	$$\hatc_t(s, a) + \frac{1}{m_t}\sum_{i=1}^{m_t}V_{\clti}(s'_{\clti}) - b'_t \leq c(s, a) + \frac{1}{m_t}\sum_{i=1}^{m_t}P_{\clti}V_{\clti} \leq c(s, a) + P_{s, a}\optV = \optQ(s, a).$$
	Combining two cases, we have $Q_t(s, a)\leq\optQ(s, a)$ for the fixed $(s, a)$.
	By a union bound over $(s, a)\in\SA$, we have $Q_t(s, a)\leq \optQ(s, a)$ for any $(s, a)\in\SA, t\geq 1$.
\end{proof}

\begin{remark}
	\label{rem:tilB}
	Note that the statement of \pref{lem:mf optimistic} still holds if we use ``compute $\iota\leftarrow 256\ln^6(4SA\tilB^8n^5/\delta)$'' in \pref{line:compute} of \pref{alg:Q} for some $\tilB\geq\B$.
	This is useful in deriving the parameter-free version of \pref{alg:Q} in \pref{sec:parameter free}; see \pref{line:compute pf B} of \pref{alg:Q-B}.
\end{remark}

\begin{proof}[\pfref{thm:mf_properties}] 
	\pref{prop:optimism} is satisfied by \pref{lem:mf optimistic}.
	For \pref{prop:recursion}, we conditioned on \pref{lem:mf optimistic}, \pref{lem:refV}, \pref{lem:refV diff}, and \pref{lem:sum bt mf}, which holds with probability at least $1-50\delta$.
	Then, for any $(\Qref, \Vref)\in\calV_H$:
	\begin{align*}
		&\sumt (\Qref(s_t, a_t) - Q_t(s_t, a_t))_+\\
		&\leq \sumt\rbr{c(s_t, a_t) - \hatc_t(s_t, a_t) + P_t\Vref - \frac{1}{n_t}\sum_{i=1}^{n_t}\refV_{\lti}(s'_{\lti}) - \frac{1}{m_t}\sum_{i=1}^{m_t}\rbr{V_{\clti}(s'_{\clti}) - \refV_{\clti}(s'_{\clti})} + b_t}_+ \tag{by \pref{eq:mf_update_rule_alt} and $\Qref(s,a)=c(s,a)+P_{s,a}\Vref$}\\
		&\leq \sumt2\B\Ind\{m_t=0\} + \sumt \rbr{\frac{1}{m_t}\sum_{i=1}^{m_t}P_{\clti}\Vref - \frac{1}{n_t}\sum_{i=1}^{n_t}P_{\lti}\refV_{\lti} - \frac{1}{m_t}\sum_{i=1}^{m_t}P_{\clti}\rbr{V_{\clti} - \refV_{\clti}} }_+ + 2b_t \tag{$P_t\Vref\leq \B\Ind\{m_t=0\}+ \frac{1}{m_t}\sum_{i=1}^{m_t}P_{\clti}\Vref$ and \pref{eq:chi}}\\
		&\leq 2\B HSA + \sumt \frac{1}{n_t}\sum_{i=1}^{n_t}P_{\lti}\rbr{\RefV - \refV_{\lti}} + \frac{1}{m_t}\sum_{i=1}^{m_t}P_{\clti}(\Vref-V_{\clti})_+ + 2b_t. \tag{$\sumt \Ind\{m_t=0\}\leq SAH$, $P_t=P_{\lti}=P_{\clti}$, and $\refV_{\clti}(s)\leq \RefV(s)$ for any $s\in\calS$ (\pref{lem:refV})}
	\end{align*}
	By \pref{lem:stage sum} and \pref{lem:refV diff},
	\[
		\sumt \frac{1}{n_t}\sum_{i=1}^{n_t}P_{\lti}\rbr{\RefV - \refV_{\lti}} = \tilO{ \sumt P_t(\RefV - \refV_t) } = \tilO{\RefC}.
	\]
	Moreover, by \pref{lem:p2a}, with probability at least $1-\frac{\delta}{H+1}$,
	\[
		\frac{1}{m_t}\sum_{i=1}^{m_t}P_{\clti}(\Vref-V_{\clti})_+ \leq \rbr{1+\frac{1}{H}}^2\sumt(\Vref(s_t) - V_t(s_t))_+ + \tilO{\B (H+S)}.
	\]
	Plugging these back, and by $(1+\frac{1}{H})^2\leq 1+\frac{3}{H}$, \pref{lem:sum bt mf} and \pref{lem:refV}, we get:
	\begin{align*}
		&\sumt (\Qref(s_t, a_t) - Q_t(s_t, a_t))_+ \leq \tilO{\B HSA + \RefC} + \rbr{1+\frac{1}{H}}^2\sumt (\Vref(s_t) -V_t(s_t))_+ + 2\sumt b_t\\
		&\leq \rbr{1+\frac{3}{H}}\sumt (\Vref(s_t) -V_t(s_t))_+ + \tilO{ \sqrt{\B SAC_K} + \sqrt{SAH\cmin C_K} + \frac{\B^2H^3S^2A}{\cmin} }. 
	\end{align*}
	Taking a union bound over $(\Qref, \Vref)\in\calV_H$ and using $H=\tilO{\frac{\B}{\cmin}}$ proves the claim.
\end{proof}
	
\subsection{\pfref{thm:Q}}
\label{sec:pf thm Q}
	
\begin{proof}
	By \pref{thm:main_regret} and \pref{thm:mf_properties}, with probability at least $1-60\delta$ and $\beta=\frac{\cmin}{2\B^2 SAK}$:
	\begin{align*}
		C_K - K\optV(\sinit) =  R_K \leq \tilO{\beta C_K + \sqrt{\B SAC_K} + \frac{\B^2H^3S^2A}{\cmin} }.
	\end{align*}
	Then by $\optV(\sinit)\leq\B, \beta\leq\frac{1}{2}$ and \pref{lem:quad}, we have $C_K=\tilO{\B K}$.
	Substituting this back and by $\beta\leq\frac{\cmin}{\B K}, H=\tilo{\B/\cmin}$, we get $R_K = \tilO{ \B\sqrt{SAK} + \frac{\B^5S^2A}{\cmin^4} }$.
\end{proof}

\subsection{Extra Lemmas for \pref{sec:mf}}

In this section, we gives proofs of auxiliary lemmas used in \pref{sec:mf}.
\pref{lem:refV diff} quantifies the cost of using reference value function.
\pref{lem:sum bt mf} quantifies the cost of using the variance-aware bonus terms $b_t$.
\pref{lem:stage sum}, \pref{lem:p2a}, and \pref{lem:sum of count} deal with the bias induced by the sparse update scheme.

\begin{lemma}
	\label{lem:refV diff}
	With probability at least $1-9\delta$, $\sumt P_t\rbr{\RefV - \refV_t} \leq \sumt P_t\refB_t = \tilO{\RefC}$, where $\RefC$ is defined in \pref{lem:refV}.
\end{lemma}
\begin{proof}
	 By \pref{lem:refV}, \pref{lem:e2r}, \pref{lem:t diff} and $\refB_{t+1}(s'_t)\leq\refB_{t+1}(s_{t+1})$ in each step:
	\begin{align*}
		\sumt P_t\rbr{ \RefV - \refV_t } &\leq \sumt P_t\refB_t \leq 2\sumt \refB_t(s'_t) + \tilO{\B}\\
		&= \tilO{ \sumt \refB_t(s_t) + S\B } = \tilO{\RefC}. 
	\end{align*}
\end{proof}

\begin{lemma}
	\label{lem:sum bt mf}
	With probability at least $1-21\delta$, 
	$$\sumt b_t = \tilO{ \sqrt{\B SAC_K} + \B H^2S^{\frac{3}{2}}A + \sqrt{SAH\cmin C_K} }.$$
\end{lemma}
\begin{proof}
	We condition on \pref{lem:refV}, which holds with probability at least $1-8\delta$.
	By \pref{eq:sum w_t=1} and \pref{eq:sum w_t=1, 1} of \pref{lem:sum of count},
	\begin{align*}
		\sumt b_t &\leq \sumt \sqrt{\frac{\refnu_t\eps_t}{n_t}\iota_t} + \sqrt{\frac{\nu_t\eps_t}{m_t}\iota_t} + \B\sum_t\rbr{\frac{4\eps_t}{n_t} + \frac{3\eps_t}{m_t}}\iota_t + \sqrt{\frac{\hatc_t\eps_t\iota_t}{n_t}}\\
		&= \tilO{ \sumt \sqrt{\frac{\refnu_t\eps_t}{n_t}} + \sqrt{\frac{\nu_t\eps_t}{m_t}}  + \B HSA + \sqrt{\frac{\hatc_t\eps_t}{n_t}} }.
	\end{align*}
	Note that by \pref{eq:ref - hatref}, \pref{eq:avg refV s} and \pref{eq:s avg refV}, when $n_t>0$, with probability at least $1-2\delta$,
	\begin{align*}
		&\refnu_t - \frac{1}{n_t}\sum_{i=1}^{n_t}\fV(P_{\lti}, \refV_{\lti}) \leq |\chi_3| + |\chi_4| - \chi_5\\
		&\leq \tilO{ \frac{B_t}{n_t}\sqrt{ \sum_{i=1}^{n_t}\fV(P_{\lti}, \refV_{\lti}) } + \frac{B_t^2}{n_t} } + \frac{1}{n_t}\sum_{i=1}^{n_t}(P_{\lti}\refV_{\lti})^2 - \rbr{\frac{1}{n_t}\sum_{i=1}^{n_t}P_{\lti}\refV_{\lti}}^2 \\
		&\overset{\text{(i)}}{=} \tilO{ \frac{B_t}{n_t}\sqrt{ \sum_{i=1}^{n_t}\fV(P_{\lti}, \refV_{\lti}) } + \frac{B_t^2}{n_t} + \frac{\B}{n_t}\sum_{i=1}^{n_t}P_{\lti}\refB_{\lti} }\\
		&\leq \frac{1}{n_t}\sum_{i=1}^{n_t}\fV(P_{\lti}, \refV_{\lti}) + \tilO{\frac{B_t^2}{n_t} + \frac{\B}{n_t}\sum_{i=1}^{n_t}P_{\lti}\refB_{\lti} }, \tag{AM-GM Inequality}
	\end{align*}
	where in (i) we apply:
	\begin{align*}
		&\frac{1}{n_t}\sum_{i=1}^{n_t}(P_{\lti}\refV_{\lti})^2 - \rbr{\frac{1}{n_t}\sum_{i=1}^{n_t}P_{\lti}\refV_{\lti}}^2 \leq (P_t\RefV)^2 - \rbr{\frac{1}{n_t}\sum_{i=1}^{n_t}P_{\lti}\refV_{\lti}}^2 \tag{$\refV_{\lti}(s)\leq\RefV(s)$ for any $s\in\calS$}\\
		&\leq \frac{2\B}{n_t}\sum_{i=1}^{n_t}P_{\lti}\rbr{ \RefV - \refV_{\lti} } \leq \frac{2\B}{n_t}\sum_{i=1}^{n_t}P_{\lti}\refB_{\lti}. \tag{$\norm{\RefV}_{\infty} \leq \B$ and \pref{lem:refV}}
	\end{align*}
	Therefore, $\refnu_t - \frac{2}{n_t}\sum_{i=1}^{n_t}\fV(P_{\lti}, \refV_{\lti}) = \tilO{\frac{B_t^2}{n_t} + \frac{\B}{n_t}\sum_{i=1}^{n_t}P_{\lti}\refB_{\lti} }$, and
	\begin{align*}
		&\refnu_t - 2\fV(P_t, \optV) = \refnu_t - \frac{2}{n_t}\sum_{i=1}^{n_t}\fV(P_{\lti}, \refV_{\lti}) + \frac{2}{n_t}\sum_{i=1}^{n_t}(\fV(P_{\lti}, \refV_{\lti}) - \fV(P_{\lti}, \optV)) \tag{$P_t=P_{\lti}$}\\
		&\overset{\text{(i)}}{\leq} \tilO{ \frac{\B^2}{n_t} + \frac{\B}{n_t}\sum_{i=1}^{n_t}P_{\lti}\refB_{\lti} } + \frac{4\B}{n_t}\sum_{i=1}^{n_t} P_{\lti}\rbr{\optV - \refV_{\lti}}\\
		&= \tilO{ \frac{\B^2}{n_t} + \frac{\B}{n_t}\sum_{i=1}^{n_t}P_{\lti}\refB_{\lti} + \B\beta_{\optq} }, \tag{$\optV(s) - \refV_{\lti}(s) \leq \refB_{\lti}(s) + \beta_{\optq},\forall s $} 
	\end{align*}
	where in (i) we apply the bound for $\refnu_t - \frac{2}{n_t}\sum_{i=1}^{n_t}\fV(P_{\lti}, \refV_{\lti})$, $B_t\leq\B$ and 
	\begin{align*}
		\fV(P_{\lti}, \refV_{\lti}) - \fV(P_{\lti}, \optV) \leq (P_{\lti}\optV)^2 - (P_{\lti}\refV_{\lti})^2 \leq 2\B P_{\lti}(\optV-\refV_{\lti}).
	\end{align*}
	Plugging the inequality above back, we have with probability at least $1-11\delta$,
	\begin{align*}
		&\sumt \sqrt{\frac{\refnu_t}{n_t}} = \tilO{ \sumt\sqrt{\frac{\fV(P_t, \optV)}{n_t}} + \frac{\B}{n_t} + \frac{1}{n_t}\sqrt{\B\sum_{i=1}^{n_t}P_{\lti}\refB_{\lti}} + \sqrt{\frac{\B\beta_{\optq}}{n_t}}}\\
		&= \tilO{ \sqrt{SA\sumt\fV(P_t, \optV)} + \B SA + \sqrt{\sumt\frac{\B}{n_t}}\sqrt{\sumt\frac{1}{n_t}\sum_{i=1}^{n_t}P_{\lti}\refB_{\lti}} + \sqrt{\B\beta_{\optq} SAT} } \tag{\pref{lem:sum of count} and Cauchy-Schwarz inequality}\\
		&= \tilO{\sqrt{\B SAC_K} + \B SA + \sqrt{\B SA\RefC} + \sqrt{\B\beta_{\optq} SAT} }. \tag{\pref{lem:var optV}, \pref{lem:sum of count}, \pref{lem:stage sum} and \pref{lem:refV diff}}
	\end{align*}
	Moreover,
	\begin{align*}
		\sumt\sqrt{\frac{\nu_t}{m_t}} &\leq \sumt\frac{\sqrt{\sum_{i=1}^{m_t}(V_{\clti}(s'_{\clti}) - \refV_{\clti}(s'_{\clti}))^2}}{m_t} \leq \sumt\frac{\sqrt{\sum_{i=1}^{m_t}(\optV(s'_{\clti}) - \refV_{\clti}(s'_{\clti}))^2 }}{m_t}\\
		&= \tilO{ \sumt \frac{\sqrt{\sum_{i=1}^{m_t}\refB_{\clti}(s'_{\clti})^2}}{m_t} + \frac{\sqrt{\sum_{i=1}^{m_t}\beta_{\optq}^2}}{m_t}} \tag{$\optV(s'_{\clti}) - \refV_{\clti}(s'_{\clti}) \leq \refB_{\clti}(s'_{\clti}) + \beta_{\optq}$, $(a+b)^2\leq 2a^2+2b^2$, and $\sqrt{x+y}\leq\sqrt{x}+\sqrt{y}$}\\
		&= \tilO{ \sqrt{\sumt\frac{1}{m_t}}\sqrt{\sumt\frac{1}{m_t}\sum_{i=1}^{m_t}\refB_{\clti}(s'_{\clti})^2 } + \sumt\sqrt{\frac{\beta_{\optq}^2}{m_t}} } \tag{Cauchy-Schwarz inequality}.
	\end{align*}
	Note that by \pref{lem:stage sum}, \pref{lem:t diff}, $\refB_{t+1}(s'_t) \leq \refB_{t+1}(s_{t+1})$ and \pref{lem:refV}:
	\begin{align*}
		&\sumt\frac{1}{m_t}\sum_{i=1}^{m_t}\refB_{\clti}(s'_{\clti})^2 \leq \rbr{1 + \frac{1}{H}}\sumt\refB_t(s'_t)^2\\
		&= \tilO{ \sumt \refB_{t+1}(s'_t) + S\B^2 } = \tilO{ \sumt\refB_t(s_t) + S\B^2 } = \tilO{ \RefCs }.
	\end{align*}
	Plugging this back to the last inequality, and by \pref{lem:sum of count}, we have:
	\begin{align*}
		\sumt\sqrt{\frac{\nu_t}{m_t}} &= \tilO{ \sqrt{SAH\RefCs} + \sqrt{SAH\beta_{\optq}^2T} }.
	\end{align*}
	Finally, by Cauchy-Schwarz inequality, \pref{eq:sum w_t=1, 1}, \pref{eq:c-hatc} and \pref{lem:e2r}:
	\begin{align*}
		\sumt \sqrt{\frac{\hatc_t\eps_t}{n_t}} &= \tilO{\sqrt{SA\sumt\hatc_t\eps_t}} = \tilO{ \sqrt{SA\rbr{ \sumt c(s_t, a_t) + \sumt (\hatc_t - c(s_t, a_t))\eps_t }} }\\
		&= \tilO{ \sqrt{SAC_K} + \sqrt{SA\sumt\sqrt{\frac{\hatc_t\eps_t}{n_t}}  } + SA }.
	\end{align*}
	Solving a quadratic equation gives $\sumt \sqrt{\frac{\hatc_t\eps_t}{n_t}} = \tilO{ \sqrt{SAC_K} + SA }$.
	Putting everything together, and by $\beta_{\optq}=\bigO{\cmin}, \beta_{\optq}T = \bigO{\cmin T} = \bigO{C_K}$:
	\begin{align*}
		\sumt b_t &= \tilO{\sqrt{\B SAC_K} + \sqrt{\B SA\RefC} + \sqrt{SAH\RefCs} + \sqrt{SAH\cmin C_K} + \B HSA }\\
		&= \tilO{ \sqrt{\B SAC_K} + \B H^2S^{\frac{3}{2}}A + \sqrt{SAH\cmin C_K} }. \tag{$H=\lowO{\frac{\B}{\cmin}}$ and definition of $\RefC, \RefCs$ (\pref{lem:refV})}
	\end{align*}
\end{proof}


\begin{lemma}[bias of the update scheme]
	\label{lem:stage sum}
	Assuming $X_t\geq 0$, we have:
	\begin{align*}
		\sumt\frac{1}{m_t}\sum_{i=1}^{m_t}X_{\clti}\leq \rbr{1+\frac{1}{H}}\sumt X_t,\qquad \sumt\frac{1}{n_t}\sum_{i=1}^{n_t} X_{\lti} = \bigO{\ln(T)\sumt X_t}.
	\end{align*}
\end{lemma}
\begin{proof}
	For the first inequality, denote by $j_t$ the stage to which time step $t$ belongs.
	When $t'=\clti$, we have $m_t=e_{j_{t'}}$.
	Therefore, 
	$\sumt\sum_{i=1}^{m_t}\frac{1}{m_t}\Ind\{t'=\clti\} \leq \frac{e_{j_{t'}+1}}{e_{j_{t'}}} \leq 1 + \frac{1}{H}$,
	and
	\begin{align*}
		\sumt\frac{1}{m_t}\sum_{i=1}^{m_t}X_{\clti} &= \sumt\frac{1}{m_t}\sum_{i=1}^{m_t}\sum_{t'=1}^TX_{t'}\Ind\{t'=\clti\} = \sum_{t'=1}^TX_{t'}\sumt\sum_{i=1}^{m_t}\frac{\Ind\{t'=\clti\}}{m_t} \leq \rbr{1+\frac{1}{H}}\sum_{t'=1}^T X_{t'}.
	\end{align*}
	For the second inequality:
	\begin{align*}
		\sumt\frac{1}{n_t}\sum_{i=1}^{n_t}X_{\lti} &= \sumt\frac{1}{n_t}\sum_{i=1}^{n_t}\sum_{t'=1}^TX_{t'}\Ind\{t'=\lti\} = \sum_{t'=1}^TX_{t'}\sumt\sum_{i=1}^{n_t}\frac{\Ind\{t'=\lti\}}{n_t}\\
		&\leq \sum_{t'=1}^TX_{t'}\sum_{z: t'\leq E_{z-1}\leq T}\frac{e_z}{E_{z-1}} =\bigO{\ln(T)\sum_{t'=1}^T X_{t'}}.
	\end{align*}
\end{proof}

\begin{lemma}
	\label{lem:p2a}
	Assuming $X_t: \calS^+\rightarrow [0, B]$ is monotonic in $t$ (i.e., $X_t(s)$ is non-increasing or non-decreasing in $t$ for any $s\in\calS^+$) and $X_t(g)=0$, with probability at least $1-\delta$,
	$$\sumt \frac{1}{m_t}\sum_{i=1}^{m_t}P_{\clti}X_{\clti} \leq \rbr{1 + \frac{1}{H}}^2\sumt X_t(s_t) + \tilO{B(H+S)}.$$
\end{lemma}
\begin{proof}
	By \pref{lem:stage sum}, \pref{lem:e2r} and \pref{lem:t diff}, $X_{t+1}(s'_t)\leq X_{t+1}(s_{t+1})$ in each step,
	\begin{align*}
		\sumt \frac{1}{m_t}\sum_{i=1}^{m_t}P_{\clti}X_{\clti} &\leq \rbr{ 1+\frac{1}{H} }\sumt P_tX_t \leq \rbr{1 + \frac{1}{H}}^2\sumt X_t(s'_t) + \tilO{BH}\\
		&\leq \rbr{1 + \frac{1}{H}}^2\sumt X_t(s_t) + \tilO{B(H+S)}.
	\end{align*}
\end{proof}

\begin{lemma}
	\label{lem:sum of count}
	For any non-negative weights $\{w_t\}_t$, and $\alpha\in (0, 1)$, we have:
	\begin{align*}
		\sumt \frac{w_t\eps_t}{n_t^{\alpha}} = \bigO{ (\norm{w}_{\infty}SA)^{\alpha}\norm{w}_1^{1-\alpha} }, \quad \sumt \frac{w_t\eps_t}{m_t^{\alpha}} = \bigO{ (\norm{w}_{\infty}HSA)^{\alpha}\norm{w}_1^{1-\alpha}\ln\frac{\norm{w}_{\infty}}{\norm{w}_1} }.
	\end{align*}
	Moreover, when $w_t=v(s_t, a_t)$ for some $v$,
	\begin{align*}
		\sumt \frac{w_t\eps_t}{n_t^{\alpha}} = \tilO{ \sumsa v(s, a)N_{T+1}(s, a)^{1-\alpha} }, \quad \sumt \frac{w_t\eps_t}{m_t^{\alpha}} = \tilO{ H^{\alpha}\sumsa v(s, a)N_{T+1}(s, a)^{1-\alpha} }.
	\end{align*}
	In case $w_t=1$ for all $t$, it holds that:
	\begin{align}
		\label{eq:sum w_t=1}
		\sumt\frac{\eps_t}{n_t^{\alpha}} = \tilO{(SA)^{\alpha}T^{1-\alpha}}, \quad \sumt \frac{\eps_t}{m_t^{\alpha}} = \tilO{(SAH)^{\alpha}T^{1-\alpha}},
	\end{align}
	when $0< \alpha < 1$, and
	\begin{align}
		\label{eq:sum w_t=1, 1}
		\sumt\frac{\eps_t}{n_t} = \bigO{ SA\ln T }, \quad \sumt \frac{\eps_t}{m_t} = \bigO{SAH\ln T},
	\end{align}
	when $\alpha=1$.
\end{lemma}
\begin{proof}
	Define $\frakn(s, a, j)=\sum_{t:(s_t, a_t)=(s, a), n_t=E_j}w_t$, $\frakn(s, a)=\sum_{j\geq 0}\frakn(s, a, j)$.
	Then, $\sumsa\frakn(s, a)=\norm{w}_1$, $\frakn(s, a, j)\leq \norm{w}_{\infty}e_{j+1} \leq \rbr{1+\frac{1}{H}}\norm{w}_{\infty}e_j$.
	Moreover, by definitions of $e_j$ and $E_j$,
	\begin{align}
		\sum_{j\geq 1}\Ind\cbr{ \rbr{1+\frac{1}{H}}\norm{w}_{\infty}E_{j-1} \leq \frakn(s, a) } = \bigO{H\ln\frac{\norm{w}_1}{\norm{w}_{\infty}}}. \label{eq:weight of n(s, a)}\\
		\sum_{j\geq 1}e_j \Ind\cbr{ \rbr{1+\frac{1}{H}}\norm{w}_{\infty}E_{j-1} \leq \frakn(s, a) } = \bigo{\frakn(s, a) / \norm{w}_{\infty}}. \label{eq:weight of e_j n(s, a)}
	\end{align}
	Since $\frac{1}{E_j^\alpha}$ and $\frac{1}{e_j^{\alpha}}$ is decreasing, by ``moving weights to earlier terms'' (from $\frakn(s, a, j)$ to $\frakn(s, a, i)$ for $i < j$),
	\begin{align*}
		\sumt \frac{w_t\eps_t}{n_t^{\alpha}} &= \sumsa\sum_{j\geq 1}\frac{\frakn(s, a, j)}{E_j^{\alpha}} \leq \sumsa\sum_{j\geq 1}\rbr{1 + \frac{1}{H}}\norm{w}_{\infty}\frac{e_j\Ind\cbr{ \rbr{1+\frac{1}{H}}\norm{w}_{\infty}E_{j-1} \leq \frakn(s, a) } }{E_j^{\alpha}}\\
		&=\bigO{\sumsa\norm{w}_{\infty}\rbr{\frac{\frakn(s, a)}{\norm{w}_{\infty}}}^{1-\alpha}} \tag{$\sum_{j=1}^J\frac{e_j}{E_j^{\alpha}}=\bigO{E_J^{1-\alpha}}$ and \pref{eq:weight of e_j n(s, a)}} \\
		&= \bigO{ (\norm{w}_{\infty}SA)^{\alpha}\norm{w}_1^{1-\alpha} }, \tag{H\"older's inequality}\\
		\sumt\frac{w_t\eps_t}{m_t^{\alpha}} &= \sumsa\sum_{j\geq 1}\frac{\frakn(s, a, j)}{e_j^\alpha} \leq \sumsa\sum_{j\geq 1}\rbr{1+\frac{1}{H}}\norm{w}_{\infty}e_j^{1-\alpha}\Ind\cbr{ \rbr{1+\frac{1}{H}}\norm{w}_{\infty}E_{j-1} \leq \frakn(s, a) }\\
		&\leq \rbr{1+\frac{1}{H}}\norm{w}_{\infty} \rbr{\sumsa\sum_{j\geq 1}\Ind\cbr{ \norm{w}_{\infty}E_{j-1} \leq \frakn(s, a) } }^{\alpha}\rbr{ \sumsa\frac{\frakn(s, a)}{\norm{w}_{\infty}} }^{1-\alpha} \tag{H\"older's inequality and \pref{eq:weight of e_j n(s, a)} }\\
		&= \bigO{ (\norm{w}_{\infty}HSA)^{\alpha}\norm{w}_1^{1-\alpha}\ln\frac{\norm{w}_1}{\norm{w}_{\infty}} }. \tag{\pref{eq:weight of n(s, a)}}
	\end{align*}
	In case $w_t=1$ and $\alpha\in(0, 1)$, we have $\norm{w}_{\infty}=1, \norm{w}_1=T$, and \pref{eq:sum w_t=1} is proved.
	When $w_t=v(s_t, a_t)$ for some $v$, $\frakn(s, a, j)\leq v(s, a)e_{j+1}\Ind\{j\leq J_{s, a}\}$, where $J_{s, a}$ is such that $E_{J_{s, a}}=n_T(s, a)$.
	Thus,
	\begin{align*}
		\sumt \frac{w_t\eps_t}{n_t^{\alpha}} &\leq \sumsa v(s, a)\sum_{j=1}^{J_{s, a}}\frac{e_{j+1} }{E_j^{\alpha}} = \bigO{\sumsa v(s, a)\sum_{j=1}^{J_{s, a}}\frac{e_j}{E_j^{\alpha}} } =  \bigO{ \sumsa v(s, a)N_{T+1}(s, a)^{1-\alpha} }.\\
		\sumt\frac{w_t\eps_t}{m_t^{\alpha}} &\leq \sumsa v(s, a)\sum_{j=1}^{J_{s, a}}\frac{e_{j+1}}{e_j^\alpha} =\bigO{\sumsa v(s, a)\sum_{j=1}^{J_{s, a}}e_j^{1-\alpha}}\\
		&= \tilO{ \sumsa v(s, a)J_{s, a}^{\alpha}\rbr{\sum_{j=1}^{J_{s, a}}e_j}^{1-\alpha} } = \tilO{ H^{\alpha}\sumsa v(s, a)N_{T+1}(s, a)^{1-\alpha} }. \tag{H\"older's inequality and $J_{s, a}=\tilO{ H }$ by how $e_j$ grows}
	\end{align*}
	In case $\alpha=1$, we have:
	\begin{align*}
		\sumt \frac{\eps_t}{n_t} &\leq \sumsa \sum_{j: 0 < E_{j-1}\leq T}\frac{e_j}{E_{j-1}} = \bigO{SA\ln T}.\\
		\sumt \frac{\eps_t}{m_t} &\leq \sumsa \sum_{j:0<E_{j-1}\leq T}\rbr{1+\frac{1}{H}} = \bigO{ SAH\ln T }.
	\end{align*}
\end{proof}

\subsection{Parameter free algorithm}
\label{sec:parameter free}

\setcounter{AlgoLine}{0}
\begin{algorithm}[t]
	\caption{\mf with an upper bound on $\B$}
	\label{alg:Q-B}
	
	\textbf{Parameter:} initial value function upper bound $\tilB\geq\B$, failure probability $\delta \in (0,1)$.
	
	\textbf{Define:} $\calL_p=\{E_{p,j}\}_{j\in\fN^+}$ where $E_{p,j}=\sum_{i=1}^je_{p,i}$, 
	$e_{p,1}=H_p$ and $e_{p,j+1}=\floor{(1+1/H_p)e_{p,j}}$.
	
	\textbf{Initialize:} $t \leftarrow 0$, $s_1\leftarrow \sinit$, $B\leftarrow1$, for all $(s, a), p\in\calP, N(s, a)\leftarrow 0, M_p(s, a)\leftarrow 0$.
	
	\textbf{Initialize:} for all $(s, a), r\in\calR, Q(s, a)\leftarrow 0, V(s)\leftarrow 0, \refV_r(s)\leftarrow V(s)$, $\hatC(s, a)\leftarrow 0$.
	
	\textbf{Initialize:} for all $(s, a), p\in\calP, r\in\calR$, $\refmu_{p,r}(s, a)\leftarrow 0$, $\refsigma_{p,r}(s, a)\leftarrow 0$, $\mu_{p,r}(s, a)\leftarrow 0$, $\sigma_{p,r}(s, a)\leftarrow 0$, $v_p(s, a)\leftarrow 0$.
	
	\For{$k=1,\ldots,K$}{
	
	     \MyRepeat{	
              Increment time step $t\overset{+}{\leftarrow}1$.
              
			Take action $a_t= \argmin_aQ(s_t, a)$, suffer cost $c_t$, transit to and observe $s'_t$. 
			
			Update global accumulators: $n=N(s_t, a_t)\overset{+}{\leftarrow} 1$, $\hatC(s_t, a_t)\overset{+}{\leftarrow} c_t$. 
			
			\nl Compute $\iota\leftarrow 256\ln^6(4SA\tilB^8n^5\cdot 8N_{\beta}^2/\delta)$, $\hatc \leftarrow \frac{\hatC(s_t, a_t)}{n}$. \label{line:compute pf B}
			
			\For{$p\in\calP$}{
				\For{$r\in\calR$}{
					\nl Update reference value accumulators: $\refmu_{p,r}(s_t, a_t) \overset{+}{\leftarrow} \refV_r(s'_t), \; \refsigma_{p,r}(s_t, a_t)\overset{+}{\leftarrow} \refV_r(s'_t)^2, \; \mu_{p,r}(s_t, a_t) \overset{+}{\leftarrow} V(s'_t) - \refV_r(s'_t), \; \sigma_{p,r}(s_t, a_t) \overset{+}{\leftarrow} (V(s'_t) - \refV_r(s'_t))^2$. \label{line:accum1}
				}
				
				\nl Update accumulators: $v_p(s_t, a_t)\overset{+}{\leftarrow} V(s'_t),\; m_p = M_p(s_t, a_t) \overset{+}{\leftarrow} 1$. \label{line:accum2}
				
				\nl \If{$n \in \calL_p$}{\label{line:update call pf}
					
					\For{$r\in\calR$}{
						$b_{p,r} \leftarrow \sqrt{\frac{\nicefrac{\refsigma_{p,r}(s_t,a_t)}{n}-(\nicefrac{\refmu_{p,r}(s_t,a_t)}{n})^2}{n}\iota} + \sqrt{\frac{\nicefrac{\sigma_{p,r}(s_t,a_t)}{m_p}-(\nicefrac{\mu_{p,r}(s_t,a_t)}{m_p})^2}{m_p}\iota} + \rbr{\frac{4B}{n} + \frac{3B}{m_p}}\iota + \sqrt{\frac{\hatc\iota}{n}}$.
						
						\nl $Q(s_t, a_t)\leftarrow \max\left\{\hatc + \frac{\refmu_{p,r}(s_t, a_t)}{n} + \frac{\mu_{p,r}(s_t, a_t)}{m_p} - b_{p,r}, Q(s_t, a_t)\right\}$. \label{line:update mf2 pf}
						
						Reset local accumulators: $\mu_{p,r}(s_t, a_t)\leftarrow 0,\; \sigma_{p,r}(s_t, a_t)\leftarrow 0$.
					}
				
				     Compute bonus $b'_p \leftarrow 2\sqrt{\frac{B^2\iota}{m_p}} + \sqrt{\frac{\hatc\iota}{n}} + \frac{\iota}{n}$.
				
					\nl $Q(s_t, a_t)\leftarrow \max\left\{\hatc + \frac{v_p(s_t, a_t)}{m_p} -b'_p, Q(s_t, a_t)\right\}$. \label{line:update mf1 pf}
					
					Reset local accumulators: $v_p(s_t, a_t)\leftarrow 0,\; M_p(s_t, a_t)\leftarrow 0$.
				}
			}
			
			$V(s_t)\leftarrow \min_a Q(s_t, a)$.
					
			\lIf{$V(s_t)>B$}{
				$B \leftarrow 2V(s_t)$.
			}
			
			\nl \lIf{$\sum_aN(s_t, a)=2^r$ for some $r\in\calR$}{\label{line:update refV pf}
					$\refV_{r'}(s_t)\leftarrow V(s_t),\;\forall r'\geq r$.
			}
			
			\lIf{$s_t' \neq g$}{$s_{t+1}\leftarrow s'_t$; \textbf{else} $s_{t+1}\leftarrow \sinit$, \textbf{break}.} 
		}		
	}
\end{algorithm}

In this section, we present a parameter-free model-free algorithm (\pref{alg:Q pf}) that achieves the same regret guarantee as \pref{alg:Q} (up to log factors).
The high level idea is to first apply the doubling trick from \cite{tarbouriech2021stochastic} to determine an upper bound on $\B$, then try logarithmically many different values of $H$ and $\thetastar$ simultaneously, each leading to a different update rule for $Q$ and $\refV$.

\subsubsection{An upper bound on $\B$ is available} 
We first introduce \pref{alg:Q-B}, which is a sub-algorithm that achieves the desired regret bound when we have an upper bound $\tilB\geq\B$.
In this case, we only need to determine the appropriate value of $H$ and $\thetastar$.
Define $N_{\beta}=\ceil{\log_2(1/\beta)}$ with $\beta=\frac{\cmin}{2\tilB^2SAK}$, 
$H_p=2^p$ for $p\in\calP$ with $\calP=[N_{\beta}]$, and $\calH=\{H_p\}_{p\in\calP}$.
Define $\calR=[8N_\beta]$.
Here, $\calH$ and $\{2^r\}_{r\in\calR}$ constitute the search range of $H$ and $\thetastar$.

For each $p, r$, we maintain accumulators $\refmu_{p,r}, \refsigma_{p,r}, \mu_{p, r}, \sigma_{p, r}, v_p, m_p$ similar to $\refmu, \refsigma, \mu, \sigma, v, m$ in \pref{alg:Q} (\pref{line:accum1} and \pref{line:accum2}).
For each $(s, a)\in\SA$ and $p\in\calP$, we divide the samples received into consecutive stages, where the length of the $j$-th stage is $e_{p,j}$ with $e_{p,1}=H_p, e_{p,j+1}=\floor{(1+\frac{1}{H_p})e_{p,j}}$.
Also define the indices indicating the end of a stage for a given $p$ as $\calL_p=\{E_{p,j}\}_{j\in\fN^+}$ with $E_{p,j}=\sum_{i=1}^je_{p,i}$.
We update $Q(s, a)$ only when the number of visits to $(s, a)$ falls into $\calL_p$ for some $p\in\calP$ (\pref{line:update call pf}), and there are two types of update rules similar to \pref{alg:Q} (\pref{line:update mf2 pf} and \pref{line:update mf1 pf}).
We also maintain $|\calR|$ reference value functions, each with different final precision (\pref{line:update refV pf}).
We show that the way we combine different update rules enable us to apply analysis of \pref{alg:Q} w.r.t any choice of $(p, r)\in\calP\times\calR$.
Notably, we can proceed with $(\optp, \optr)$ with $H_{\optp}=H, 2^{\optr}=\thetastar$, which gives us the same regret bound as \pref{alg:Q} without knowing $\B$.

Now we introduce some notations only used in this section.
When it is clear from the context, we ignore dependency on $p$, and define $n_t(s, a), m_t(s, a), \lti(s, a), \clti(s, a), \hatc_t(s, a)$ similarly as before for a given $p$.
Denote by $\refV_{r,t}(s)$ the value of $\refV_r(s)$ at the beginning of time step $t$, and by $b_{p,r,t}(s, a), b'_{p,t}(s, a)$ the value of $b_{p,r}(s, a), b'_p(s, a)$ in $Q_t(s, a)$.
Also define:
\begin{align*}
	\overline{Q}_{p,r,t}(s, a) &= \hatc_t(s, a) + \frac{1}{n_t}\sum_{i=1}^{n_t}\refV_{r,\lti}(s'_{\lti}) + \frac{1}{m_t}\sum_{i=1}^{m_t}\rbr{V_{\clti}(s'_{\clti}) - \refV_{r,\clti}(s'_{\clti})} - b_{p,r,t}.\\
	\overline{Q}'_{p,t}(s, a) &= \hatc_t(s, a) + \frac{1}{m_t}\sum_{i=1}^{m_t}V_{\clti}(s'_{\clti}) - b'_{p,t}.
\end{align*}
Note that for any $(s, a)\in\SA, t>1$,
\begin{equation}
	\label{eq:decompose Q}
	Q_t(s, a)=\max\cbr{\max_{p,r}\overline{Q}_{p,r,t}(s, a), \max_p\overline{Q}'_{p,t}(s, a), Q_{t-1}(s, a)}.
\end{equation}
Next, we prove the key lemma of \pref{alg:Q-B}, which shows that $Q_t$ is an optimistic estimator of $\optQ$.

\begin{lemma}
	\label{lem:mf optimistic B}
	With probability at least $1-7\delta$, \pref{alg:Q-B} with input $\tilB\geq\B$ ensures $Q_t(s, a) \leq Q_{t+1}(s, a)\leq \optQ(s, a)$ for any $(s, a)\in\SA$.
\end{lemma}
\begin{proof}
	The first inequality is by the update rule of $Q_t$.
	Next, we prove $Q_t(s, a)\leq \optQ(s, a)$ by induction on $t$.
	It is clearly true when when $t=1$.
	For the induction step, note that for any $p, r$, the proof of \pref{lem:mf optimistic} still proceeds to conclude that $\overline{Q}_{p,r,t}(s, a)\leq\optQ(s, a)$ and $\overline{Q}'_{p,t}(s, a)\leq\optQ(s, a)$, where we substitute $b_t$ with $b_{p,r,t}$, $b'_t$ with $b'_{p,t}$, and $\refV_t$ with $\refV_{r,t}$ (also note \pref{rem:tilB}).
	Thus, by a union bound over $8N_{\beta}^2$ update rules, the computation of $\iota$ (\pref{line:compute}), and \pref{eq:decompose Q}, the claim is proved.
\end{proof}

\begin{theorem}
	\label{thm:Q-B}
	With probability at least $1-60\delta$, \pref{alg:Q-B} with input $\tilB\geq\B$ ensures $R_K = \tilO{\B\sqrt{SAK} + \frac{\B^5S^2A}{\cmin^4} }$.
\end{theorem}
\begin{proof}
	Define $\refV = \refV_{\optr}, \RefV = \refV_{\optr, T+1}, b_t = b_{\optp, \optr, t}, b'_t=b'_{\optp, t}, H=H_{\optp}$, and $n_t$, $m_t$, $\lti$, $\clti$ are defined for $\optp$.
	We have \pref{lem:stage sum}, \pref{lem: loop-free approx}, \pref{cor:estimate refV}, \pref{lem:refV}, \pref{lem:refV diff}, \pref{lem:sum bt mf} and \pref{thm:mf_properties} holds for \pref{alg:Q-B}.
	Following the steps in the proof of \pref{thm:Q} gives the desired result.
\end{proof}

\subsubsection{Without knowledge of $\B$}

\setcounter{AlgoLine}{0}
\begin{algorithm}[t]
	\caption{\mf without knowledge of $\B$}
	\label{alg:Q pf}

	\textbf{Parameter:} failure probability $\delta \in (0,1)$.
	
	\textbf{Define:} $\calL_p=\{E_{p,j}\}_{j\in\fN^+}$ where $E_{p,j}=\sum_{i=1}^je_{p,i}$, 
	$e_{p,1}=H_p$ and $e_{p,j+1}=\floor{(1+1/H_p)e_{p,j}}$.
	
	\textbf{Initialize:} $\tilB\leftarrow K$, $C\leftarrow 0$.
	
	\textbf{Initialize:} $t \leftarrow 0$, $s_1\leftarrow \sinit$, $B\leftarrow1$, for all $(s, a), p\in\calP, N(s, a)\leftarrow 0, M_p(s, a)\leftarrow 0$.
	
	\textbf{Initialize:} for all $(s, a), r\in\calR, Q(s, a)\leftarrow 0, V(s)\leftarrow 0, \refV_r(s)\leftarrow V(s), \hatC(s, a)\leftarrow 0$.
	
	\textbf{Initialize:} for all $(s, a), p\in\calP, r\in\calR$, $\refmu_{p,r}(s, a)\leftarrow 0$, $\refsigma_{p,r}(s, a)\leftarrow 0$, $\mu_{p,r}(s, a)\leftarrow 0$, $\sigma_{p,r}(s, a)\leftarrow 0$, $v_p(s, a)\leftarrow 0$.
	
	\For{$k=1,\ldots,K$}{
	
	     \MyRepeat{	
              Increment time step $t\overset{+}{\leftarrow}1$.
              
			Take action $a_t= \argmin_aQ(s_t, a)$, suffer cost $c_t$, transit to and observe $s'_t$. 
			
			Update global accumulators: $n=N(s_t, a_t)\overset{+}{\leftarrow} 1$, $\hatC(s_t, a_t)\overset{+}{\leftarrow} c_t$, $C\overset{+}{\leftarrow} c_t$. 
			
			Compute $\iota\leftarrow 256\ln^6(4SA\tilB^8n^5\cdot 8N_{\beta}^2/\delta)$, $\hatc \leftarrow \frac{\hatC(s_t, a_t)}{n}$.
			
			\For{$p\in\calP$}{
				\For{$r\in\calR$}{
					Update reference value accumulators: $\refmu_{p,r}(s_t, a_t) \overset{+}{\leftarrow} \refV_r(s'_t), \; \refsigma_{p,r}(s_t, a_t)\overset{+}{\leftarrow} \refV_r(s'_t)^2, \; \mu_{p,r}(s_t, a_t) \overset{+}{\leftarrow} V(s'_t) - \refV_r(s'_t), \; \sigma_{p,r}(s_t, a_t) \overset{+}{\leftarrow} (V(s'_t) - \refV_r(s'_t))^2$.
				}
				
				Update accumulators: $v_p(s_t, a_t)\overset{+}{\leftarrow} V(s'_t),\; m_p = M_p(s_t, a_t) \overset{+}{\leftarrow} 1$.
				
				\If{$n \in \calL_p$}{
					
					\For{$r\in\calR$}{
						$b_{p,r} \leftarrow \sqrt{\frac{\nicefrac{\refsigma_{p,r}(s_t,a_t)}{n}-(\nicefrac{\refmu_{p,r}(s_t,a_t)}{n})^2}{n}\iota} + \sqrt{\frac{\nicefrac{\sigma_{p,r}(s_t,a_t)}{m_p}-(\nicefrac{\mu_{p,r}(s_t,a_t)}{m_p})^2}{m_p}\iota} + \rbr{\frac{4B}{n} + \frac{3B}{m_p}}\iota + \sqrt{\frac{\hatc\iota}{n}}$.
						
						$Q(s_t, a_t)\leftarrow \max\left\{\hatc + \frac{\refmu_{p,r}(s_t, a_t)}{n} + \frac{\mu_{p,r}(s_t, a_t)}{m_p} - b_{p,r}, Q(s_t, a_t)\right\}$.
						
						Reset local accumulators: $\mu_{p,r}(s_t, a_t)\leftarrow 0,\; \sigma_{p,r}(s_t, a_t)\leftarrow 0$.
					}
				
				     Compute bonus $b'_p \leftarrow 2\sqrt{\frac{B^2\iota}{m_p}} + \sqrt{\frac{\hatc\iota}{n}} + \frac{\iota}{n}$.
				
					$Q(s_t, a_t)\leftarrow \max\left\{\hatc + \frac{v_p(s_t, a_t)}{m_p} -b'_p, Q(s_t, a_t)\right\}$.
					
					Reset local accumulators: $v_p(s_t, a_t)\leftarrow 0,\; M_p(s_t, a_t)\leftarrow 0$.
				}
			}
			
			$V(s_t)\leftarrow \min_a Q(s_t, a)$.
					
			\lIf{$V(s_t)>B$}{
				$B \leftarrow 2V(s_t)$.
			}
			
			\lIf{$\sum_aN(s_t, a)=2^r$ for some $r\in\calR$}{
					$\refV_{r'}(s_t)\leftarrow V(s_t),\;\forall r'\geq r$.
			}
			
			\If{$B>\tilB$ or $C>\tilB K + x\rbr{\tilB\sqrt{SAK} + \frac{\tilB^5S^2A}{\cmin^4}}$}{
				$\tilB\leftarrow 2\tilB$, $C\leftarrow 0$.
				
				$B\leftarrow1$, for all $(s, a), p\in\calP, N(s, a)\leftarrow 0, M_p(s, a)\leftarrow 0$.
	
				for all $(s, a), r\in\calR, Q(s, a)\leftarrow 0, V(s)\leftarrow 0, \refV_r(s)\leftarrow V(s), \hatC(s, a)\leftarrow 0$.
	
				for all $(s, a), p\in\calP, r\in\calR$, $\refmu_{p,r}(s, a)\leftarrow 0$, $\refsigma_{p,r}(s, a)\leftarrow 0$, $\mu_{p,r}(s, a)\leftarrow 0$, $\sigma_{p,r}(s, a)\leftarrow 0$, $v_p(s, a)\leftarrow 0$.
			}
			
			\lIf{$s_t' \neq g$}{$s_{t+1}\leftarrow s'_t$; \textbf{else} $s_{t+1}\leftarrow \sinit$, \textbf{break}.} 
		}		
	}
	
\end{algorithm}

Now we introduce our parameter-free algorithm that achieves the desired regret bound without knowledge of $\B$.
The main idea is to determine an upper bound on $\B$ using a doubling trick from \citep{tarbouriech2021stochastic}, and then run \pref{alg:Q-B} as a sub-algorithm.
We divide the learning process into epochs indexed by $\phi$.
We maintain value function upper bound $\tilB$ and cost accumulator $C$ recording the total costs suffered in current epoch.
In epoch $\phi$, we execute \pref{alg:Q-B} with value function upper bound $\tilB$.
Moreover, we start a new epoch whenever:
\begin{enumerate}
	\item $B>\tilB$,
	\item or $C>\tilB K + x\rbr{\tilB\sqrt{SAK} + \frac{\tilB^5S^2A}{\cmin^4}}$.
\end{enumerate}
Here, $x$ is a large enough constant determined by \pref{thm:Q-B}, so that when $\tilB\geq\B$, we have with probability at least $1-60\delta$:
$$ C - \optV(\sinit^{\phi}) - (K-1)\optV(\sinit) \leq x\rbr{\tilB\sqrt{SAK} + \frac{\tilB^5S^2A}{\cmin^4}},$$
where $\sinit^{\phi}$ is the initial state of epoch $\phi$ (note that \pref{thm:Q} still holds when the initial state is changing over episodes).
Moreover, we double the value of $\tilB$ whenever a new epoch starts.
We summarize ideas above in \pref{alg:Q pf}.

\begin{theorem}
	With probability at least $1-60\delta$, \pref{alg:Q pf} ensures $R_K=\tilO{\B\sqrt{SAK} + \frac{\B^5S^2A}{\cmin^4}}$.
\end{theorem}
\begin{proof}
	Denote by $B_{\phi}$ the value of $B$ in epoch $\phi$, and by $C_{\phi}$ the value of $C$ at the end of epoch $\phi$.
	Define $\phistar=\inf_{\phi}\{B_{\phi}\geq\B\}$.
	Clearly $B_{\phi} \leq \max\{2\B, K\}$ for $\phi\leq\phistar$.
	By \pref{thm:Q-B}, with probability at least $1-60\delta$, there is at most $\phistar$ epochs since the condition of starting a new epoch will never be triggered in epoch $\phistar$, and the regret in epoch $\phistar$ is properly bounded:
	\begin{align*}
		C_{\phistar} - \optV(\sinit^{\phistar}) - (K-1)\optV(\sinit) = \tilO{\tilB_{\phistar}\sqrt{SAK} + \frac{\tilB_{\phistar}^5S^2A}{\cmin^4}} = \tilO{\B\sqrt{SAK} + \frac{\B^5S^2A}{\cmin^4}}.
	\end{align*}
	Conditioned on the event that there are at most $\phistar$ epochs, we partition the regret into two parts: the total costs suffered before epoch $\phistar$, and the regret starting from epoch $\phistar$.
	It suffices to bound the total costs before epoch $\phistar$ assuming $K\leq \B$ (otherwise $\phistar=1$).
	By the update scheme of $\tilB$, we have at most $\ceil{\log_2\B}+1$ epochs before epoch $\phistar$.
	Moreover, by the second condition of starting a new epoch, the accumulated cost in epoch $\phi<\phistar$ is bounded by:
	\begin{align*}
		C_{\phi} \leq K\tilB_{\phi} + \tilO{ \tilB_{\phi}\sqrt{SAK} + \tilB_{\phi}S^2A } = \tilO{ \frac{\B^5S^2A}{\cmin^4} }.
	\end{align*}
	Combining these two parts, we get:
	\begin{align*}
		R_K &= \sum_{\phi=1}^{\phistar-1}C_{\phi} + ( C_{\phistar} - \optV(\sinit^{\phistar}) - (K-1)\optV(\sinit) ) + (\optV(\sinit^{\phistar}) - \optV(\sinit))\\
		&= \tilO{ \B\sqrt{SAK} + \frac{\B^5S^2A}{\cmin^4} },
	\end{align*}
	where we assume $C_{\phistar}=0$ and $\sinit^{\phistar}=\sinit$ if there are less than $\phistar$ epochs.
\end{proof}

\section{Omitted Details for \pref{sec:mb}}
\label{app:mb}

\paragraph{Extra Notations} Denote by $Q_t(s, a), V_t(s)$ the value of $Q(s, a), V(s)$ at the beginning of time step $t$, $V_0(s)=0$, and $b_t(s, a), n_t(s, a), \P_{t, s, a}(s'), \iota_t(s, a)$, $\hatc_t(s, a)$ the value of $b, n, \P_{s, a}(s'), \iota, \hatc$ used in computing $Q_t(s, a)$ (note that $b_t(s, a)=0$ and $\hatc_t(s, a)=0$ if $n_t(s, a)=0$).
Denote by $\lt(s, a)$ the last time step the agent visits $(s, a)$ among those $n_t(s, a)$ steps before the current stage, and $\lt(s, a)=t$ if the first visit to $(s, a)$ is at time step $t$.
Also define $\bar{P}_t = \bar{P}_{t, s_t, a_t}$ and $\n_t(s, a)=\max\{1, n_t(s, a)\}$.
With these notations, we have by the update rule of the algorithm:
\begin{equation}
	\label{eq:mb_update_rule_alt}
	Q_t(s, a) = \max\{Q_{t-1}(s, a), \hatc_t(s, a) + \bar{P}_{t, s, a}V_{\lt} - b_t \},
\end{equation}
where $b_t$ represents $b_t(s, a)$, and $\lt$ represents $\lt(s, a)$ for notational convenience.

Before proving \pref{thm:mb} (\pref{sec:pf thm mb}), we first show some basic properties of our proposed update scheme (\pref{sec:properties update mb}), and proves the two required properties for \pref{alg:SVI} (\pref{sec:properties mb}).

\subsection{Properties of Proposed Update Scheme}
\label{sec:properties update mb}

In this section, we prove that our proposed update scheme has the desired properties, that is, it suffers constant cost independent of $H$, while maintaining sparse update in the long run similar to the update scheme of \pref{alg:Q} (\pref{lem:update scheme}).
We also quantify the bias induced by the sparse update compared to full-planning (that is, update every state-action pair at every time step) in \pref{lem:update bias}.

\begin{lemma}
	\label{lem:update scheme}
	The proposed update scheme satisfies the following:
	\begin{enumerate}
		\item For $\{X_t\}_{t\geq 0}$ such that $X_t\in[0, \Bref]$ and $t<t', (s_t, a_t)=(s_{t'}, a_{t'})$ implies $X_t \geq X_{t'}$, we have: $\sumt X_{\lt}\leq \Bref SA + (1+\frac{1}{H})\sumt X_t$.
		\item Denote $\istar_h=\inf\{i\geq\fN^+: e_i\geq h\}$ for $h\in\fN^+$.
		Then $\istar_h=\bigo{H\ln(h)}$.
	\end{enumerate}
\end{lemma}
\begin{proof}
	For any given $n\in\fN^+$, define $y_n$ as the index of the end of last stage, that is, the largest element in $\calL$ that is smaller than $n$ (also define $y_1=1$).
	For the first property, we first prove by induction that for any $j\in\fN^+$, there exist non-negative weights $\{w_{n,i}\}_{n,i}$ such that:
	\begin{enumerate}
		\item For all $n\leq E_j$, $\sum_{i=1}^{y_n}w_{n,i}=\Ind\{n>1\}$, and $w_{n,i}=0$ for $i>y_n$.
		\item $\sum_{n=1}^{E_j}w_{n, i}\leq 1 + \frac{1}{H}$ for any $i\leq E_j$.
		\item $\tile_{j+1} + \sum_{n=1}^{E_j}\sum_{n'=1}^{E_j} w_{n,n'} = (1+1/H)E_j$.
	\end{enumerate}
	To give some intuition, we can imagine a continuous process where we process index $n$ at time step $n$.
	Indices are divided into consecutive stages, and there are $e_j$ indices in the $j$-th stage.
	At index $n$ we need to consume $1$ unit of energy accumulated up to the last stage (that is, up to index $y_n$) and then contributes $(1+\frac{1}{H})$ energy to the future stages.
	We can think of $\tile_j$ as the available amount of energy at the beginning of stage $j$ (accumulated from indices up to $E_{j-1}$), and $e_j$ as the amount of energy consumed in stage $j$ (one unit by each index in stage $j$).
	The assignment of energy consumption is represented by $\{w_{n,i}\}$, where $w_{n,i}$ is the amount of energy consumed by index $n$ which is contributed by index $i$.
	The result we are going to prove by induction states that the process described above can proceed indefinitely.
	
	The base case of $j=1$ is clearly true by $w_{1,i}=0$ for any $i\in\fN^+$ and $\tile_2=1+\frac{1}{H}$.
	For the induction step, by the third property, there are in total $(1+\frac{1}{H})E_j$ energy contributed by indices up to $E_j$, where $\tile_{j+1}$ is the amount of energy available to use for stages starting from $j+1$, and $\sum_{n=1}^{E_j}\sum_{n'=1}^{E_j} w_{n,n'}$ is the amount of energy consumed by indices up to $E_j$ (we use one of the possible assignments of $\{w_{n,i}\}_{n,i}$ for $n\leq E_j$ from the previous induction step).
	We can easily distribute $e_{j+1}$ weights (from $\tile_{j+1}$) to indices in stage $j+1$ so that $\sum_{i=1}^{y_n}w_{n,i}=1$ and $w_{n,i}=0$ for $i>y_n$ for all $E_j<n\leq E_{j+1}$ (note that $y_n=E_j$ in this range),
	and $\sum_{n=1}^{E_{j+1}}w_{n,i}\leq 1+\frac{1}{H}$ for any $i\leq E_{j+1}$.
	Moreover,
	\begin{align*}
		\tile_{j+2} + \sum_{n=1}^{E_{j+1}}\sum_{n'=1}^{E_{j+1}}w_{n,n'} &= \tile_{j+1} + \frac{1}{H}e_{j+1} + \sum_{n=1}^{E_j}\sum_{n'=1}^{E_j}w_{n,n'} + e_{j+1}\\ 
		&= \rbr{1+\frac{1}{H}}E_j + \rbr{1+\frac{1}{H}}e_{j+1} = \rbr{1+\frac{1}{H}}E_{j+1}.\\
	\end{align*}
	Thus, the induction step also holds.
	We are now ready to prove the first property.
	Denote by $t_i(s, a)$ the time step of the $i$-th visit to $(s, a)$, and by $N(s, a)$ the total number of visits to $(s, a)$ in $K$ episodes.
	We have
	\begin{align*}
		\sumt X_{\lt} &= \sumsa\sum_{n=1}^{N(s, a)}X_{t_{y_n}(s, a)} \leq \sumsa X_{t_1(s, a)} + \sumsa \sum_{n=2}^{N(s, a)}\sum_{i=1}^{y_n}w_{n,i}X_{t_i(s, a)} \tag{$y_1=1$, $X_{t_i(s, a)}$ is non-increasing in $i$, and $\{w_{n, i}\}_{n, i}$ is from the induction result}\\
		&\leq \Bref SA + \sumsa\sum_{i=1}^{N(s, a)}X_{t_i(s, a)}\sum_{n=1}^{N(s, a)}w_{n,i} \leq \Bref SA + \rbr{1+\frac{1}{H}}\sumsa\sum_{i=1}^{N(s, a)}X_{t_i(s, a)} \tag{$X_{t_1(s, a)}\leq \Bref$ and $\sum_{n=1}^{N(s, a)}w_{n, i}\leq 1+\frac{1}{H}$}\\
		&= \Bref SA + \rbr{1 + \frac{1}{H}}\sumt X_t.
	\end{align*}
	For the second property, note that $\istar_h=\inf\{i\in\fN^+: \tile_i\geq h\}$ since $h$ is an interger.
	Moreover,
	\begin{align*}
		&\tile_{i+1} = \rbr{1+\frac{1}{H}}\tile_i + \frac{1}{H}(e_i-\tile_i) \geq \rbr{1+\frac{1}{H}}\tile_i - \frac{1}{H}
		\implies\tile_{i+1} - 1 \geq \rbr{1+\frac{1}{H}}(\tile_i-1)\\
		&\implies \tile_i \geq (\tile_{\istar_2}-1)\rbr{1+\frac{1}{H}}^{i-\istar_2} + 1 \geq \rbr{1+\frac{1}{H}}^{i-\istar_2} + 1,\quad\forall i\geq\istar_2.
	\end{align*}
	Therefore, $\istar_h \leq \inf_i\{ i\geq \istar_2: (1+1/H)^{i-\istar_2} + 1 \geq h \} = \istar_2 + \bigo{H\ln(h)}$.
	Also, by inspecting $e_i$ for small $i$ we observe that $\istar_2=\bigo{H}$, which implies that $\istar_h=\bigo{H\ln(h)}$.
\end{proof}

\begin{remark}
	\label{rem:update scheme}
	\pref{lem:update scheme} implies that there are at most $\bigo{\min\{SAH\ln T, ST\}}$ updates in $T$ steps.
 \end{remark}
 
 \begin{remark}
 	\label{rem:update constant}
 	Note that the update scheme in \citep{zhang2020almost} (also used in \pref{alg:Q}) induces a constant cost of order $\tilo{\B HSA}$, which ruins the horizon free regret.
	This is because their update scheme collects $H$ samples before the first update.
	On the contrary, our update scheme updates frequently at the beginning, but has the same update frequency as that of \citep{zhang2020almost} in the long run.
	This reduces the constant cost to $\tilo{\B SA}$ while maintaining the $\tilo{SAH}$ time complexity.
 \end{remark}
 
 The following lemma quantifies the dominating bias introduced by the sparse update.
 \begin{lemma}[bias of the update scheme]
 	\label{lem:update bias}
 	$\sumt P_t(V_t - V_{\lt}) \leq \B SA + \frac{1}{H}\sumt P_t(\optV - V_t)$ and $\sumt \fV(P_t, V_t - V_{\lt})\leq \tilO{\B^2SA} + \frac{\B}{H}\sumt P_t(\optV-V_t)$.
 \end{lemma}
 \begin{proof}
 	For the first statement, we apply \pref{lem:update scheme} and $P_t=P_{\lt}$ to obtain
	$$\sumt P_t(V_t - V_{\lt}) = \sumt P_{\lt}(\optV - V_{\lt}) - \sumt P_t(\optV - V_t) \leq \B SA + \frac{1}{H}\sumt P_t(\optV - V_t).$$
	Similarly, for the second statement
	\begin{align*}
		\sumt\fV(P_t, V_t-V_{\lt}) &\leq \sumt P_t(V_t-V_{\lt})^2 \leq \B\sumt P_t(V_t-V_{\lt})\\
		&\leq \B^2 SA + \frac{\B}{H}\sumt P_t(\optV-V_t). 
	\end{align*}
\end{proof}
 
\subsection{Proofs of Required Properties}
\label{sec:properties mb}
 
In this section, we prove \pref{prop:optimism} (\pref{lem:optimistic mb}) and \pref{prop:recursion} of \pref{alg:SVI},
where \pref{lem:recursion} proves a preliminary form of \pref{prop:recursion}.
 
\begin{lemma}
	\label{lem:optimistic mb}
	With probability at least $1-\delta$, $Q_t(s, a)\leq Q_{t+1}(s, a) \leq \optQ(s, a)$, for any $(s, a)\in\SA, t\geq 1$.
\end{lemma}
\begin{proof}
	The first inequality is clearly true by the update rule.
	Next, we prove $Q_t(s, a)\leq \optQ(s, a)$. 
	By \pref{eq:mb_update_rule_alt}, it is clearly true when $n_t(s, a)=0$.
	When $n_t(s, a)>0$, by \pref{lem:mvp}: (here, $\lt, \iota_t$ is a shorthand of $\lt(s, a), \iota_t(s, a)$):
	\begin{align*}
		&\hatc_t(s, a) + \P_{t, s, a}V_{\lt} - b_t(s, a) = \hatc_t(s, a) + f(\P_{t, s, a}, V_{\lt}, n_t(s, a), B, \iota_t) - \sqrt{\frac{\hatc_t(s, a)\iota_t}{n_t(s, a)}} \\
		&\leq c(s, a) + f(\P_{t, s, a}, \optV, n_t(s, a), B, \iota_t) + \frac{\iota_t}{n_t(s, a)} \tag{\pref{eq:c-hatc mb}}\\
		&= c(s, a) + \P_{t, s, a}\optV - \max\cbr{7\sqrt{\frac{\fV(\bar{P}_{t, s, a}, \optV)\iota_t}{n_t(s, a)}}, \frac{49B\iota_t}{n_t(s, a)}} + \frac{\iota_t}{n_t(s, a)}\\
		&\leq \optQ(s, a) + (\P_{t, s, a}-P_{s, a})\optV - 3\sqrt{\frac{\fV(\P_{t, s, a}, \optV)\iota_t}{n_t(s, a)}} - \frac{24B\iota_t}{n_t(s, a)} + \frac{B\iota_t}{n_t(s, a)} \tag{$B\geq \B\geq 1$, $\optQ(s, a)=c(s, a)+P_{s, a}\optV$ and $\max\{a, b\}\geq\frac{a+b}{2}$}\\
		&\leq \optQ(s, a) + (2\sqrt{2}-3)\sqrt{\frac{\fV(\P_{t, s, a}, \optV)\iota_t}{n_t(s, a)}} + (20-24)\frac{B\iota_t}{n_t(s, a)} \leq \optQ(s, a). \tag{\pref{lem:anytime bernstein}}
	\end{align*}
\end{proof}

\begin{lemma}
	\label{lem:recursion}
	With probability at least $1-9\delta$, for all $(\Qref, \Vref)\in\calV_H$
	\begin{align*}
		&\sumt (\Qref(s_t, a_t) - Q_t(s_t, a_t))_+ \leq \rbr{1+\frac{1}{H}}\sumt(\Vref(s_t)-V_t(s_t))_+ \\
		&\qquad + \tilO{ \sqrt{\B SAC_K} + BS^2A + \sqrt{\frac{\B S^2A}{H}\sumt \optV(s_t)-V_t(s_t)} }.
	\end{align*}
\end{lemma}
\begin{proof}
	We first prove useful properties related to the cost estimator.
	For a fixed $(s, a)$, by \pref{lem:anytime bernstein}, with probability at least $1-\frac{\delta}{SA}$, when $n_t(s, a)>0$:
	\begin{equation}
		\label{eq:c-hatc mb}
		\abr{c(s, a) - \hatc_t(s, a)} \leq 2\sqrt{\frac{2\hatc_t(s, a)}{n_t(s, a)}\ln\frac{2SA}{\delta}} + \frac{19\ln\frac{2SA}{\delta}}{n_t(s, a)} \leq \sqrt{\frac{\hatc_t(s, a)\iota_t}{n_t(s, a)}} + \frac{\iota_t}{n_t(s, a)}.
	\end{equation}
	Taking a union bound, we have \pref{eq:c-hatc mb} holds for all $(s, a)$ when $n_t(s, a)>0$ with probability at least $1-\delta$.
	Then by definition of $b_t$, we have
	\begin{equation}
		\label{eq:c-hatc b mb}
		c(s_t, a_t)-\hatc_t(s_t, a_t)\leq \Ind\{n_t=0\} + b_t.
	\end{equation}
	Note that with probability at least $1-2\delta$, for all $(\Qref, \Vref)\in\calV_H$,
	\begin{align*}
		&\sumt ( \Qref(s_t, a_t) - Q_t(s_t, a_t) )_+ \leq \sumt (c(s_t, a_t) - \hatc_t(s_t, a_t) + P_t\Vref - \P_tV_{\lt} )_+ + b_t \tag{$\Qref(s_t, a_t)=c(s_t, a_t) + P_t\Vref$ and \pref{eq:mb_update_rule_alt}}\\
		&\leq \sumt\Ind\{n_t=0\} + \sumt \sbr{(P_t(\Vref - V_{\lt}) + (P_t-\P_t)\optV + (P_t-\P_t)(V_{\lt}-\optV) )_+ + 2b_t}\\
		&\leq SA + \sumt \sbr{P_t(\Vref - V_{\lt})_+ + \tilO{\sqrt{\frac{\fV(P_t, \optV)}{\n_t}} + \sqrt{\frac{S\fV(P_t, \optV-V_{\lt})}{\n_t}} + \frac{S\B}{\n_t} } + 2b_t} \tag{$(x+y)_+ \leq (x)_+ + (y)_+$, \pref{lem:anytime bernstein}, and \pref{lem:P-eP}}.
	\end{align*}
	Note that:
	\begin{align*}
		&\sumt P_t(\Vref - V_{\lt})_+ \leq \rbr{1 + \frac{1}{H}}\sumt P_t(\Vref - V_t)_+ + \B SA \tag{$P_{\lt}=P_t$ and \pref{lem:update scheme}}\\
		&= \B SA + \rbr{1 + \frac{1}{H}} \sumt \rbr{(\Vref(s'_t) - V_t(s'_t) )_+ + (P_t - \Ind_{s'_t})(\Vref - V_t)_+}\\
		&\leq \bigO{\B SA} + \rbr{1 + \frac{1}{H}}\sumt \rbr{ (\Vref(s_t) - V_t(s_t) )_+ + (P_t - \Ind_{s'_t})(\Vref - V_t)_+ }. \tag{\pref{lem:t diff} and $(\Vref(s'_t)-V_{t+1}(s'_t))_+ \leq (\Vref(s_{t+1})-V_{t+1}(s_{t+1}))_+$}
	\end{align*}
	Plugging this back to the previous inequality, and by Cauchy-Schwarz inequality and \pref{lem:sum nt}:
	\begin{align*}
		&\sumt ( \Qref(s_t, a_t) - Q_t(s_t, a_t) )_+ \leq \rbr{1+\frac{1}{H}}\sumt \rbr{(\Vref(s_t) - V_t(s_t))_+ + (P_t - \Ind_{s'_t})(\Vref-V_t)_+ + b_t}\\
		&\qquad + \tilO{ \sqrt{SA\sumt\fV(P_t, \optV)} + \sqrt{S^2A\sumt\fV(P_t, \optV-V_{\lt})} + \B S^2A }.
	\end{align*}
	Next, we bound the term $\sumt (P_t - \Ind_{s'_t})(\Vref-V_t)_+$.
	We condition on \pref{lem:sum Vref-V_t}, which holds with probability at least $1-\delta$.
	Then, for a given $(\Qref, \Vref)\in\calV_H$, by \pref{lem:var recursion} with $X_t=(\Vref-V_t)_+/\B$, we have with probability $1-\frac{\delta}{H+1}$ ($F_T, Y_T$, and $\zeta_T$ are defined in \pref{lem:var recursion}):
	\begin{align*}
		\B F_T(0) &= \sumt (P_t - \Ind_{s'_t})(\Vref-V_t)_+ \leq \B( \sqrt{3Y_T\zeta_T} + 4\zeta_T) = \tilO{ \sqrt{\B^2Y_T} + \B }\\ 
		&= \tilO{ \sqrt{\B^2\rbr{S+1 + \sumt (X_t(s_t) - P_tX_t)_+}} + \B}\\
		&= \tilO{ \sqrt{\B^2 S + \B\sumt ( \Vref(s_t)-V_t(s_t) - P_t(\Vref-V_t) )_+} + \B }. \tag{$(x)_+ - (y)_+ \leq (x-y)_+$}\\
		&\overset{\text{(i)}}{=} \tilO{ \sumt b_t + \B S\sqrt{A} + \sqrt{\frac{\B}{H}\sumt P_t(\optV-V_t)} }\\
		&\qquad + \tilO{ \sqrt{SA\sumt\fV(P_t, \optV)} + \sqrt{S^2A\sumt\fV(P_t, \optV-V_{\lt})} },
	\end{align*}
	where in (i) we apply:
	\begin{align*}
		&\sqrt{\B\sumt ( \Vref(s_t)-V_t(s_t) - P_t(\Vref-V_t) )_+} \leq \sqrt{\B \rbr{ \sumt 2b_t + \frac{P_t(\optV-V_t)}{H} } }\\
		&\qquad + \tilO{ \sqrt{\B \rbr{\sqrt{SA\sumt\fV(P_t, \optV)} + \sqrt{S^2A\sumt\fV(P_t, \optV-V_{\lt})}} } + \B S\sqrt{A} } \tag{\pref{lem:sum Vref-V_t} and $\sqrt{x+y}\leq\sqrt{x}+\sqrt{y}$}\\
		&\leq 2\sumt b_t + \sqrt{\frac{\B}{H}\sumt P_t(\optV-V_t)} + \tilO{ \sqrt{SA\sumt\fV(P_t, \optV)} + \sqrt{S^2A\sumt\fV(P_t, \optV-V_{\lt})} + \B S\sqrt{A} }. \tag{AM-GM inequality and $\sqrt{x+y}\leq\sqrt{x}+\sqrt{y}$}
	\end{align*}
	Hence, by a union bound, the bound above for $\sumt (P_t - \Ind_{s'_t})(\Vref-V_t)_+$ holds for all $(\Qref, \Vref)\in\calV_H$ with probability at least $1-\delta$, and with probability at least $1-4\delta$, for all $(\Qref, \Vref)\in\calV_H$,
	\begin{align*}
		&\sumt ( \Qref(s_t, a_t) - Q_t(s_t, a_t) )_+ \leq \rbr{1+\frac{1}{H}}\sumt (\Vref(s_t)-V_t(s_t))_+ + \tilO{ \B S^2A + \sumt b_t }\\
		&\qquad +\tilO{ \sqrt{SA\sumt\fV(P_t, \optV)} + \sqrt{S^2A\sumt\fV(P_t, \optV-V_{\lt})} + \sqrt{\frac{\B}{H}\sumt P_t(\optV-V_t)} } \\
		&\leq \rbr{1+\frac{1}{H}}\sumt (\Vref(s_t)-V_t(s_t))_+ + \tilO{B S^2A + \sqrt{SA\sumt\fV(P_t, \optV)} }\\
		&\qquad + \tilO{ \sqrt{S^2A\sumt\fV(P_t, \optV-V_{\lt})} + \sqrt{\frac{\B SA}{H}\sumt P_t(\optV-V_t)} + \sqrt{SAC_K} }.\tag{\pref{lem:bound terms}}
	\end{align*}
	Note that:
	\begin{align*}
		&\sqrt{S^2A\sumt\fV(P_t, \optV-V_{\lt})}\\
		&= \tilO{ \sqrt{ \B S^2A\sqrt{SA\sumt\fV(P_t, \optV)} + B^2S^4A^2 + \frac{\B S^2A}{H}\sumt P_t(\optV-V_t) + \B S^2A\sqrt{SAC_K} } } \tag{\pref{lem:bound terms}}\\
		&= \tilO{ \sqrt{ \B S^2A\sqrt{SA\sumt\fV(P_t, \optV)}}  + BS^2A + \sqrt{ \frac{\B S^2A}{H}\sumt P_t(\optV-V_t) } + \sqrt{SAC_K} } \tag{$\sqrt{x+y}\leq\sqrt{x}+\sqrt{y}$ and AM-GM inequality}\\
		&= \tilO{ \sqrt{SA\sumt\fV(P_t, \optV)} + BS^2A + \sqrt{ \frac{\B S^2A}{H}\sumt P_t(\optV-V_t) } + \sqrt{SAC_K} }. \tag{AM-GM inequality}
	\end{align*}
	Plug this back to the previous inequality, and then by \pref{lem:var optV}
	\begin{align*}
		\sumt ( \Qref(s_t, a_t) - Q_t(s_t, a_t) )_+ &\leq \rbr{1+\frac{1}{H}}\sumt(\Vref(s_t)-V_t(s_t))_+\\ 
		&+ \tilO{ \sqrt{\B SAC_K} + BS^2A + \sqrt{\frac{\B S^2A}{H}\sumt P_t(\optV-V_t)} }.
	\end{align*}
	Finally, applying \pref{lem:e2r}, \pref{lem:t diff} and $(\optV-V_{t+1})(s'_t)\leq (\optV-V_{t+1})(s_{t+1})$, the claim is proved by
	\begin{align*}
		\sumt P_t(\optV-V_t) \leq \tilO{\B} + 2\sumt (\optV(s'_t) - V_t(s'_t)) \leq \tilO{S\B} + 2\sumt(\optV(s_t) - V_t(s_t)).
	\end{align*}
\end{proof}

\begin{proof}[\pfref{thm:mb_properties}]
	\pref{prop:optimism} is proved in \pref{lem:optimistic mb}.
	For \pref{prop:recursion}, by \pref{lem:recursion}, it suffices to bound $\sumt \optV(s_t)-V_t(s_t)$.
	By \pref{lem:recursion}, $\optV_{h-1}(s_t)\leq\optQ_h(s_t, a_t)$, and $V_t(s_t)=Q_t(s_t, a_t)$, we have with probability at least $1-9\delta$, for all $\Qref=\optQ_h, \Vref=\optV_{h-1}, h\in[H]$:
	\begin{align*}
		&\sumt ( \optQ_h(s_t, a_t) - Q_t(s_t, a_t) )_+ \leq \rbr{1+\frac{1}{H}}\sumt (\optQ_{h-1}(s_t, a_t)-Q_t(s_t, a_t))_+ \\
		&\qquad + \tilO{ \sqrt{\B SAC_K} + B S^2A + \sqrt{\frac{\B S^2A}{H}\sumt \optV(s_t)-V_t(s_t)}}, \quad \forall h\in[H].
	\end{align*}
	Applying the inequality above recursively starting from $h=H$ and by $\optQ_0(s, a)=0, (1+\frac{1}{H})^H\leq 3$ we have:
	\begin{align*}
		&\sumt ( \optQ_H(s_t, a_t) - Q_t(s_t, a_t) )_+ = \tilO{ H\rbr{\sqrt{\B SAC_K} + BS^2A} + \sqrt{\B HS^2A\sumt \optV(s_t)-V_t(s_t)} }.
	\end{align*}
	Then by \pref{lem:loop free approx} with $H=\ceil{\frac{4B}{\cmin}\ln(\frac{2}{\beta})+1}_2$:
	\begin{align*}
		\sumt \optV(s_t) - V_t(s_t) &\leq \sumt (\optQ(s_t, a_t) - \optQ_H(s_t, a_t)) + \sumt (\optQ_H(s_t, a_t) - Q_t(s_t, a_t))\\
		&\leq \B\beta T + \tilO{ H\rbr{\sqrt{\B SAC_K} + BS^2A} + \sqrt{BHS^2A\sumt \optV(s_t)-V_t(s_t)} }.
	\end{align*}
	Solving a quadratic equation w.r.t $\sumt \optV(s_t) - V_t(s_t)$ (\pref{lem:quad}), we have:
	\begin{align*}
		\sumt \optV(s_t) - V_t(s_t) \leq \B\beta T + \tilO{ H\rbr{ \sqrt{\B SAC_K} + BS^2A}}.
	\end{align*}
	Plug this back to the bound of \pref{lem:recursion} and by AM-GM inequality, we have for all $(\Qref, \Vref)\in\calV_H$:
	\begin{align*}
		&\sumt (\Qref(s_t, a_t) - Q_t(s_t, a_t))_+\\
		&\leq \rbr{1+\frac{1}{H}}\sumt(\Vref(s_t)-V_t(s_t))_+ + \frac{\B\beta T}{H} + \tilO{ \sqrt{\B SAC_K} + BS^2A }.
	\end{align*}
	Moreover, by $H\geq \frac{\B}{\cmin}$, we have $\frac{\B\beta T}{H} \leq \beta\cmin T \leq \beta C_K$.
	Hence, \pref{prop:recursion} is satisfied with $d=1, \xi_H= \beta C_K + \tilo{\sqrt{\B SAC_K} + BS^2A }$ with probability at least $1-9\delta$.
\end{proof}

\subsection{\pfref{thm:mb}}
\label{sec:pf thm mb}

\begin{proof}
	By \pref{thm:main_regret} and \pref{thm:mb_properties}, with probability at least $1-12\delta$:
	\begin{align*}
		C_K - K\optV(\sinit) =  R_K \leq \beta C_K + \tilO{\sqrt{\B SAC_K} + BS^2A }.
	\end{align*}
	Then by $\optV(\sinit)\leq\B, \beta\leq\frac{1}{2}$ and \pref{lem:quad}, we have $C_K=\tilO{\B K}$.
	Substituting this back and by $\beta\leq\frac{\cmin}{\B K}, H=\tilo{\B/\cmin}$, we get $R_K = \tilO{ \B\sqrt{SAK} + BS^2A }$.
\end{proof}

\subsection{Extra Lemmas for \pref{sec:mb}}

In this section, we give full proofs of auxiliary lemmas used in \pref{sec:mb}.
Notably, \pref{lem:sum Vref-V_t} and \pref{lem:bound terms} bound the additional terms appears in the recursion in \pref{lem:recursion}.
\pref{lem:var recursion} gives recursion-based analysis on bounding the sum of martingale difference sequence, which is the key in obtaining horizon-free regret.

\begin{lemma}
	\label{lem:sum Vref-V_t}
	With probability at least $1-\delta$, we have for all $(\Qref, \Vref)\in\calV_H$,
	\begin{align*}
		&\sumt ( (\Ind_{s_t} - P_t)(\Vref-V_t) )_+ \leq \sumt 2b_t + \frac{P_t(\optV-V_t)}{H} \\
		&\qquad + \tilO{ \sqrt{SA\sumt\fV(P_t, \optV)} + \sqrt{S^2A\sumt\fV(P_t, \optV-V_{\lt})} + \B S^2A }.
	\end{align*}
\end{lemma}
\begin{proof}
	With probability at least $1-\delta$, for all $(\Qref, \Vref)\in\calV_H$, 
	\begin{align*}
		&\sumt ( \Vref(s_t) - V_t(s_t) - P_t(\Vref-V_t) )_+ \leq \sumt ( \Qref(s_t, a_t) - P_t\Vref + P_tV_t - V_t(s_t) )_+\\
		&\leq \sumt ( c(s_t, a_t) + P_tV_{\lt} - V_t(s_t) )_+ + P_t(V_t-V_{\lt}) \tag{$\Qref(s_t, a_t)=c(s_t, a_t) + P_t\Vref$, $(x+y)_+ \leq (x)_+ + (y)_+$, and $V_t$ is increasing in $t$} \\
		&\leq \B SA + \sumt (c(s_t, a_t) - \hatc_t(s_t, a_t))_+ + ((P_t-\P_t)V_{\lt})_+ + b_t + \frac{1}{H} P_t(\optV-V_t) \tag{$V_t(s_t)=Q_t(s_t, a_t)$, \pref{eq:mb_update_rule_alt}, and \pref{lem:update bias}} \\
		&\leq 2\B SA  + \sumt ( (P_t-\P_t)\optV + (P_t-\P_t)(V_{\lt}-\optV) )_+ + 2b_t + \frac{1}{H} P_t(\optV-V_t). \tag{\pref{eq:c-hatc b mb}}
	\end{align*}
	Now by \pref{lem:anytime bernstein} and \pref{lem:P-eP}, we have with probability at least $1-\delta$: $(P_t-\P_t)\optV = \bigO{\sqrt{\frac{\fV(P_t, \optV)}{\n_t}} + \frac{\B}{\n_t} }$ and $(P_t-\P_t)(V_{\lt}-\optV) = \tilO{ \sqrt{\frac{S\fV(P_t, \optV-V_{\lt})}{\n_t}} + \frac{S\B}{\n_t} }$.
	Plugging these back to the previous inequality, we have for all $(\Qref, \Vref)\in\calV_H$:
	\begin{align*}
		&\sumt ( \Vref(s_t) - V_t(s_t) - P_t(\Vref-V_t) )_+\\
		&\leq 2\B SA + \sumt \tilO{\sqrt{\frac{\fV(P_t, \optV)}{\n_t}} + \sqrt{\frac{S\fV(P_t, \optV-V_{\lt})}{\n_t}} + \frac{S\B}{\n_t} } + 2b_t + \frac{1}{H} P_t(\optV-V_t)\\
		&\leq \tilO{ \sqrt{SA\sumt\fV(P_t, \optV)} + \sqrt{S^2A\sumt\fV(P_t, \optV-V_{\lt})} + \B S^2A } + \sumt 2b_t + \frac{P_t(\optV-V_t)}{H}. \tag{Cauchy-Schwarz inequality and \pref{lem:sum nt}}
	\end{align*}
	This completes the proof.
\end{proof}

\begin{lemma}
	\label{lem:bound terms}
	With probability at least $1-3\delta$,
	\begin{align*}
		&\sumt b_t = \tilO{B S^{3/2}A + \sqrt{SA\sumt\fV(P_t, \optV)} + \sqrt{\frac{\B SA}{H}\sumt P_t(\optV-V_t)} + \sqrt{SAC_K}},\\
		&\sumt\fV(P_t, \optV-V_{\lt}) = \tilO{ \B\sqrt{SA\sumt\fV(P_t, \optV)} + B^2S^2A + \frac{\B}{H}\sumt P_t(\optV-V_t) + \B\sqrt{SAC_K} }.
	\end{align*}
\end{lemma}
\begin{proof}
	First note that:
	\begin{align*}
		\sumt b_t &\overset{\text{(i)}}{=} \tilO{BSA + \sumt \sqrt{\frac{\fV(\P_t, V_{\lt})}{\n_t}} + \sqrt{\frac{\hatc_t}{\n_t}} } \overset{\text{(ii)}}{=} \tilO{BSA + \sumt \sqrt{\frac{\fV(P_t, V_{\lt})}{\n_t}} + \frac{\B\sqrt{S}}{\n_t} + \sqrt{\frac{\hatc_t}{\n_t}} }.
	\end{align*}
	where in (i) we apply $\max\{a, b\}\leq a+ b$ and \pref{lem:sum nt}, and in (ii) we have with probability at least $1-\delta$,
	\begin{align*}
		&\fV(\P_t, V_{\lt}) = \P_t(V_{\lt} - \P_t V_{\lt})^2 \leq \P_t(V_{\lt} - P_tV_{\lt})^2 \tag{$\frac{\sum_ip_ix_i}{\sum_ip_i}=\argmin_z\sum_ip_i(x_i-z)^2$}\\
		&= \fV(P_t, V_{\lt}) + (P_t-\P_t)(V_{\lt}-P_tV_{\lt})^2 \\
		&\leq \fV(P_t, V_{\lt}) + \tilO{\sum_{s'}\rbr{\sqrt{\frac{P_t(s')}{\n_t}} + \frac{1}{\n_t}}(V_{\lt}(s')-P_tV_{\lt})^2}  \tag{\pref{lem:anytime bernstein}}\\
		&\leq \fV(P_t, V_{\lt}) + \tilO{ \B\sqrt{\frac{S\fV(P_t, V_{\lt})}{\n_t}} + \frac{S\B^2}{\n_t} } = \tilO{ \fV(P_t, V_{\lt}) + \frac{S\B^2}{\n_t} }. \tag{Cauchy-Schwarz inequality and AM-GM inequality}
	\end{align*}
	Thus, by \pref{lem:var diff}, Cauchy-Schwarz inequality, and \pref{lem:sum nt}, we have:
	\begin{align}
		\sumt b_t &= \tilO{ BS^{3/2}A + \sumt \sqrt{\frac{\fV(P_t, \optV)}{\n_t}} + \sumt \sqrt{\frac{\fV(P_t, \optV-V_{\lt})}{\n_t}} + \sqrt{\frac{\hatc_t}{\n_t}} } \notag\\
		&= \tilO{B S^{3/2}A + \sqrt{SA\sumt\fV(P_t, \optV)} + \sqrt{SA\sumt \fV(P_t, \optV-V_{\lt})} + \sqrt{SAC_K} }, \label{eq:intermediate bt}
	\end{align}
	where in the last inequality we apply:
	\begin{align*}
		&\sumt \sqrt{\frac{\hatc_t}{\n_t}} \leq \sqrt{SA\rbr{ \sumt c(s_t, a_t) + \sumt (c(s_t, a_t) - \hatc_t) }} \tag{Cauchy-Schwarz inequality and \pref{lem:sum nt}}\\
		&\leq \sqrt{SA\rbr{2C_K + \tilO{1} + \sumt \sqrt{\frac{\hatc_t\iota_t}{\n_t}} + \frac{\iota_t}{\n_t}}} = \tilO{\sqrt{SAC_K} + \sqrt{SA\sumt\sqrt{\frac{\hatc_t}{\n_t}}} + SA}, \tag{\pref{lem:e2r} and \pref{eq:c-hatc mb}}
	\end{align*}
	and by \pref{lem:quad} we obtain: $\sumt\sqrt{\frac{\hatc_t}{\n_t}}=\tilo{\sqrt{SAC_K} + SA}$.
	Applying \pref{lem:var recursion} with $X_t(s) = (\optV(s)-V_t(s))/\B$, we have with probability at least $1-\delta$ ($G_T, Y_T$, and $\zeta_T$ are defined in \pref{lem:var recursion}),
	\begin{align*}
		&\sumt\fV(P_t, \optV-V_t) = \B^2 G_T(0) \leq 3\B^2Y_T + 9\B^2\zeta_T \leq 3\B\sumt (  (\Ind_{s_t} - P_t)(\optV-V_t) )_+  + \tilO{S\B^2}.\\
	\end{align*}
	By \pref{lem:sum Vref-V_t} and \pref{eq:intermediate bt}, with probability at least $1-\delta$,
	\begin{align*}
		&\sumt (  (\Ind_{s_t} - P_t)(\optV-V_t) )_+ \leq \sumt 2b_t + \frac{1}{H} P_t(\optV-V_t) \\
		&\qquad + \tilO{ \sqrt{SA\sumt\fV(P_t, \optV)} + \sqrt{S^2A\sumt\fV(P_t, \optV-V_{\lt})} + \B S^2A }.\\
		&= \tilO{ BS^2A + \sqrt{SA\sumt\fV(P_t, \optV)} + \sqrt{S^2A\sumt \fV(P_t, \optV-V_{\lt})} + \frac{1}{H}\sumt P_t(\optV-V_t) + \sqrt{SAC_K} }\\
		&\overset{\text{(i)}}{=} \tilO{ BS^2A + \sqrt{SA\sumt\fV(P_t, \optV)} + \sqrt{S^2A\sumt \fV(P_t, \optV-V_t)} + \frac{1}{H}\sumt P_t(\optV-V_t) + \sqrt{SAC_K} },
	\end{align*}
	where in (i) we apply
	\begin{align*}
		&\sqrt{S^2A\sumt \fV(P_t, \optV-V_{\lt})} = \tilO{\sqrt{S^2A\sumt \fV(P_t, \optV-V_t)} + \sqrt{S^2A\sumt \fV(P_t, V_t-V_{\lt})}  } \tag{$\var[X+Y]\leq 2\var[X] + 2\var[Y]$ and $\sqrt{x+y}\leq\sqrt{x}+\sqrt{y}$} \\
		&= \tilO{ \sqrt{S^2A\sumt \fV(P_t, \optV-V_t)} + \sqrt{ S^2A\rbr{ \B^2SA +  \frac{\B}{H}\sumt P_t(\optV-V_t) } } } \tag{\pref{lem:update bias}}\\
		&= \tilO{ \sqrt{S^2A\sumt \fV(P_t, \optV-V_t)} + \B S^2A + \frac{1}{H}\sumt P_t(\optV-V_t)} . \tag{$\sqrt{x+y}\leq \sqrt{x}+\sqrt{y}$ and AM-GM Inequality}
	\end{align*}
	Plugging the bound on $\sumt (  (\Ind_{s_t} - P_t)(\optV-V_t) )_+$ back, we have
	\begin{align*}
		\sumt\fV(P_t, \optV-V_t) &= \tilO{ B^2S^2A + \B\sqrt{SA\sumt\fV(P_t, \optV)} + \B\sqrt{S^2A\sumt \fV(P_t, \optV-V_t)} } \\
		&\qquad+ \tilO{ \frac{\B}{H}\sumt P_t(\optV-V_t) + \B\sqrt{SAC_K}}.
	\end{align*}
	Solving a quadratic inequality w.r.t $\sumt\fV(P_t, \optV-V_t)$ (\pref{lem:quad}), we obtain
	$$\sumt\fV(P_t, \optV-V_t) = \tilO{ B^2S^2A + \B\sqrt{SA\sumt\fV(P_t, \optV)} + \frac{\B}{H}\sumt P_t(\optV-V_t) + \B\sqrt{SAC_K} },$$
	and by $\var[X+Y]\leq 2\var[X] + 2\var[Y]$ and \pref{lem:update bias},
	\begin{align*}
		\sumt\fV(P_t, \optV-V_{\lt}) &=\tilO{ \sumt\fV(P_t, \optV-V_t) + \fV(P_t, V_t-V_{\lt}) }\\
		&= \tilO{ \B\sqrt{SA\sumt\fV(P_t, \optV)} + B^2S^2A + \frac{\B}{H}\sumt P_t(\optV-V_t) + \B\sqrt{SAC_K} }.
	\end{align*}
	Moreover, by $\sqrt{x+y}\leq\sqrt{x}+\sqrt{y}$ and AM-GM inequality:
	\begin{align*}
		&\sqrt{ SA\sumt\fV(P_t, \optV-V_{\lt}) }\\
		&= \tilO{ \sqrt{ \B SA\sqrt{SA\sumt\fV(P_t, \optV)} } + BS^{3/2}A + \sqrt{\frac{\B SA}{H}\sumt P_t(\optV-V_t)} + \sqrt{\B SA\sqrt{SAC_K}} }\\
		&= \tilO{ \sqrt{SA\sumt\fV(P_t, \optV)} + BS^{3/2}A + \sqrt{\frac{\B SA}{H}\sumt P_t(\optV-V_t)} + \sqrt{SAC_K} }.
	\end{align*}
	Plug this back to \pref{eq:intermediate bt}:
	$$\sumt b_t = \tilO{ BS^{3/2}A + \sqrt{SA\sumt\fV(P_t, \optV)} + \sqrt{\frac{\B SA}{H}\sumt P_t(\optV-V_t)} + \sqrt{SAC_K} }.$$
\end{proof}

\begin{lemma}
	\label{lem:var recursion}
	Suppose $X_t:\calS^+\rightarrow [0, 1]$ is monotonic in $t$ (that is, $X_t(s)$ is non-decreasing or non-increasing in $t$ for all $s\in\calS^+$), and $X_t(g)=0$. 
	Define:
	\begin{align*}
		F_n(d) = \sum_{t=1}^n P_tX_t^{2^d} - (X_t(s'_t))^{2^d}, \quad G_n(d) = \sum_{t=1}^n \fV(P_t, X_t^{2^d}).
	\end{align*}
	Then with probability at least $1-\delta$, for all $n\in\fN^+$ simultaneously, $G_n(0) \leq 3Y_n + 9\zeta_n, F_n(0) \leq \sqrt{3Y_n\zeta_n} + 4\zeta_n$, where $Y_n = S + 1 + \sum_{t=1}^n( X_t(s_t) - P_tX_t )_+, \zeta_n=32\ln^3\frac{4n^4}{\delta}$.
\end{lemma}
\begin{proof}
	Note that:
	\begin{align*}
		G_n(d) &= \sum_{t=1}^n P_tX_t^{2^{d+1}} - (P_tX_t^{2^d})^2 \leq \sum_{t=1}^n P_tX_t^{2^{d+1}} - (P_tX_t)^{2^{d+1}} \tag{$x^p$ is convex for $p>1$}\\
		&= \sum_{t=1}^n P_tX_t^{2^{d+1}} - X_t(s'_t)^{2^{d+1}} + \sum_{t=1}^n X_t(s'_t)^{2^{d+1}} - X_t(s_t)^{2^{d+1}} + \sum_{t=1}^n X_t(s_t)^{2^{d+1}} - (P_tX_t)^{2^{d+1}} \\
		&\overset{\text{(i)}}{\leq} F_n(d+1) + S + 1 + 2^{d+1}( X_t(s_t) - P_tX_t )_+ \leq F(d+1) + 2^{d+1}Y_n,
	\end{align*}
	where in (i) we apply \pref{lem:ak-bk} and,
	\begin{align*}
		&\sum_{t=1}^n X_t(s'_t)^{2^{d+1}} - X_t(s_t)^{2^{d+1}} = \sum_{t=1}^n X_t(s'_t)^{2^{d+1}} - X_{t+1}(s'_t)^{2^{d+1}} +  \sum_{t=1}^n X_{t+1}(s'_t)^{2^{d+1}} - X_t(s_t)^{2^{d+1}}\\
		&\leq S + \sum_{t=1}^n X_{t+1}(s_{t+1})^{2^{d+1}} - X_t(s_t)^{2^{d+1}} = S + X_{n+1}(s_{n+1})^{2^{d+1}} - X_1(s_1)^{2^{d+1}} \leq S + 1. \tag{\pref{lem:t diff} and $X_{t+1}(s'_t)\leq X_{t+1}(s_{t+1})$}
	\end{align*}
	For a fixed $d, n$, by \pref{eq:strong freedman} of \pref{lem:any interval freedman}, with probability $1-\frac{\delta}{2n^2\ceil{\log_2n+1}}$, 
	$$F_n(d) \leq \sqrt{G_n(d)\zeta_n} + \zeta_n \leq \sqrt{(F_n(d+1) + 2^{d+1}Y_n)\zeta_n} + \zeta_n.$$
	Taking a union bound on $d=0,\ldots,\ceil{\log_2n}$, and by \pref{lem:exp recursion} with $\lambda_1=n, \lambda_2=\sqrt{\zeta_n}, \lambda_3=Y_n, \lambda_4=\zeta_n$, we have:
	$$ F_n(1) \leq \max\{(\sqrt{\zeta_n} + \sqrt{2\zeta_n})^2, \sqrt{8Y_n\zeta_n} + \zeta_n\} \leq \max\{6\zeta_n, \sqrt{8Y_n\zeta_n} + \zeta_n\}.$$
	Therefore, $G_n(0) \leq F_n(1) + 2Y_n \leq \max\{6\zeta_n, Y_n + 9\zeta_n \} + 2Y_n \leq  3Y_n + 9\zeta_n$, and $F_n(0)\leq \sqrt{G_n(0)\zeta_n}+\zeta_n \leq \sqrt{3Y_n\zeta_n} + 4\zeta_n$.
	Taking a union bound over $n\in\fN^+$ proves the claim.
\end{proof}

\begin{lemma}
	\label{lem:P-eP}
	Given $X_t: \calS^+\rightarrow \fR$ with $\norm{X_t}_{\infty}\leq B$, with probability at least $1-\delta$, it holds that for all $t\geq 1$ simultaneously: $(P_t-\P_t)X_t = \tilO{\sqrt{\frac{S\fV(P_t, X_t)}{\n_t}} + \frac{SB}{\n_t}}.$
\end{lemma}
\begin{proof}
	For a fixed $(s, a)\in\SA$, by \pref{lem:anytime bernstein}, with probability $1-\frac{\delta}{SA}$, for any $t\geq 1$ such that $(s_t, a_t)=(s, a)$:
	\begin{align*}
		&(P_t-\P_t)X_t = \sum_{s'}(P_t(s')-\P_t(s'))(X_t(s') - P_tX_t ) \tag{$\sum_{s'}P_t(s')-\P_t(s') = 0$}\\
		&= \tilO{ \sum_{s'}\rbr{\sqrt{\frac{P_t(s')}{\n_t}} + \frac{1}{\n_t}}|X_t(s') - P_tX_t| } = \tilO{ \sqrt{\frac{S\fV(P_t, X_t)}{\n_t}} + \frac{SB}{\n_t} }.
	\end{align*}
	Taking a union bound over $(s, a)\in\SA$, the statement is proved.
\end{proof}

\begin{lemma}
	\label{lem:sum nt}
	$\sumt\frac{1}{\n_t} =\bigo{SA\ln T}$.
\end{lemma}
\begin{proof}
	Define $J_{s, a}$ such that $E_{J_{s, a}}=n_T(s, a)$.
	It is easy to see that $e_{j+1}/e_j\leq 2$.
	Then,
	$$\sumt\frac{1}{\n_t} \leq SA + \sumsa\sum_{j=1}^{J_{s, a}}\frac{e_{j+1}}{E_j} \leq SA + 2\sumsa\sum_{j=1}^{J_{s, a}}\frac{e_j}{E_j} = \bigO{SA\ln T}.$$
\end{proof}
 
 \subsection{Parameter-free Algorithm}
 \label{sec:pf mb}
 
 \setcounter{AlgoLine}{0}
\begin{algorithm}[t]
	\caption{\mb without knowledge of $\B$}
	\label{alg:SVI-B}
	
	\textbf{Parameters:} failure probability $\delta \in (0,1)$.
	
	\textbf{Define:} $\calL=\{E_j\}_{j\in\fN^+}$, where $E_j=\sum_{i=1}^je_i, e_j=\floor{\tile_j}$, and $\tile_1=1, \tile_{j+1}=\tile_j + \frac{1}{H}e_j$.
	
	\textbf{Initialize:} $B\leftarrow \frac{\sqrt{K}}{S^{3/2}A^{1/2}}, H\leftarrow\ceil{\frac{4B}{\cmin}\ln\frac{4B^2SAK}{\cmin}}_2, C\leftarrow 0, t\leftarrow 0, s_1\leftarrow \sinit$. 
	
	\textbf{Initialize:} for all $(s, a, s'), n(s, a, s')\leftarrow 0, n(s, a)\leftarrow 0$, $Q(s, a)\leftarrow 0$, $V(s)\leftarrow 0$, $\hatC(s, a)\leftarrow 0$. 
	
	\For{$k=1,\ldots,K$}{
		
		\MyRepeat{
			Increment time step $t\overset{+}{\leftarrow}1$.
			
			Take action $a_t= \argmin_aQ(s_t, a)$, suffer cost $c_t$, transit to and observe $s'_t$.
			
			Update visitation counters: $n=n(s_t, a_t)\overset{+}{\leftarrow} 1, n(s_t, a_t, s'_t)\overset{+}{\leftarrow} 1$.
			
			Update cost accumulator $C\overset{+}{\leftarrow} c_t$, $\hatC(s, a)\leftarrow c_t$.
						
			\If{$n \in\calL$}{
				Update empirical transition: $\P_{s_t, a_t}(s')\leftarrow\frac{n(s_t, a_t, s')}{n}$ for all $s'$. 
			
				Compute $\iota\leftarrow \ln\frac{2SAn}{\delta}$, $\hatc\leftarrow\frac{\hatC(s_t, a_t)}{n}$, and bonus $b\leftarrow \max\Big\{7\sqrt{\frac{\fV(\P_{s_t, a_t}, V)\iota}{n}}, \frac{49B\iota}{n}\Big\}$.
			
				$Q(s_t, a_t) \leftarrow \max\{\hatc + \P_{s_t, a_t}V - b, Q(s_t, a_t)\}$.
				
				$V(s_t)\leftarrow \argmin_a Q(s_t, a)$.
			}
			
			\If{$\norm{V}_{\infty}>B$ or $C> KB + x(B\sqrt{SAK} + BS^2A)$}{
				$B\leftarrow 2B, H\leftarrow\ceil{\frac{4B}{\cmin}\ln\frac{4B^2SAK}{\cmin}}_2, C\leftarrow 0$, and update $x$.
			
				$n(s, a, s')\leftarrow 0, n(s, a)\leftarrow 0$, $Q(s, a)\leftarrow 0$, $V(s)\leftarrow 0$, $\hatC(s, a)\leftarrow 0$ for all $(s, a, s')$.
			}
			
			\lIf{$s_t' \neq g$}{$s_{t+1}\leftarrow s'_t$; \textbf{else} $s_{t+1}\leftarrow \sinit$, \textbf{break}.} 
		}
	}
\end{algorithm}
 
%
%
%
%
 
Following \citep{tarbouriech2021stochastic}, we divide the learning process into epochs indexed by $\phi$.
We maintain value function upper bound $B$ initialized with $\frac{\sqrt{K}}{S^{3/2}A^{1/2}}$ and cost accumulator $C$ recording the total costs suffered in the current epoch.
In epoch $\phi$, we execute an instance of \pref{alg:SVI} with value function upper bound $B$.
Moreover, we start a new epoch whenever:
\begin{enumerate}
	\item $\norm{V}_{\infty}>B$,
	\item or $C > KB + x(B\sqrt{SAK} + BS^2A)$.
\end{enumerate}
Here, $x$ is a large enough constant determined by \pref{thm:mb}, so that when $B\geq\B$, we have with probability at least $1-12\delta$:
$$ C - \optV(\sinit^{\phi}) - (K-1)\optV(\sinit) \leq x(\B\sqrt{SAK} + BS^2A),$$
where $\sinit^{\phi}$ is the initial state of epoch $\phi$ (note that \pref{thm:mb} still holds when the initial state is changing over episodes).
Moreover, we double the value of $B$ whenever a new epoch starts.
We summarize ideas above in \pref{alg:SVI-B}.

\begin{theorem}
	With probability at least $1-12\delta$, \pref{alg:SVI-B} ensures $R_K=\tilo{\B\sqrt{SAK} + \B^3S^3A}$.
\end{theorem}
\begin{proof}
	Denote by $B_{\phi}$ the value of $B$ in epoch $\phi$, and by $C_{\phi}$ the value of $C$ at the end of epoch $\phi$.
	Define $\phistar=\inf_{\phi}\{B_{\phi}\geq\B\}$.
	Clearly $B_{\phi} \leq \max\{2\B, \sqrt{K}/S^{3/2}A^{1/2}\}$ for $\phi\leq\phistar$.
	By \pref{thm:mb}, with probability at least $1-12\delta$, there is at most $\phistar$ epochs since the condition of starting a new epoch will never be triggered in epoch $\phistar$, and the regret in epoch $\phistar$ is properly bounded:
	\begin{align*}
		C_{\phistar} - \optV(\sinit^{\phistar}) - (K-1)\optV(\sinit) = \tilO{ \B\sqrt{SAK} + B_{\phistar}S^2A } = \tilO{ \B\sqrt{SAK} + \B S^2A }.
	\end{align*}
	Conditioned on the event that there are at most $\phistar$ epochs, we partition the regret into two parts: the total costs suffered before epoch $\phistar$, and the regret starting from epoch $\phistar$.
	It suffices to bound the total costs before epoch $\phistar$ assuming $K\leq \B^2 S^3A$ (otherwise $\phistar=1$).
	By the update scheme of $B$, we have at most $\ceil{\log_2\B}+1$ epochs before epoch $\phistar$.
	Moreover, by the second condition of starting a new epoch, the accumulated cost in epoch $\phi<\phistar$ is bounded by:
	\begin{align*}
		C_{\phi} \leq KB_{\phi} + \tilO{ B_{\phi}\sqrt{SAK} + B_{\phi}S^2A } = \tilO{ \B^3S^3A }.
	\end{align*}
	Combining these two parts, we get:
	\begin{align*}
		R_K &= \sum_{\phi=1}^{\phistar-1}C_{\phi} + ( C_{\phistar} - \optV(\sinit^{\phistar}) - (K-1)\optV(\sinit) ) + (\optV(\sinit^{\phistar}) - \optV(\sinit))\\
		&= \tilO{ \B\sqrt{SAK} + \B^3S^3A },
	\end{align*}
	where we assume $C_{\phistar}=0$ and $\sinit^{\phistar}=\sinit$ if there are less than $\phistar$ epochs.
\end{proof}

\section{Auxiliary Lemmas}


\begin{lemma}
	\label{lem:quad}
	If $x\leq (a\sqrt{x}+b)\ln^p(cx)$ for some $a, b, c>0$ and absolute constant $p\geq 0$, then $x = \tilo{a^2 + b}$.
	Specifically, $x \leq a\sqrt{x} + b$ implies $x\leq (a+\sqrt{b})^2\leq 2a^2+2b$.
\end{lemma}

\begin{lemma}
	\label{lem:ak-bk}
	For any $a,b\in[0,1]$ and $k\in\fN^+$, we have: $a^k-b^k\leq k(a-b)_+$.
\end{lemma}
\begin{proof}
	$a^k-b^k=(a-b)(\sum_{i=1}^ka^{i-1}b^{k-i})\leq (a-b)_+\cdot \sum_{i=1}^k1=k(a-b)_+$.
\end{proof}

\begin{lemma}{(\cite[Lemma 11]{zhang2020reinforcement})}
	\label{lem:exp recursion}
	Let $\lambda_1,\lambda_2,\lambda_4\geq 0, \lambda_3\geq 1$ and $i'=\log_2(\lambda_1)$. Let $a_1, a_2,\ldots,a_{i'}$ be non-negative reals such that $a_i\leq\lambda_1$ and $a_i\leq\lambda_2\sqrt{a_{i+1}+2^{i+1}\lambda_3}+\lambda_4$ for any $1\leq i\leq i'$. Then, $a_1\leq\max\{(\lambda_2+\sqrt{\lambda_2^2+\lambda_4})^2, \lambda_2\sqrt{8\lambda_3}+\lambda_4\}$.
\end{lemma}

\begin{lemma}
	\label{lem:t diff}
	Assume $v_t: \calS^+\rightarrow [0, B]$ is monotonic in $t$ (i.e., $v_t(s)$ is non-increasing or non-decreasing in $t$ for any $s\in\calS^+$).
	Then, for any state sequence $\{s_t\}_{t=1}^n, n\in\fN^+$, we have: $|\sum_{t=1}^n v_{t+1}(s_t) - v_t(s_t) | \leq SB$.
\end{lemma}
\begin{proof}
	\begin{align*}
		&\abr{\sum_{t=1}^n v_{t+1}(s_t) - v_t(s_t)} \leq \sum_{s\in\calS^+} \abr{ \sum_{t=1}^n (v_{t+1}(s) - v_t(s))\Ind\{s_t=s\} }\\
		&\leq \sum_{s\in\calS^+}\abr{\sum_{t=1}^n v_{t+1}(s) - v_t(s)} \leq \sum_{s\in\calS^+}\abr{ v_{n+1}(s) - v_1(s) } \leq SB. \tag{$v_t(s)$ is monotonic in $t$}
 	\end{align*}
\end{proof}

\begin{lemma}{(\cite[Lemma C.3]{cohen2021minimax})}
	\label{lem:var diff}
	For any two random variables $X, Y$ with $\var[X]<\infty, \var[Y]<\infty$.
	We have: $\sqrt{\var[X]} - \sqrt{\var[Y]} \leq \sqrt{\var[X-Y]}$.
\end{lemma}

\begin{lemma}
	\label{lem:var xy}
	For any two random variables $X, Y$, we have: 
	$$\var[XY] \leq 2\var[X]\norm{Y}_{\infty}^2 + 2(\E[X])^2\var[Y].$$
	Consequently, $\norm{X}_{\infty}\leq C$ implies $\var[X^2]\leq 4C^2\var[X]$.
\end{lemma}
\begin{proof}
	First note that for any two random variables $U, V$, we have $\var[U+V] \leq 2\var[U] + 2\var[V]$.
	Now let $U=(X-\E[X])Y$ and $V=\E[X]Y$, we have:
	\begin{align*}
		\var[XY] &\leq 2\var[(X-\E[X])Y] + 2\var[\E[X]Y] \leq 2\E[(X-\E[X])^2Y^2] + 2(\E[X])^2\var[Y]\\
		&\leq 2\var[X]\norm{Y}^2_{\infty} + 2(\E[X])^2\var[Y].
	\end{align*}
\end{proof}

\begin{lemma}{(\citep[Lemma 14]{tarbouriech2021stochastic})}
	\label{lem:mvp}
	Define $\Upsilon=\{ v\in[0, B]^{\calS^+}: v(g)=0 \}$.
	Let $f: \Delta_{\calS^+}\times\Upsilon\times\fR^+\times\fR^+\times\fR^+\rightarrow\fR^+$ with $f(p, v, n, B, \iota)=pv-\max\Big\{c_1\sqrt{\frac{\fV(p, v)\iota}{n}}, c_2\frac{B\iota}{n}\Big\}$, with $c_1=7$ and $c_2=49$.
	Then $f$ satisfies for all $p\in\Delta_{\calS^+}, v\in\Upsilon$ and $n, \iota>0$,
	\begin{enumerate}
		\item $f(p, v, n, B, \iota)$ is non-decreasing in $v(s)$, that is,
		$$\forall v, v'\in\Upsilon, v(s)\leq v'(s), \forall s\in\calS^+ \implies f(p, v, n, B, \iota)\leq f(p, v', n, B, \iota);$$
		\item $f(p, v, n, B, \iota)\leq pv-\frac{c_1}{2}\sqrt{\frac{\fV(p, v)\iota}{n}}-\frac{c_2}{2}\frac{B\iota}{n} \leq pv - 3\sqrt{\frac{\fV(p, v)\iota}{n}} - 24\frac{B\iota}{n}$.
	\end{enumerate}
\end{lemma}

\begin{lemma}{(\citep[Lemma 19]{jaksch2010near}, \citep[Lemma B.18]{cohen2020near})}
	\label{lem:z sum}
	For any sequence of numbers $z_1,\ldots,z_n$ with $0\leq z_t\leq Z_{t-1}=\max\{1, \sum_{i=1}^{t-1}z_i\}$:
	\begin{align*}
		\sum_{t=1}^n\frac{z_t}{Z_{t-1}} \leq 2\ln Z_n, \quad \sum_{t=1}^n\frac{z_t}{\sqrt{Z_{t-1}}} \leq 3\sqrt{Z_n}.
	\end{align*}
\end{lemma}

\section{Concentration Inequalities}

\begin{lemma}{(\citep[Theorem D.1]{cohen2020near})}
	\label{lem:anytime azuma}
	Let $\{X_t\}_t$ be a martingale difference sequence such that $|X_t|\leq B$.
	Then with probability at least $1-\delta$,
	\begin{align*}
		\abr{\sum_{t=1}^n X_t} \leq B\sqrt{n\ln\frac{2n}{\delta}},\quad\forall n\geq1.
	\end{align*}
\end{lemma}

\begin{lemma}
	\label{lem:anytime bernstein}
	Let $\{X_t\}_t$ be a sequence of i.i.d random variables with mean $\mu$, variance $\sigma^2$, and $0\leq X_t \leq B$.
	Then with probability at least $1-\delta$, the following holds for all $n\geq 1$ simultaneously:
	\begin{align*}
		\abr{\sum_{t=1}^n(X_t-\mu)} &\leq 2\sqrt{2\sigma^2 n\ln\frac{2n}{\delta}} + 2B\ln\frac{2n}{\delta}.\\
		\abr{\sum_{t=1}^n(X_t-\mu)} &\leq 2\sqrt{2\hat{\sigma}^2_nn\ln\frac{2n}{\delta}} + 19B\ln\frac{2n}{\delta}.
	\end{align*}
	where $\hat{\sigma}_n^2=\frac{1}{n}\sum_{t=1}^nX_t^2 - (\frac{1}{n}\sum_{t=1}^nX_t)^2$.
\end{lemma}
\begin{proof}
	For a fixed $n$, the first inequality holds with probability at least $1-\frac{\delta}{4n^2}$ by Freedman's inequality.
	Then by \citep[Lemma 19]{efroni2021confidence}, with probability at least $1-\frac{\delta}{4n^2}$, $|\sigma - \hat{\sigma}_n| \leq \sqrt{\frac{36B^2\ln(2n/\delta)}{\n}}$.
	Therefore, $\sqrt{n}\sigma = \sqrt{n}\hat{\sigma}_n + \sqrt{n}(\sigma-\hat{\sigma}_n) \leq \sqrt{n}\hat{\sigma}_n + 6B\sqrt{\ln(2n/\delta)}$.
	Plugging this back to the first inequality gives the second inequality.
\end{proof}

\begin{lemma}{(Strengthened Freedman's inequality)}\label{lem:any interval freedman}
Let $X_{1:\infty}$ be a martingale difference sequence with respect to a filtration $\{\calF_t\}_t$ such that $\E[X_t|\calF_{t-1}] = 0$.
Suppose $B_t \in [1,b]$ for a fixed constant $b$, $B_t\in\calF_{t-1}$ and $X_t \leq B_t$ almost surely.
Then for a given $n$, with probability at least $1-\delta$:
\begin{equation}
	\label{eq:strong freedman}
	\abr{\sum_{t=1}^nX_t} \leq  C\big(\sqrt{8V_{1,n}\ln\left(2C/\delta\right)} + 5B_{1,n} \ln\left(2C/\delta\right)\big),
\end{equation}
and with probability at least $1 -\delta$ we have for all $1\leq l \leq n $ simultaneously
\begin{equation}
	\label{eq:anytime strong freedman}
    \abr{\sum_{t=l}^{l+n-1}  X_t} \leq  C\big(\sqrt{8V_{l,n}\ln\left(4Cn^3/\delta\right)} + 5B_{l,n} \ln\left(4Cn^3/\delta\right)\big) \leq 8CB_{l,n}\sqrt{n}\ln(4Cn^3/\delta),
\end{equation}
where $V_{l,n} = \sum_{t=l}^{l+n-1} \E[X_t^2|\calF_{t-1}], B_{l,n}=\max_{l\leq t < l+n}B_t$, and 
$C = \ceil{\ln(b)}\ceil{\ln(nb^2)}$.
\end{lemma}
\begin{proof}
	\pref{eq:strong freedman} is simply from applying \citep[Theorem 2.2]{lee2020bias} to $\{X_t\}_t$ and $\{-X_t\}_t$.
	Fix some $l, n\geq 1$.
	\pref{eq:anytime strong freedman} holds with probability at least $1-\frac{\delta}{2n^3}$ by \pref{eq:strong freedman}.
	By a union bound (first sum over $l$, then sum over $n$), the statement is proved.
\end{proof}

\begin{lemma}
	\label{lem:e2r}
	Given $\alpha\geq 1$ and a martingale sequence $\{X_t\}_t$ such that $X_t\in\calF_t, 0\leq X_t \leq B$, with probability at least $1-\delta$:
	\begin{align*}
		\sum_{t=1}^n \E[X_t|\calF_{t-1}] \leq \rbr{1 + \frac{1}{\alpha}}\sum_{t=1}^n X_t + 8B\alpha\ln\frac{2n}{\delta},\quad \forall n\geq 1.
	\end{align*}
\end{lemma}
\begin{proof}
	Define $Y_t = \E[X_t|\calF_{t-1}] - X_t$.
	For a given $n$, by Freedman's inequality, with probability at least $1-\frac{\delta}{2n^2}$:
	\begin{align*}
		\sum_{t=1}^n Y_t \leq \eta\sum_{t=1}^n\E[(X_t - \E[X_t|\calF_{t-1}])^2|\calF_{t-1}] + \frac{2\ln(2n/\delta)}{\eta} \leq B\eta\E[X_t|\calF_{t-1}] + \frac{2\ln(2n/\delta)}{\eta},
	\end{align*}
	for some $\eta < \frac{1}{B}$.
	Reorganizng terms, we get when $\eta=\frac{1}{2B\alpha}<\frac{1}{B}$ (note that $B\eta \leq \frac{1}{2}$):
	\begin{align*}
		\sum_{t=1}^n \E[X_t|\calF_{t-1}] &\leq \frac{1}{1-B\eta}\rbr{\sum_{t=1}^n X_t + \frac{2\ln(2n/\delta)}{\eta}} \leq (1 + 2B\eta)\sum_{t=1}^n X_t + \frac{4\ln(2n/\delta)}{\eta}\\
		&\leq \rbr{1 + \frac{1}{\alpha}}\sum_{t=1}^n X_t + 8B\alpha\ln\frac{2n}{\delta}. \tag{$\frac{1}{1-x}\leq 1+2x$ when $x\in [0, \frac{1}{2}]$}
	\end{align*}
	By a union bound over $n$, we obtain the desired bound.
\end{proof}

\section{Experiments}
\label{app:exp}

In this section, we benchmark known SSP algorithms empirically.
We consider two environments, RandomMDP and GridWorld.
In RandomMDP, there are 5 states and 2 actions, and both transition and cost function are chosen uniformly at random.
In GridWorld, there are $12$ states (including the goal state) and 4 actions (LEFT, RIGHT, UP, DOWN) forming a $3\times 4$ grid.
The agent starts at the upper left corner of the grid, and the goal state is at the lower right corner of the grid.
Taking each action initiates an attempt to moves one step towards the indicated direction with probability $0.85$, and moves randomly towards the other three directions with probability $0.15$.
The movement attempt fails if the agent tries to move out of the grid, and in this case the agent stays at the same position.
The cost is $1$ for each state-action pair.
In our experiments, $\B\approx 1.5$ and $\cmin\approx 0.04$ in RandomMDP, and $\B\approx 6$ and $\cmin=1$ in GridWorld.

We implement two model-free algorithms: Q-learning with  $\epsilon$-greedy exploration~\citep{yu2013boundedness} and \mf, and five model-based algorithms: UC-SSP~\citep{tarbouriech2020no}\footnote{we implement a variant of UC-SSP with a fixed pivot horizon for a much better empirical performance, where $\gamma_{k,j}=10^{-6}$ always (see their Algorithm 2 for the definition of $\gamma_{k,j}$)}, Bernstein-SSP~\citep{cohen2020near}, ULCVI~\citep{cohen2021minimax}, EB-SSP~\citep{tarbouriech2021stochastic}, and \mb.
For each algorithm, we optimize hyper-parameters for the best possible results.
Moreover, instead of incorporating the logarithmic terms from confidence intervals suggested by the theory, we treat it as a hyper-parameter $\iota$ and search its best value.
The hyper-parameters used in the experiments are shown in \pref{tab:hp}.
All experiments are performed in Google Cloud Platform on a compute engine with machine type ``e2-medium''.

The plot of accumulated regret is shown in \pref{fig:plot}.
Q-learning with $\epsilon$-greedy exploration suffers linear regret, indicating that naive $\epsilon$-greedy exploration is inefficient.
UC-SSP and SVI-SSP show competitive results in both environments.
SVI-SSP also consistently outperforms EB-SSP, both of which are minimax-optimal and horizon-free.

In \pref{tab:update}, we also show the time spent in updates (policy, accumulators, etc) in the whole learning process for each algorithm.
Our model-based algorithm \mb spends least time in updates among all algorithms, confirming our theoretical arguments.
ULCVI and UC-SSP spend most time in updates, which is reasonable since these two algorithms computes a new policy in each episode, instead of exponentially sparse updates.

\begin{figure}[t]
	\centering
	\begin{tabular}{cc}
		\includegraphics[width=0.5\textwidth]{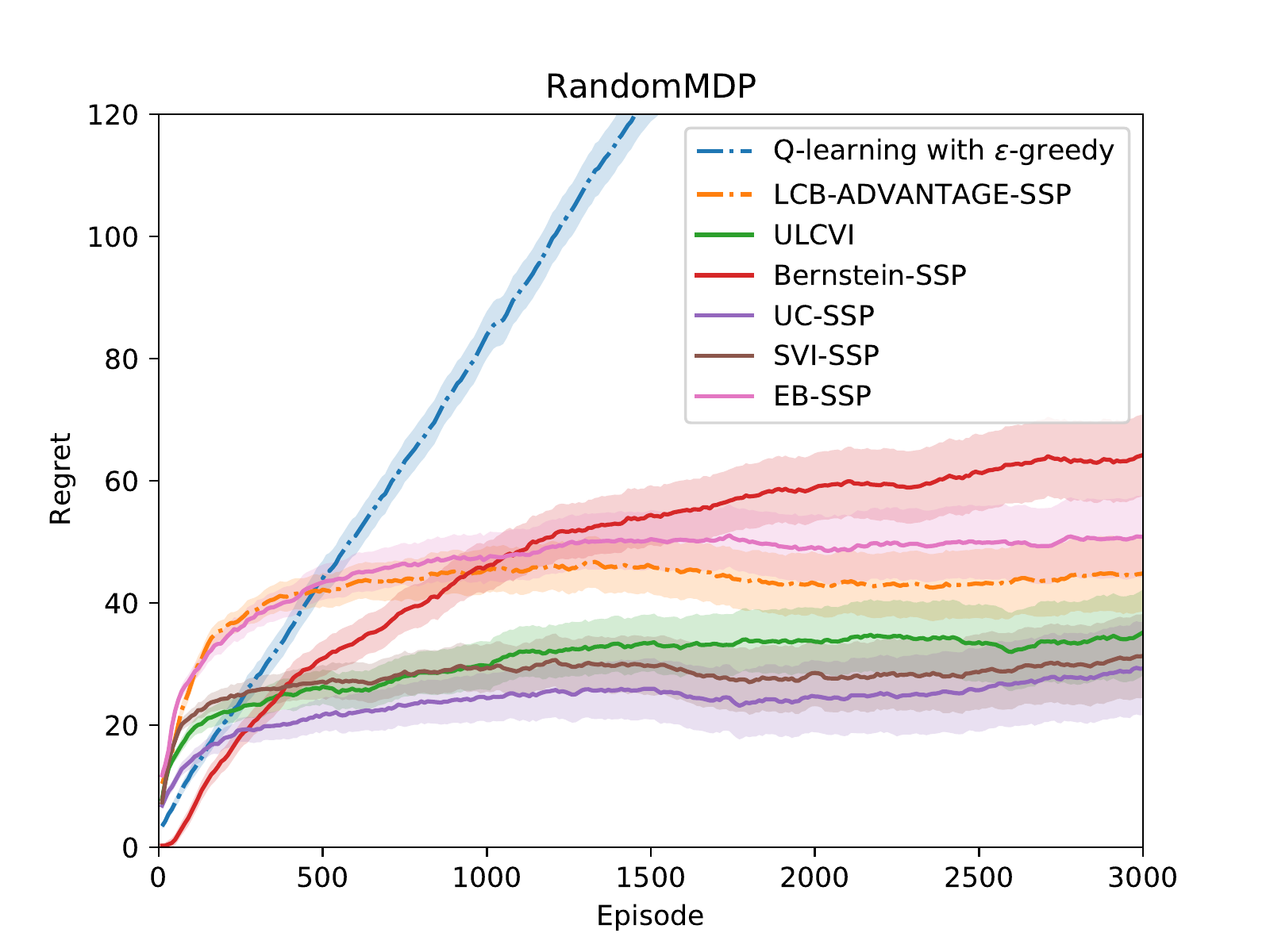} &
		\includegraphics[width=0.5\textwidth]{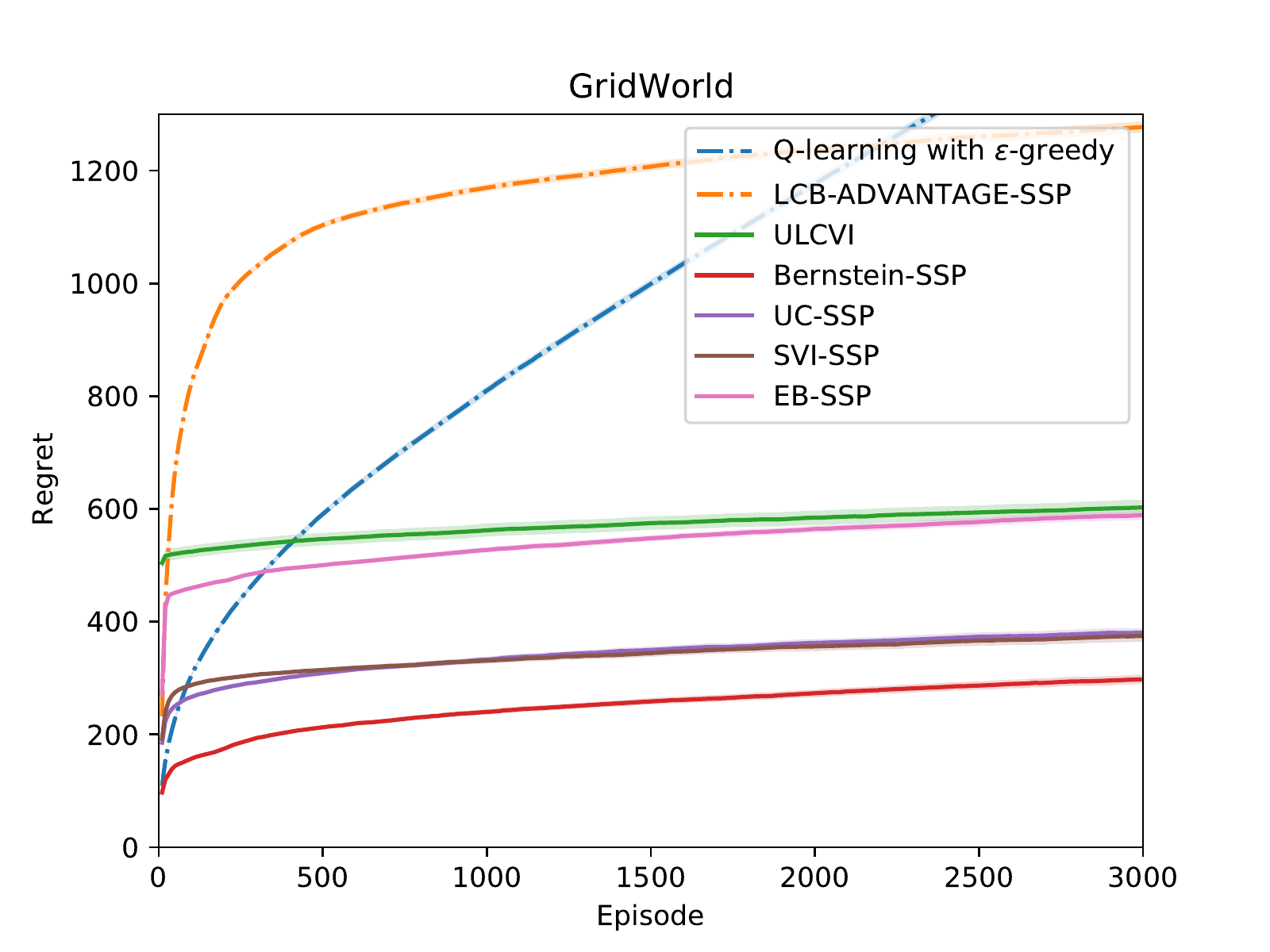}
	\end{tabular}
	\caption{
		Accumulated regret of each algorithm on RandomMDP (left) and GridWorld (right) in $3000$ episodes.
		Each plot is an average of 500 repeated runs, and the shaded area is 95\% confidence interval.
		Dotted lines represent model-free algorithms and solid lines represent model-based algorithms.
	}
	\label{fig:plot}
\end{figure}

\begin{table}[t]
	\centering
	\caption{
		\small Average time (in seconds) spent in updates in $3000$ episodes for each algorithm.
		Our model-based algorithm \mb is the most efficient algorithm.
	}
	\label{tab:update}
	\begin{tabular}{|c|c|c|}
		\hline
		& RandomMDP & GridWorld\\
		\hline
		Q-learning with $\epsilon$-greedy & 0.3385 & 0.3773 \\
		\mf  & 0.3517 & 0.3982 \\
		UC-SSP & 14.4472 & 8.6886 \\
		Bernstein-SSP & 0.2918 & 0.4656 \\
		ULCVI & 15.7128 & 22.8062 \\
		EB-SSP & 0.2319 & 0.4619 \\
		SVI-SSP & \textbf{0.1207} & \textbf{0.1419} \\
		\hline
	\end{tabular}
\end{table}

\begin{table}[t]
	\centering
	\caption{\small 
		Hyper-parameters used in the experiments.
		We search the best parameters for each algorithm.
	}
	\label{tab:hp}
	\begin{tabular}{|c|c|c|}
		\hline
		& Algorithm & Parameters\\
		\hline
		\multirow{7}{*}{RandomMDP} & Q-learning with $\epsilon$-greedy & $\epsilon=0.05$ \\
		& \mf & $H=5, \iota=0.05, \thetastar=4096$ \\
		& UC-SSP & $\iota=1.0$ \\
		& Bernstein-SSP & $\iota=2.0$ \\
		& ULCVI & $H=80, \iota=2.0$ \\
		& EB-SSP & $\iota=0.05$ \\
		& SVI-SSP & $H=15, \iota=0.05$ \\
		\hline
		\multirow{7}{*}{GridWorld} & Q-learning with $\epsilon$-greedy & $\epsilon=0.05$ \\
		& \mf & $H=5, \iota=0.1, \thetastar=4096$ \\
		& UC-SSP & $\iota=0.5$ \\
		& Bernstein-SSP & $\iota=0.5$ \\
		& ULCVI & $H=100, \iota=1.0$ \\
		& EB-SSP & $\iota=0.01$ \\
		& SVI-SSP & $H=10, \iota=0.01$ \\
		\hline
	\end{tabular}
\end{table}



\end{document}